\documentclass[final,12pt]{colt2026}
\usepackage{times}
\usepackage{helvet}
\usepackage{inconsolata}

\usepackage{microtype}
\usepackage[american]{babel}

\usepackage{amsfonts}
\usepackage{mathrsfs}
\usepackage{bm}

\setcitestyle{square}

\usepackage{hyperref}

\usepackage{enumitem}

\usepackage{booktabs}
\usepackage{makecell}

\usepackage{algorithmic}
\usepackage{algorithm}

\usepackage{fancybox}

\usepackage{xspace}

\usepackage{wrapfig}

\usepackage{tikz}

\hypersetup{
  colorlinks,
  breaklinks,
  urlcolor = black,
  linkcolor = blue,
  citecolor = blue,
}

\newtheorem{thm}{Theorem}
\newtheorem{lem}{Lemma}
\newtheorem{cor}{Corollary}
\newtheorem{prp}{Proposition}

\newtheorem{cdt}{Condition}
\newtheorem{asm}{Assumption}
\newtheorem{rmk}{Remark}

\newcounter{romancounter}
\newcommand{\rom}[1]{\setcounter{romancounter}{#1}\emph{(\roman{romancounter})}}

\newcommand{\algcomment}[1]{\hfill$\rhd\ $\text{#1}}

\def \rme {\mathrm{e}}

\def \gr {\nabla}
\def \bgr {\bm{\nabla}}

\def \reals {\mathbb{R}}

\def \transposed {\mathsf{T}}
\def \trs {\transposed}

\DeclareMathOperator*{\argmax}{arg\,max}
\DeclareMathOperator*{\argmin}{arg\,min}

\renewcommand{\tilde}{\widetilde}
\renewcommand{\hat}{\widehat}

\let\mid\undefined

\def \mid {\,\vert\,}

\newcommand \abs[1] {\left| #1 \right|}
\newcommand \norm[1] {\left\| #1 \right\|}
\newcommand \tnorm[1] {\| #1 \|}
\newcommand \norme[1] {\norm{#1}_{2}}
\newcommand \normf[1] {\norm{#1}_{\operatorname{F}}}
\newcommand \tnormf[1] {\tnorm{#1}_{\operatorname{F}}}
\newcommand \inner[2] {\left\langle #1, #2 \right\rangle}

\newcommand \innerf[2] {\inner{#1}{#2}_{\operatorname{F}}}

\newcommand \sbr[1] {\left( #1 \right)}
\newcommand \mbr[1] {\left[ #1 \right]}
\newcommand \lbr[1] {\left\{ #1 \right\}}

\newcommand \ttr[1] {\operatorname{tr}(#1)}

\newcommand \tdet[1] {\operatorname{det}(#1)}
\newcommand \diag[1] {\operatorname{diag}\!\sbr{#1}}

\newcommand \expectation[2] {\mathop{\mathbb{E}}_{#1}\!\mbr{#2}}

\newcommand \probability[1] {\mathop{\mathbb{P}}\!\sbr{#1}}

\def \a {\mathbf{a}}
\def \b {\mathbf{b}}

\def \e {\mathbf{e}}

\def \u {\mathbf{u}}
\def \v {\mathbf{v}}
\def \w {\mathbf{w}}
\def \x {\mathbf{x}}
\def \y {\mathbf{y}}
\def \z {\mathbf{z}}

\def \bdelta {\bm{\delta}}

\def \eps {\epsilon}
\def \veps {\varepsilon}

\def \BB {\mathcal{B}}

\def \DD {\mathcal{D}}

\def \LL {\mathcal{L}}

\def \SS {\mathcal{S}}

\def \WW {\mathcal{W}}
\def \XX {\mathcal{X}}
\def \YY {\mathcal{Y}}

\def \ftx {f_t(\x)}
\def \ftxt {f_t(\x_t)}

\def \ftu {f_t(\u)}

\def \gty {g_t(\y)}
\def \gtyt {g_t(\y_t)}

\def \reg {\normalfont\textsc{Reg}}
\def \runtime {\normalfont\textsc{Runtime}}
\def \EP {\normalfont\texttt{EP}}
\def \MP {\normalfont\texttt{MP}}

\def \lightons {\mbox{\text{LightONS}}\xspace}
\def \lightonscore {\mbox{\text{LightONS.Core}}\xspace}
\def \lightson {\mbox{\text{LightONS.Sketch}}\xspace}
\def \fastproj {\mbox{\text{FastProj}}\xspace}

\usepackage{pifont}
\def \Yes {\textcolor{red}{\ding{51}}}
\def \No {\textcolor{blue}{\ding{55}}}

\usepackage{cleveref}
\crefname{thm}{Theorem}{Theorems}
\crefname{lem}{Lemma}{Lemmas}
\crefname{cor}{Corollary}{Corollaries}
\crefname{prp}{Proposition}{Propositions}
\crefname{dfn}{Definition}{Definitions}
\crefname{cdt}{Condition}{Conditions}
\crefname{asm}{Assumption}{Assumptions}
\crefname{rmk}{Remark}{Remarks}
\crefname{pbl}{Problem}{Problems}
\crefname{prfskt}{Proof Sketch}{Proof Sketches}
\crefname{als}{Analysis}{Analysis}
\crefname{equation}{Eq.}{Eqs.}
\crefname{table}{Table}{Tables}
\crefname{section}{Section}{Sections}
\crefname{subsection}{Section}{Sections}
\crefname{subsubsection}{Section}{Sections}
\crefname{appendix}{Appendix}{Appendices}
\crefname{algorithm}{Algorithm}{Algorithms}
\crefname{figure}{Figure}{Figures}

\title[A Simple, Optimal and Efficient Algorithm for Online Exp-Concave Optimization]{A Simple, Optimal and Efficient Algorithm for \\ Online Exp-Concave Optimization}
\coltauthor{
  \Name{Yi-Han Wang} \Email{wangyh@lamda.nju.edu.cn} \\
  \Name{Peng Zhao\nametag{\thanks{Corresponding Author.}}} \Email{zhaop@lamda.nju.edu.cn} \\
  \Name{Zhi-Hua Zhou} \Email{zhouzh@lamda.nju.edu.cn} \\
  \addr State Key Laboratory for Novel Software Technology, Nanjing University, China \\ School of Artificial Intelligence, Nanjing University, China
}

\begin{document}

\maketitle

\begin{abstract}
  Online eXp-concave Optimization (OXO) is a fundamental problem in online learning, where the goal is to minimize regret when loss functions are exponentially concave.
  The standard algorithm, Online Newton Step (ONS), guarantees an optimal $O(d \log T)$ regret, where $d$ is the dimension and $T$ is the time horizon.
  Despite its simplicity, ONS may face a computational bottleneck due to the \emph{Mahalanobis projection} at each round.
  This step costs $\tilde{O}(d^\omega)$ arithmetic operations for bounded domains, even for simple domains such as the unit ball, where $\omega \in (2,3]$ is the matrix-multiplication exponent.
  As a result, the total runtime can reach $\tilde{O}(d^\omega T)$, particularly when iterates frequently oscillate near the domain boundary.
  This paper proposes a simple variant of ONS, called \lightons, which reduces the total runtime to $O(d^2 T + d^\omega \sqrt{T \log T})$ while preserving the optimal regret.
  Deploying \lightons with the online-to-batch conversion implies a method for stochastic exp-concave optimization with runtime $\tilde{O}(d^3/\varepsilon)$, thereby answering an open problem posed by~\citet{koren2013open}.
  The design leverages domain-conversion techniques from parameter-free online learning and defers expensive Mahalanobis projections until necessary, thereby preserving the elegant structure of ONS and enabling \lightons to act as an efficient plug-in replacement in broader scenarios, including gradient-norm adaptivity, parametric stochastic bandits, and memory-efficient OXO.
\end{abstract}

\section{Introduction}
\label{sec:intro}

Online Convex Optimization (OCO) provides a versatile framework for online learning, with deep connections to stochastic optimization, game theory, and information theory~\citep{book/Cambridge/cesa2006prediction,hazan2016introductionOCO}.
Online eXp-concave Optimization (OXO) is an important instance, where loss functions are exponentially concave (exp-concave), i.e., $\exp(- \alpha f(\cdot))$ is concave for some $\alpha>0$ with $f$ denoting the loss function.
Exp-concavity naturally arises in many machine learning applications such as linear/logistic regression~\citep{foster1991OLR,vovk1997OLR,COLT18:improper-LR}, portfolio selection~\citep{cover1991universal}, linear-quadratic regulator control~\citep{ICML20:log-control}, and so on.
From a theoretical perspective, it introduces rich structures beyond convexity, allowing algorithms to exploit the local geometry of the loss landscape, often through local norms induced along the optimization trajectory.
Such structures yield sharper statistical guarantees:
Exp-concave losses admit a minimax-optimal regret of $O(d \log T)$~\citep{ordentlich-cover-1998-portfolio}, an exponential improvement over the $\Omega(\sqrt{T})$ lower bound for convex losses~\citep{COLT08:lower-bound}.

Online Newton Step (ONS)~\citep{journals/ml/HazanAK07} is the de facto standard for OXO, achieving an optimal regret bound of $O(d \log T)$ with constant time and space per round.
ONS exhibits remarkable simplicity, which has driven many advances in diverse optimization settings~\citep{aistats12:exp-concave-smooth,NIPS16:sketch-ONS,COLT18:black-box-reduction} and in machine learning applications even beyond regret minimization, such as generalized linear bandits~\citep{ICML16:Zhang-one-bit,NIPS17:online-GLB,ZYJ2025LogB}.
The important insight of ONS is to leverage exp-concavity by maintaining a \emph{Hessian-related matrix} that captures the local geometry of the optimization trajectory, which is crucial for a balance between statistical optimality and computational practicality.
Indeed, all known optimal OXO algorithms that avoid ONS-like Hessian maintenance, such as Exponential Weight Online Optimization (EWOO)~\citep{journals/ml/HazanAK07}, incur prohibitive time complexity $O(T^{24})$ due to integrating over the domain~\citep{bubeck2018sampling}.

However, the time-varying Hessian-related matrix in ONS necessitates a \emph{Mahalanobis projection} at each round to ensure feasibility, which introduces the computational bottleneck.
Specifically, the Mahalanobis projection solves the quadratic program $\Pi_\XX^M[\y] = \argmin_{\x \in \XX} (\x-\y)^\trs M (\x-\y)$ for some positive-definite and symmetric matrix $M$ and compact convex domain $\XX \subseteq \reals^d$.
When the domain $\XX$ is bounded, best-known implementations require $\tilde{O}(d^\omega)$ arithmetic operations due to matrix factorizations such as matrix square root~\citep{golub-loan-mat-compute}, where $\omega \in (2,3]$ is the matrix-multiplication exponent.
\footnote{The factorization underlying Mahalanobis projections is related to eigendecomposition, equivalent to finding roots of a $d$-degree polynomial, which is not exactly solvable by finitely many arithmetic operations when $d \ge 5$.}
Even for simple domains such as the unit ball and the probability simplex, Mahalanobis projection requires $\tilde{O}(d^\omega)$ time (see \cref{apd:alg-impl}) and $\tilde{O}(d^{\omega+0.5})$ time (via interior-point methods~\citep{interior-point-method}), respectively.
Thus, the total runtime of ONS can reach $\tilde{O}(d^\omega T)$, particularly when iterates frequently oscillate near the domain boundary.
Although \citet{mat-mul-exponent-best2024} establish $\omega < 2.3714$ theoretically, linear algebra libraries typically operate with $\omega=3$, yielding $\tilde{O}(d^3 T)$ time in practice.
In contrast, for convex and strongly convex online optimization, Online Gradient Descent (OGD)~\citep{ICML03:zinkvich} requires only Euclidean projections and achieves minimax-optimal regret with a runtime of $O(dT)$ for these simple domains.

Similar computational challenges arise in Stochastic eXp-concave Optimization (SXO).
As highlighted by a COLT'13 open problem of~\citet{koren2013open}, ONS remains the default algorithm for SXO when equipped with the online-to-batch conversion.
The optimal regret $O(d \log T)$ translates into an optimal sample complexity $T = \tilde{O}(d/\veps)$ for an excess risk of $\veps$, where $\tilde{O}(\cdot)$ hides poly-logarithmic factors in $d/\veps$.
Consequently, solving SXO with ONS incurs a total runtime of $\tilde{O}(d^{\omega+1}/\veps)$, which evaluates to $\tilde{O}(d^4/\veps)$ in practice.
The open problem asks for an SXO algorithm with runtime below $\tilde{O}(d^4/\veps)$, i.e., one that achieves both statistical optimality and computational efficiency.

\paragraph{Related work.}
The pursuit of computationally efficient OXO algorithms can be divided into two main research lines.
\rom{1} The first, which we focus on in this paper, aims to minimize runtime while preserving optimal regret.
The most closely related work is by~\citet{COLT23:OQNS}, who proposed OQNS (Online Quasi-Newton Steps), achieving the optimal regret $O(d \log T)$ with runtime $O(d^2 T \log T + d^\omega \sqrt{T \log T})$.
They employ a log-barrier to eliminate Mahalanobis projections, transferring the computational burden to Hessian-inverse evaluations under log-barrier regularization, which is similar to interior-point methods~\citep{HM10:EfficientIPMOCO}.
However, their method departs from canonical algorithmic frameworks such as online mirror descent (OMD) and Follow-the-Regularized-Leader (FTRL)~\citep{book22:FO-book}, making it difficult to accommodate diverse local norms and thereby limiting its applicability beyond standard OXO.
A concrete comparison between our method and OQNS is deferred to \cref{sec:app}.
\rom{2} The second line aims to reduce regret within a time or space budget, for example via matrix sketching techniques~\citep{NIPS16:sketch-ONS}.
These methods achieve linear-in-$d$ runtime and working memory, albeit under additional assumptions on the loss functions and domains.
This line also includes projection-free methods that trade statistical optimality for computational gains, leading to suboptimal regret bounds like $O(T^{2/3})$~\citep{colt23garber,JMLR22:ProjectionFreeDistributed}.

\begin{table}[t]
  \caption{Algorithmic upper bounds on regret and total runtime of OXO algorithms with respect to $d$ and $T$ over a simple domain of the unit ball. ``ONS-Like'' indicates whether the method can be integrated into other settings where ONS serves as the backbone, including gradient-norm adaptivity, parametric stochastic bandits, and memory-efficient OXO.}
  \label{tab:compare}
  \centering
  \begin{tabular}{cccc}
    \toprule
    \textbf{Algorithm} & \textbf{Regret} & \textbf{Total Runtime} & \textbf{ONS-Like} \\
    \midrule
    OGD~\citep{ICML03:zinkvich} & $\sqrt{T}$ & $d T$  & -- \\
    ONS~\citep{journals/ml/HazanAK07} & $d \log{T}$ & $d^2 T + d^\omega T \log T$ & -- \\
    OQNS~\citep{COLT23:OQNS} & $d \log{T}$ & $d^2 T \log T + d^\omega \sqrt{T \log T}$ & \No \\
    \midrule
    \lightons (This Paper) & $d \log{T}$ & $d^2 T + d^\omega \sqrt{T \log T}$ & \Yes \\
    \bottomrule
  \end{tabular}
\end{table}

\paragraph{Contributions.}
We propose \lightons (Light Online Newton Step), a simple variant of ONS that substantially reduces the runtime while retaining the optimal regret.
Our method preserves the elegant structure of ONS and inherits its applicability to various scenarios.
Our contributions include:
\begin{itemize}
  \item \textbf{An optimal and efficient algorithm for OXO.}~
    As summarized in \cref{tab:compare}, \lightons attains the best-known total runtime $O(d^2 T + d^\omega \sqrt{T \log T})$ to achieve the minimax-optimal regret $O(d \log T)$ for OXO.
    In terms of regret, \lightons matches ONS's dependence on all problem parameters in OXO ($T$, $d$, $D$, $G$, $\alpha$), whereas the prior method~\citep{COLT23:OQNS} suffers from large multiplicative constants.
    Empirical validations in \cref{apd:experiment} also corroborate the theoretical superiority of our method.
  \item \textbf{An optimal and efficient algorithm for SXO.}~
    Equipped with the online-to-batch conversion, \lightons yields (up to poly-logarithmic factors) the optimal sample complexity $T = \tilde{O}(d/\veps)$ for an excess risk of $\veps$, thus reducing the total runtime to $\tilde{O}(d^3/\veps)$ when $\veps=O(1/d)$.
    Our result answers a COLT'13 open problem of~\citet{koren2013open}.
    We further provide evidence that the runtime $\tilde{O}(d^3/\veps)$ is unlikely to be improved in practice.
  \item \textbf{Applicability across various scenarios.}~
    Importantly, \lightons preserves the online mirror descent (OMD) framework of ONS and inherits ONS's structural flexibility, especially in accommodating diverse local norms.
    In \cref{sec:app}, we demonstrate \lightons smoothly replaces ONS and reduces runtime in various scenarios, including gradient-norm adaptivity, parametric stochastic bandits, and memory-efficient OXO.
    In contrast, the prior method~\citep{COLT23:OQNS}, tailored for OXO, lacks the flexibility to fit these scenarios.
\end{itemize}

\paragraph{Outline.}
The remainder of this paper is organized as follows.
\cref{sec:prel} presents preliminaries.
\cref{sec:alg} introduces our method and key analytical ingredients.
\cref{sec:stochastic} discusses the implications of \lightons to SXO, which answers a COLT'13 open problem.
\cref{sec:app} demonstrates the broad applicability of \lightons inherited from ONS.
\cref{sec:conclusion} concludes the paper and discusses future directions.
Omitted technical details are deferred to the appendices.

\section{Preliminaries}
\label{sec:prel}

In this section, we introduce notations used throughout this paper, formalize the problem setting of Online eXp-concave Optimization (OXO), and review key relevant prior work.

\paragraph{Notations.}
Let $[a]_+ = \max\{0, a\}$, $[N] = \{1, \dots, N\}$, and $\BB(R) = \{ \x \mid \norme{\x} \leq R \}$.
For a positive-definite and symmetric matrix $M$, let $\lambda_i(M)$ be its $i$-th greatest eigenvalue and let $\norm{\x}_M = \sqrt{ \x^\trs M \x }$ be the Mahalanobis norm.
Let $\Pi_\XX^M[\y] = \argmin_{\x \in \XX} \norm{\x-\y}_M^2$ be the Mahalanobis projection, which exists and is unique for a convex and compact domain $\XX$~\citep{Book04:CvxOpt}, and let $\Pi_\XX[\y] = \Pi^I_\XX[\y] = \argmin_{\x \in \XX} \norme{\x-\y}^2$ be the Euclidean projection.
We write $\EP_\XX$ and $\MP_\XX$ for the runtime of Euclidean and Mahalanobis projection onto $\XX$, respectively.

\subsection{Problem Setting}
\label{sec:prel-formal}

Online Convex Optimization (OCO) unfolds as a game between a learner and an environment over $T$ rounds.
At each round $t \in [T]$, the learner selects a decision $\x_t$ from a compact convex domain $\XX \subseteq \reals^d$, and the environment simultaneously reveals a convex loss function $f_t: \XX \to \reals$;
then the learner incurs a loss $\ftxt$ and observes a gradient $\gr\ftxt$ for updates.
The performance of the learner is measured by its \emph{regret} against some comparator $\u \in \XX$, which is defined as:
\begin{equation*}
  \reg_T(\u) = \sum_{t=1}^{T} \ftxt - \sum_{t=1}^{T} \ftu .
\end{equation*}

The definition of exp-concavity~\citep{kivinen1999averaging,book/Cambridge/cesa2006prediction} and standard regularity assumptions of OXO~\citep{hazan2016introductionOCO,journals/ml/HazanAK07,COLT23:OQNS} are formally stated below.

\begin{asm}[bounded domain]
  \label{asm:bounded-domain}
  The domain $\XX \subseteq \reals^d$ is convex and compact, and has a diameter of $D$, i.e., $\max_{(\x,\y)\in\XX^2} \norme{\x-\y} \leq D$.
  For technical convenience, we further assume that $\max_{\x\in\XX} \norme{\x} \leq D/2$.
\end{asm}

\begin{asm}[bounded gradient]
  \label{asm:bounded-gradient}
  For any $t \in [T]$, the loss function $f_t: \XX \to \reals$ is differentiable and $G$-Lipschitz on $\XX$, i.e., $\max_{\x \in \XX} \norme{\gr\ftx} \leq G$.
\end{asm}

\begin{asm}[exp-concave loss]
  \label{asm:exp-concave}
  For any $t \in [T]$, the loss function $f_t: \XX \to \reals$ is $\alpha$-exp-concave on $\XX$, i.e., $\exp\sbr{- \alpha f_t(\cdot)}$ is concave on $\XX$.
\end{asm}

\subsection{Important Progress}
\label{sec:prel-progress}

In this subsection, we review two important OXO algorithms, namely Online Newton Step (ONS) \citep{journals/ml/HazanAK07} and Online Quasi-Newton Step (OQNS) \citep{COLT23:OQNS}.

\paragraph{Online Newton Step (ONS).}
\cref{alg:ons} summarizes the update of ONS and enjoys the following theoretical guarantees on its regret and runtime.

\begin{prp}[Theorem~2 of~\citet{journals/ml/HazanAK07}]
  \label{prp:ons-regret}
  Under \cref{asm:bounded-domain,asm:bounded-gradient,asm:exp-concave}, ONS (\cref{alg:ons}) with $\gamma_0 = \frac{1}{2} \min \{ \frac{1}{DG}, \alpha \}$ satisfies that, for any $\u \in \XX$,
  \begin{subequations}
    \begin{align}
      \label{eq:ons-regret}
      \reg_T(\u) & \leq \frac{d}{2\gamma_0} \log \bigg(  1 + \frac{G^2}{d \eps} T  \bigg) + \frac{\gamma_0 \eps D^2}{8} , \\
      \label{eq:ons-time}
      \runtime & \leq O( (d^2 + \MP_\XX) T ) .
    \end{align}
  \end{subequations}
\end{prp}

ONS follows the classical framework of online mirror descent (OMD)~\citep{book22:FO-book}, equipped with a time-varying Mahalanobis norm induced by the Hessian-related matrix $A_t$.
The computational bottleneck of ONS lies in Line~6 of \cref{alg:ons}, the Mahalanobis projection $\x_{t+1} =\Pi^{A_t}_{\XX}[\hat{\x}_{t+1}]$.
Once the decision $\hat{\x}_{t+1}$ exits the domain $\XX$, it must be projected back immediately.
Since each Mahalanobis projection takes $\tilde{O}(d^\omega)$ arithmetic operations and ONS projects in $O(T)$ rounds in the worst case, the crippling total runtime $\tilde{O}(d^\omega T)$ emerges, which evaluates to $\tilde{O}(d^3 T)$ in practice.

\citet{journals/ml/HazanAK07} introduce the following lemma to characterize the curvature induced by exp-concavity, which is crucial to ONS and other practical OXO algorithms.
This lemma is essentially Lemma~3 of~\citet{journals/ml/HazanAK07}, with improved constants.
Its proof is deferred to \cref{apd:alg-proof-exp-concave}.

\begin{lem}
  \label{lem:exp-concave}
  If a function $f: \XX \to \reals$ is $\alpha$-exp-concave and differentiable, then for any $(\x,\u) \in \XX^2$, $D \geq \norme{\x-\u}$, $G \geq \norme{\gr f(\x)}$, $\gamma \leq \gamma_0 = \frac{1}{2} \min \{ \frac{1}{DG}, \alpha \}$, it holds that
  \begin{equation}
    \label{eq:exp-concave-taylor}
    f(\x) - f(\u) \leq \gr f(\x)^\trs (\x-\u) - \frac{\gamma}{2} \sbr{ \gr f(\x)^\trs (\x-\u) }^2 .
  \end{equation}
\end{lem}

We remark that the preceding lemma necessitates the bounded domain assumption (\cref{asm:bounded-domain}) to achieve the optimal regret $O(d \log T)$.
In particular, a finite diameter $D$ is always required to obtain a non-zero curvature parameter $\gamma$, which reflects the curvature induced by exp-concavity.

\paragraph{Online Quasi-Newton Step (OQNS).}
OQNS~\citep{COLT23:OQNS} effectively reduces the runtime to reach the optimal asymptotic regret $O(d \log T)$, as the following proposition states.

\begin{prp}[Theorem~9 of~\citet{COLT23:OQNS}]
  \label{prp:oqns-regret}
  Under \cref{asm:bounded-domain,asm:bounded-gradient,asm:exp-concave}, OQNS (Algorithm 3 of~\citet{COLT23:OQNS}) with $\gamma = \frac{1}{2} \min \{ \frac{1}{2DG}, \alpha \}$ satisfies that, for any $\u \in \XX$,
  \begin{subequations}
    \begin{align}
      \label{eq:oqns-regret}
      \reg_T(\u) & \leq \frac{5d}{\gamma} \log\sbr{ d + T } + \frac{11DGd}{2} \log T + 3DGd , \\
      \label{eq:oqns-time}
      \runtime & \leq O \big(  (\EP_\XX + d^2 \log T) T + d^\omega \sqrt{T \log T}  \big) .
    \end{align}
  \end{subequations}
\end{prp}

OQNS eliminates Mahalanobis projections with a log-barrier, such as $\log \frac{1}{1 - \norme{\x}^2}$ for $\XX = \BB(1)$, shifting the computational burden to Hessian-inverse evaluations.
Key components of OQNS are illustrated below, with $\XX = \BB(1)$ and $G = 1$ as in~\citep{COLT23:OQNS}.
\begin{equation}
  \label{eq:oqns-update-objective}
  \begin{aligned}
    \x_{t+1} & = \x_t - \textrm{Approx}\big(  (\gr^2 \Phi_t(\x_t))^{-1} \gr \Phi_t(\x_t)  \big) , ~~ \text{where} \\
    \Phi_t(\x) & \triangleq \eta d \log \frac{1}{1 - \norme{\x}^2} + \frac{d + \eta}{2} \norme{\x}^2 + \sum_{s=1}^{t} \bigg(  \gr f_s(\x_s)^\trs \x  +  \frac{\gamma}{2} \big(\gr f_s(\x_s)^\trs (\x - \x_s)\big)^2  \bigg) .
  \end{aligned}
\end{equation}
Evaluating the Hessian-inverse still takes $O(d^\omega)$ time and OQNS mitigates this issue by approximating the Hessian-inverse gradient product with incremental updates.
The approximation procedure, $\textrm{Approx}(\cdot)$, returns within $O(d^2 \log T)$ time under proper conditions.
OQNS controls the number of exact Hessian-inverse evaluations to $O(\sqrt{T \log T})$, leading to a total runtime of $O(d^2 T \log T + d^\omega \sqrt{T \log T})$, while achieving the optimal regret $O(d \log T)$.

However, their improvement introduces large constant factors.
Specifically, when $\alpha \ge \frac{1}{DG}$, the leading term of ONS, $\log T$, carries coefficient $DGd$ in \cref{eq:ons-regret}, whereas OQNS's leading coefficient is $\frac{51}{2}DGd$ in \cref{eq:oqns-regret}.
Moreover, OQNS departs from the OMD framework of ONS, limiting its adaptability to broader scenarios, which will be discussed in \cref{sec:app}.

\section{Our Algorithm: \lightons}
\label{sec:alg}

Prior progress naturally raises a question:
Can we retain the optimality and simplicity of ONS, while achieving a total runtime that is competitive with, or even superior to, the state-of-the-art?
Our work is motivated by answering this question in the affirmative.

In this spirit, we present our algorithm \lightons.
In \cref{sec:alg-core}, to illustrate the key idea, we introduce the core algorithm that only differs from ONS by one line of code.
However, the core algorithm is essentially \emph{improper learning}.
To address this issue, we introduce the improper-to-proper conversion in \cref{sec:alg-conversion}, which yields the complete version of \lightons.

Furthermore, we empirically validate the superiority of our method, which corroborates theoretical guarantees.
Details of the experiments are deferred to \cref{apd:experiment}.

\begin{figure*}[t]
  \begin{minipage}[t]{0.44\linewidth}
    \begin{algorithm}[H]
      \caption{ONS~\citep{journals/ml/HazanAK07}}
      \label{alg:ons}
      \begin{algorithmic}[1]
        \REQUIRE preconditioner coefficient $\eps$.
        \STATE Initialize $\gamma_0 = \frac{1}{2} \min \{ \frac{1}{DG}, \alpha \}$, \\ $A_0 = \eps I$, $\x_1 = \bm{0}$.
        \vspace{0.30em}
        \FOR {$t = 1, \dots, T$}
        \STATE Observe $\gr\ftxt$.
        \STATE $A_t = A_{t-1} + \gr\ftxt \gr\ftxt^\trs$.
        \STATE $\hat{\x}_{t+1} = \x_t - \frac{1}{\gamma_0} A_t^{-1} \gr\ftxt$.
        \STATE $\x_{t+1} = \Pi^{A_t}_{\XX}[\hat{\x}_{t+1}]$.
        \vspace{1.33em}
        \ENDFOR
      \end{algorithmic}
    \end{algorithm}
  \end{minipage}
  \begin{minipage}[t]{0.52\linewidth}
    \begin{algorithm}[H]
      \caption{\lightonscore}
      \label{alg:lightonscore}
      \begin{algorithmic}[1]
        \REQUIRE preconditioner coef. $\eps$, deferral coef. $k$.
        \STATE Initialize $\gamma = \frac{1}{2} \min \{ \frac{2}{(k+1) DG}, \alpha \}$, $A_0 = \eps I$, \\ $\x_1 = \bm{0}$.
        \FOR {$t = 1, \dots, T$}
        \STATE Observe $\gr\ftxt$.
        \STATE $A_t = A_{t-1} + \gr\ftxt \gr\ftxt^\trs$.
        \STATE $\hat{\x}_{t+1} = \x_t - \frac{1}{\gamma} A_t^{-1} \gr\ftxt$.
        \STATE $\x_{t+1} =
        \begin{cases}
          \hat{\x}_{t+1} & \text{if} ~ \norme{\hat{\x}_{t+1}} \leq k D/2 \\
          \Pi^{A_t}_{\XX}[\hat{\x}_{t+1}] & \text{otherwise}
        \end{cases}$.
        \ENDFOR
      \end{algorithmic}
    \end{algorithm}
  \end{minipage}
\end{figure*}

\subsection{Amortizing Deferred Projections}
\label{sec:alg-core}

We first introduce the \lightonscore in \cref{alg:lightonscore}, which amortizes the costly Mahalanobis projections with a \emph{deferred-projection} mechanism.
\lightonscore only differs from ONS by one line of code, as shown in Line~6 of \cref{alg:ons,alg:lightonscore}.
While ONS projects onto $\XX$ immediately when the decision exits $\XX$, \lightonscore continues to update without projection outside $\XX$, and projects only when the decision exits an \emph{expanded} domain $\tilde{\XX}_k \subseteq \reals^d$ defined as
\begin{equation}
  \label{eq:expanded-domain}
  \tilde{\XX}_k \triangleq \BB(kD/2) = \{ \x \in \reals^d \mid \norm{\x}_2 \leq kD/2\},
\end{equation}
where $k > 1$ is the deferral coefficient.
Essentially, $\tilde{\XX}_k$ is obtained by scaling a Euclidean ball that contains $\XX$ by a factor of $k > 1$.

Clearly, a larger $k$ implies lower runtime, as it invokes fewer Mahalanobis projections.
The following lemma quantifies the relationship between the deferral coefficient $k$ and the number of Mahalanobis projections $N$.
Its proof is deferred to \cref{apd:alg-proof-proj-count}.

\begin{lem}
  \label{lem:proj-count}
  Under \cref{asm:bounded-domain}, and that the loss functions $\lbr{f_t}_{t=1}^T$ are $\alpha$-exp-concave, differentiable and $G$-Lipschitz on $\tilde{\XX}_k$, let $N$ denote the number of Mahalanobis projections in \lightonscore (\cref{alg:lightonscore}) over $T$ rounds, then
  \begin{equation}
    \label{eq:proj-count}
    N \leq \left\lfloor \frac{2}{(k - 1) D \gamma} \sqrt{\frac{d}{\eps} T} \right\rfloor .
  \end{equation}
\end{lem}

On the other hand, a larger $k$ exacerbates the deviation of \lightonscore from ONS, potentially harming the regret guarantee.
In particular, increasing $k$ degrades the curvature parameter $\gamma$, jeopardizing the curvature benefit of exp-concavity.
In the extreme case where $k$ approaches infinity, $\gamma$ collapses to zero, and exp-concavity degenerates to mere convexity.
The following theorem establishes the theoretical guarantees of \lightonscore, revealing the trade-off between regret and runtime induced by the deferral coefficient $k$.
Its proof is deferred to \cref{apd:alg-proof-improper-theorem}.

\begin{thm}
  \label{thm:lightonscore}
  Under \cref{asm:bounded-domain}, and that the loss functions $\lbr{f_t}_{t=1}^T$ are $\alpha$-exp-concave, differentiable and $G$-Lipschitz on $\tilde{\XX}_k$, with $\gamma = \frac{1}{2} \min \{ \frac{2}{(k+1) DG}, \alpha \}$, \lightonscore (\cref{alg:lightonscore}) satisfies that, for any $\u \in \XX$,
  \begin{subequations}
    \begin{align}
      \label{eq:lightonscore-regret}
      \reg_T(\u) & \leq \frac{d}{2\gamma} \log \bigg(  1 + \frac{G^2}{d \eps} T  \bigg) + \frac{\gamma \eps D^2}{8} , \\
      \label{eq:lightonscore-time}
      \runtime & \leq O \big(  d^2 T + (k - 1)^{-1} \sqrt{dT/\eps} \cdot \MP_\XX  \big) .
    \end{align}
  \end{subequations}
\end{thm}

\paragraph{Trade-off of the deferral coefficient $k$.}
When $k$ increases, the regret bound in~\cref{eq:lightonscore-regret} grows as $O(\frac{1}{\gamma}) = O(k+1)$, whereas the time spent on Mahalanobis projections in~\cref{eq:lightonscore-time} decreases as $O(N) = O(\frac{1}{k-1})$.
With $k = 2$, \lightonscore already achieves a significant runtime improvement over ONS without sacrificing the optimal regret.
Specifically, the regret bound grows by at most a factor of $\frac{3}{2}$, while the number of projections is substantially reduced from $O(T)$ to $O(\sqrt{T})$.

\paragraph{Improper learning issue.}
Unfortunately, \lightonscore falls in the scope of \emph{improper learning} \citep{shaishaiuml}, as the algorithm's decisions $\x_t \in \tilde{\XX}_k \supset \XX$ may reside beyond the domain while the comparator $\u \in \XX$ is strictly constrained to the domain.
Moreover, \lightonscore requires additional assumptions that the Lipschitzness and exp-concavity of the loss functions extend to the expanded domain $\tilde{\XX}_k$.
Vitally, improper learning suppresses the theoretical performance limits of its proper counterparts.
A notable illustration is online logistic regression, where an improper learner achieves the regret bound of $O(d \log(GT))$~\citep{COLT18:improper-LR}, while proper learners are limited to a regret lower bound of $\Omega(d \rme^G \log T)$~\citep{COLT14:LR}.

\begin{algorithm}[t]
  \caption{\lightons}
  \label{alg:lightons}
  \begin{algorithmic}[1]
    \REQUIRE preconditioner coefficient $\eps$, deferral coefficient $k$.
    \STATE Initialize $\gamma' = \frac{1}{2} \min \{ \frac{1}{DG}, \alpha, \frac{4}{(k+1) c_f c_g DG} \}$, $A_0 = \eps I$, $\x_1 = \y_1 = \bm{0}$.
    \FOR {$t = 1, \dots, T$}
    \STATE Observe $\gr\ftxt$; and construct $\gr\gtyt$ satisfying \cref{cdt:domain-conversion}.
    \STATE $A_t = A_{t-1} + \bgr_t \bgr_t^\trs$, where $\bgr_t = c_f \gr\gtyt$.
    \STATE $\hat{\y}_{t+1} = \y_t - \frac{1}{\gamma'} A_t^{-1} \bgr_t$.
    \STATE $\y_{t+1} =
    \begin{cases}
      \hat{\y}_{t+1} & \text{if} ~ \norme{\hat{\y}_{t+1}} \leq k D/2 \\
      \Pi^{A_t}_{\BB(D/2)}[\hat{\y}_{t+1}] & \text{otherwise}
    \end{cases}$.
    \STATE $\x_{t+1} = \Pi_{\XX}[\y_{t+1}]$.
    \ENDFOR
  \end{algorithmic}
\end{algorithm}

\subsection{The Improper-to-Proper Conversion}
\label{sec:alg-conversion}

We introduce \lightons in \cref{alg:lightons}, which inherits the favorable regret-runtime trade-off of \lightonscore while ensuring proper learning over the domain $\XX$, thanks to domain-conversion techniques from parameter-free online learning~\citep{COLT18:black-box-reduction,ICML20:Ashok}.
The improper-to-proper conversion is conceptually straightforward yet technically subtle, as the surrogate loss preserves the curvature benefit of exp-concavity even though it is not exp-concave.

The conversion works by constructing surrogate losses and projecting surrogate decisions.
At each round, \lightons constructs a surrogate loss $g_t: \reals^d \to \reals$ and supplies $g_t$ to an underlying \lightonscore;
then the underlying \lightonscore outputs a surrogate decision $\y_t \in \tilde{\XX}_k$, which \lightons maps to a proper true decision $\x_t \in \XX$ via a Euclidean projection.

Any surrogate loss satisfying the condition below guarantees a valid conversion.

\begin{cdt}
  \label{cdt:domain-conversion}
  For some $c_f \geq 1$ and $c_g \geq 1$, the surrogate loss function $g_t: \reals^d \to \reals$ satisfies that, for any $\u \in \XX$, $\norme{\gr\gtyt} \leq c_g \norme{\gr\ftxt}$ and $\gr\ftxt^\trs (\x_t-\u) \leq c_f \gr\gtyt^\trs (\y_t-\u)$.
\end{cdt}

We note that two such conversions have been proposed in the literature.
The first, by~\citet{COLT18:black-box-reduction}, achieves $c_f = 2$ and $c_g = 1$.
Later, an improved conversion by~\citet{ICML20:Ashok} achieves $c_f = c_g = 1$.
We prefer the latter which yields smaller constants and tighter regret.

\begin{lem}[Theorem~2 of~\citet{ICML20:Ashok}]
  \label{lem:domain-conversion}
  Under \cref{asm:bounded-domain,asm:bounded-gradient}, let $\x_t = \Pi_\XX[\y_t]$, the surrogate loss function $g_t: \reals^d \to \reals$ and its subgradient at $\y_t$ are defined as follows:
  \begin{subequations}
    \begin{align}
      \label{eq:func-g}
      \gty & \triangleq \gr\ftxt^\trs \y  + \frac{[- \gr\ftxt^\trs (\y_t-\x_t)]_{+}}{\norme{\y_t-\x_t}} \norme{ \y-\Pi_\XX[\y] } ,
      \\
      \label{eq:grad-g}
      \gr\gtyt & = \gr\ftxt + \frac{[- \gr\ftxt^\trs (\y_t-\x_t)]_{+}}{\norme{\y_t-\x_t}^2} \sbr{ \y_t-\x_t } ,
    \end{align}
  \end{subequations}
  then for any $\u \in \XX$, $\norme{\gr\gtyt} \leq \norme{\gr\ftxt}$, and $\gr\ftxt^\trs (\x_t-\u) \leq \gr\gtyt^\trs (\y_t-\u)$.
\end{lem}

\paragraph{Side effects of the conversion.}
The computational overhead of the conversion is negligible.
In each round, \cref{eq:grad-g} constructs the surrogate gradient $\gr\gtyt$ with only $O(d)$ time apart from the Euclidean projection, dominated by the $O(d^2)$ cost of updating $A_t$ and $A_t^{-1}$ (see \cref{apd:alg-impl}).

A notable observation is that the surrogate loss $g_t$ in \cref{lem:domain-conversion} is \emph{not} exp-concave, seemingly precluding the curvature benefit.
Fortunately, the following lemma shows that the surrogate loss $g_t$ inherits the curvature in a form closely mirroring \cref{lem:exp-concave}, with proof in~\cref{apd:alg-proof-surrogate-taylor}.

\begin{lem}
  \label{lem:surrogate-taylor}
  Under \cref{asm:bounded-domain,asm:bounded-gradient,asm:exp-concave}, with $\gamma' = \frac{1}{2} \min \{ \frac{1}{DG}, \alpha, \frac{4}{(k+1) c_f c_g DG} \}$, for any $\u \in \XX$ and any surrogate loss function $g_t$ satisfying \cref{cdt:domain-conversion}, let $\bgr_t = c_f \gr\gtyt$, then
  \begin{equation}
    \label{eq:surrogate-taylor}
    \begin{aligned}
      \ftxt - \ftu & \leq \gr\ftxt^\trs (\x_t-\u) - \frac{\gamma_0}{2} \sbr{ \gr\ftxt^\trs (\x_t-\u) }^2 \\
      & \leq \bgr_t^\trs (\y_t-\u) - \frac{\gamma'}{2} \sbr{ \bgr_t^\trs (\y_t-\u) }^2 .
    \end{aligned}
  \end{equation}
\end{lem}

We remark that, when the deferral coefficient $k \leq 3$ and $c_f = c_g = 1$ (e.g., in \cref{lem:domain-conversion}), the surrogate curvature parameter $\gamma'$ in \cref{lem:surrogate-taylor} is \emph{unimpaired} relative to the original $\gamma_0$ in \cref{lem:exp-concave}.

It is worth noting that prior work has also employed this domain-conversion technique to address projection-related issues, though for different purposes, including non-stationary online learning~\citep{JMLR25:efficient-nonstationary} and universal online learning~\citep{NeurIPS24:universal-1-projection}.

The following theorem establishes the theoretical guarantees of \lightons, and an immediate corollary specifies its regret and runtime by instantiating deferral coefficient $k=2$ and the surrogate loss in \cref{lem:domain-conversion}.
Its proof is deferred to \cref{apd:alg-proof-main-theorem}.

\begin{thm}
  \label{thm:lightons}
  Under \cref{asm:bounded-domain,asm:bounded-gradient,asm:exp-concave}, and that each surrogate gradient $\gr\gtyt$ takes $O(\EP_\XX + d)$ time, with $\gamma' = \frac{1}{2} \min \{ \frac{1}{DG}, \alpha, \frac{4}{(k+1) c_f c_g DG} \}$, \lightons (\cref{alg:lightons}) satisfies that, for any $\u \in \XX$,
  \begin{subequations}
    \begin{align}
      \label{eq:lightons-regret}
      \reg_T(\u) & \leq \frac{d}{2\gamma'} \log\bigg( 1 + \frac{c_f^2 c_g^2 G^2}{d \eps} T \bigg) + \frac{\gamma' \eps D^2}{8} , \\
      \label{eq:lightons-time}
      \runtime & \leq O \big(  (\EP_\XX + d^2) T + (k - 1)^{-1} d^{\omega + 0.5} \sqrt{T/\eps} \log T  \big) .
    \end{align}
  \end{subequations}
\end{thm}

\begin{cor}
  \label{cor:lightons}
  Under \cref{asm:bounded-domain,asm:bounded-gradient,asm:exp-concave}, using the surrogate loss function in \cref{lem:domain-conversion}, $k=2$, and $\eps = 1 + d \log T$, with $\gamma_0 = \frac{1}{2} \min \{ \frac{1}{DG}, \alpha \}$, \lightons (\cref{alg:lightons}) satisfies that, for any $\u \in \XX$,
  \begin{subequations}
    \begin{align}
      \label{eq:lightons-regret-specified}
      \reg_T(\u) & \leq \frac{d}{2\gamma_0} \log \bigg( 1 + \frac{G^2}{d \eps} T \bigg) + \frac{\gamma_0 \eps D^2}{8}, \\
      \label{eq:lightons-time-specified}
      \runtime & \leq O \big(  (\EP_\XX + d^2) T + d^{\omega} \sqrt{T \log T}  \big) .
    \end{align}
  \end{subequations}
\end{cor}

\begin{rmk}
  \lightons achieves a runtime better than both ONS and OQNS, at the cost of a regret slightly worse than ONS, as seen in the second inequality in \cref{eq:surrogate-taylor}.
  Nonetheless, this degradation is dominated by the problem-dependent parameters of OXO, namely $T$, $d$, $D$, $G$, $\alpha$.
  In their dependence on these parameters, \lightons's regret in \cref{eq:lightons-regret-specified} exactly matches ONS's regret in \cref{eq:ons-regret}.
\end{rmk}

\begin{rmk}
  \lightons retains the flexible OMD framework of ONS.
  In fact, \lightons differs from ONS only in two aspects: \rom{1} the deferred-projection mechanism (introduced in \lightonscore to boost efficiency) and \rom{2} the improper-to-proper conversion (introduced here to ensure proper learning), both of which are largely \emph{orthogonal} to the mirror-descent update in the standard ONS.
  Consequently, \lightons applies to various scenarios where ONS is essential, particularly when its mirror-descent update plays a critical role.
  We illustrate these applications in \cref{sec:app}.
\end{rmk}

\paragraph{Numerical implementation.}
The runtime in \cref{eq:lightons-time-specified} arises from efficient numerical primitives, detailed in \cref{apd:alg-impl}.
The term $O(d^2 T)$ reflects rank-one updates of $A_t^{-1}$, which avoids the $O(d^\omega)$ cost of inverting $A_t$ from scratch.
The term $O(d^\omega \sqrt{T \log T})$ accounts for the infrequent Mahalanobis projections onto $\BB(D/2)$:
each projection reduces to a one-dimensional root-finding problem solvable by bisection in $O(d^\omega \log T)$ time, and the deferred-projection mechanism caps the total number of such projections, so that their cumulative cost remains sublinear in $T$.

\section{Answering a COLT'13 Open Problem}
\label{sec:stochastic}

Via the online-to-batch conversion, our method applies to Stochastic eXp-concave Optimization (SXO), where the optimal OXO regret translates into the (near) optimal SXO sample complexity while substantially reducing the computational cost.
This extension answers a COLT'13 open problem posed by~\citet{koren2013open}, demonstrating our method's significance beyond OXO.

\subsection{Restatement of the Open Problem}
\label{sec:stochastic-restate}

SXO seeks to minimize an exp-concave function $F: \XX \to \reals$ over a convex domain $\XX \subseteq \reals^d$, and the learner has access to $F$ only through some stochastic oracle.
In this section, we consider the case where $F$ is the expectation of a random function $f: \XX \times \Xi \to \reals$ over a distribution on $\Xi$, and the learner has access to the gradient of $f(\cdot; \xi)$ with $\xi$ drawn from $\DD_\xi$.
We are interested in the sample complexity and total runtime required to find an $\veps$-optimal solution $\x_\veps$, i.e.,
\begin{equation*}
  F(\x_\veps) - \min_{\x \in \XX} F(\x) \leq \veps, \quad \text{where} ~ F(\x) = \expectation{\xi \sim \DD_\xi}{f(\x; \xi)}.
\end{equation*}
The regularity of the random function is formally stated below.

\begin{asm}
  \label{asm:exp-concave-stochastic}
  The loss function $F: \XX \to \reals$ is the expectation of a random function $f: \XX \times \Xi \to \reals$ over an unknown distribution $\DD_\xi$ on $\Xi$, i.e., $F(\x) = \expectation{\xi \sim \DD_\xi}{f(\x; \xi)}$.
  For any $\xi \in \Xi$, the random loss function $f(\cdot; \xi)$ is $\alpha$-exp-concave, differentiable and $G$-Lipschitz over $\XX$.
  For any query point $\x \in \XX$, the stochastic gradient oracle returns $\nabla f(\x; \xi)$ with $\xi$ i.i.d. drawn from $\DD_\xi$.
\end{asm}

\citet{koren2013open} notes that, with the online-to-batch conversion, ONS's $O(d \log T)$ regret for OXO implies a sample complexity of $\tilde{O}(d/\veps)$ and a total runtime of $\tilde{O}(d^4/\veps)$ for SXO,
\footnote{From a theoretical perspective, the total runtime of applying ONS to SXO is $\tilde{O}(d^{\omega+1}/\veps)$, as discussed in \cref{sec:intro}.
We adopt $\omega=3$ following the statement of the open problem~\citep{koren2013open}, with emphasis on implementability.}
where $\tilde{O}(\cdot)$ hides poly-logarithmic factors in $d/\veps$.
This quartic dependence on $d$ is even more pronounced in comparison to stochastic strongly convex optimization, where Online Gradient Descent (OGD) implies a total runtime $\tilde{O}(d/\veps)$~\citep{hazan2014stochasticstronglyconvex}, motivating the following open problem.

\begin{center}
  \shadowbox{
    \begin{minipage}[t]{0.9\columnwidth}
      \paragraph{Open problem~\citep{koren2013open}.}
      Under \cref{asm:exp-concave-stochastic} and that $\XX=\BB(1)$,
      \vspace{1mm}
      \begin{enumerate}[noitemsep,topsep=0pt]
        \item[(a)]
          Is it possible to find an SXO algorithm that attains the sample complexity of $\tilde{O}(d/\veps)$ with only linear-in-$d$ runtime per iteration, i.e., $\tilde{O}(d^2/\veps)$ runtime overall?
          \vspace{1mm}
        \item[(b)]
          Is it possible to perform any better than $\tilde{O}(d^4/\veps)$ runtime overall?
      \end{enumerate}
    \end{minipage}
  }
\end{center}

The first part of the open problem remains open.
In particular, \citet{COLT15:exp-concave} establish an information-theoretic sample complexity lower bound of $\Omega(d/\veps)$ for SXO, which implies a runtime lower bound of $\Omega(d^2/\veps)$ since each gradient query costs $\Omega(d)$ time.

\cref{sec:stochastic-guarantee} answers the second part in the affirmative by combining \lightons with the online-to-batch conversion, obtaining total runtime $\tilde{O}(d^3/\veps)$ while maintaining the optimal $\tilde{O}(d/\veps)$ sample complexity.
In \cref{sec:stochastic-contribution}, we further present evidence that this $\tilde{O}(d^3/\veps)$ runtime is likely unimprovable.
If confirmed, it would refute the first part and settle the full open problem.

\subsection{Answering the Open Problem with \lightons}
\label{sec:stochastic-guarantee}

Built upon the online-to-batch conversion of~\citet{AISTATS17:exp-concave-high-prob}, the \lightons-based SXO method reduces the total runtime to $\tilde{O}(d^3/\veps)$ for $\veps = O(1/d)$.
It is reasonable to only consider $\veps = O(1/d)$. Indeed, when $\veps = \Omega(1/d)$, the OGD-based SXO method achieves both a better sample complexity of $T = \tilde{O}(1/\veps^2)$ and a better runtime of $\tilde{O}(d/\veps^2)$.
The \lightons-based SXO method also achieves the (near) optimal convergence rate both with high probability and in expectation, as the following theorem shows.
Its proof is deferred to \cref{apd:stochastic-proof-lightons}.

\begin{thm}
  \label{thm:lightons-stochastic}
  Under \cref{asm:exp-concave-stochastic} and that $\XX=\BB(1)$,
  let \lightons (\cref{alg:lightons}) run for $T$ rounds with gradients $\{ \nabla f(\x_t; \xi_t) \}_{t=1}^{T}$ where $\{ \x_t \}_{t=1}^{T}$ are decisions,
  let $\bar{\x}_T = \frac{1}{T} \sum_{t=1}^{T} \x_t$, then
  \vspace{1mm}
  \begin{itemize}[noitemsep,topsep=0pt]
    \item For any $\delta \in (0,1)$, with probability at least $1-\delta$, when $T = \Theta(\frac{d}{\veps} \log \frac{d}{\veps} \log \frac{1}{\delta})$, \\ $F(\bar{\x}_T) - \min_{\x \in \XX} F(\x) \leq O(\veps)$ and $\runtime \leq \tilde{O}(\frac{d^3}{\veps} + \frac{d^{3.5}}{\sqrt{\veps}})$. \vspace{1mm}
    \item When $T' = \Theta(\frac{d}{\veps} \log \frac{d}{\veps})$, $\expectation{}{F(\bar{\x}_{T'}) - \min_{\x \in \XX} F(\x)} \leq O(\veps)$ and $\runtime \leq \tilde{O}(\frac{d^3}{\veps} + \frac{d^{3.5}}{\sqrt{\veps}})$.
  \end{itemize}
\end{thm}

\subsection{Discussions on SXO}
\label{sec:stochastic-contribution}

The \lightons-based SXO method in \cref{thm:lightons-stochastic} matches the best-known SXO runtime up to poly-logarithmic factors.
An online-to-batch conversion of OQNS also yields a total runtime of $\tilde{O}(d^3/\veps)$ but carries significant practical overheads due to its large constant factors, as discussed in \cref{sec:alg}.
Our experiments in \cref{apd:experiment}, conducted under SXO settings, confirm this difference.

We conjecture that no SXO algorithm can asymptotically beat total runtime $\tilde{O}(d^3/\veps)$.
A natural alternative approach to SXO is empirical risk minimization (ERM)~\citep{NIPS15:fast-rate-exp-concave,AISTATS17:exp-concave-high-prob}, which reduces SXO to offline exp-concave optimization and admits offline acceleration techniques.
The following two observations indicate that ERM-based SXO methods are unlikely to surpass the $\tilde{O}(d^3/\veps)$ runtime barrier, even when these methods use more working memory $\Omega(d^2 + d/\veps)$ to store all samples, whereas the OXO-based methods require only $O(d^2)$ memory.
\begin{itemize}
  \item \textbf{Fast matrix multiplication.}~
    Linear regression with random design, a special case of SXO, reduces to solving a linear system with $d$ variables and $\tilde{O}(d/\veps)$ equations.
    Although fast matrix multiplication accelerates this to $\tilde{O}(d^\omega/\veps)$ runtime~\citep{Ibarra-Moran-Hui-Reduction}, this runtime does not plausibly extend beyond linear regression and reverts to $\tilde{O}(d^3/\veps)$ in practice.
  \item \textbf{Cutting-plane methods.}~
    The best-known cutting-plane methods of~\citet{lee2015cuttingplane,jiang2020cuttingplane}, as detailed in \cref{apd:stochastic-proof-erm}, solve the offline exp-concave optimization problem to $O(\veps)$-accuracy in $\tilde{O}(d^3/\veps)$ time.
    We remark that cutting-plane methods further assume well-roundedness of the domain, which may bring additional computational overhead.
\end{itemize}

Moreover, we give an intuitive explanation of the runtime barrier $\tilde{O}(d^3/\veps)$ in \cref{apd:stochastic-proof-erm}.

\begin{table}[t]
  \caption{Regret and total runtime comparison between ONS and \lightons on three benchmark applications. ``OQNS?'' indicates whether OQNS supports the same application.}
  \label{tab:app-compare}
  \centering
  \small
  \begin{tabular}{cccccc}
    \toprule
    \textbf{Application}
    & \makecell{\textbf{Regret}\\\textbf{(ONS)}}
    & \makecell{\textbf{Runtime}\\\textbf{(ONS)}}
    & \makecell{\textbf{Regret}\\\textbf{(\lightons)}}
    & \makecell{\textbf{Runtime}\\\textbf{(\lightons)}}
    & \textbf{OQNS?} \\
    \midrule
    \makecell{Gradient-norm\\adaptivity}
    & $O(d \log G_T)$
    & $\tilde{O}(d^\omega T)$
    & same
    & $O(d^2 T + d^\omega \sqrt{T} \log T)$
    & hardly \\ \midrule
    \makecell{Generalized\\linear bandits}
    & $\tilde{O}(d\sqrt{T} + \kappa d^2)$
    & \makecell{$\tilde{O}(d^2KT\,+$\\$d^\omega T)$}
    & same order
    & \makecell{$\tilde{O}(d^2KT + d^\omega\,\cdot$\\$\min\{\sqrt{\kappa dT}\log\kappa,T\})$}
    & hardly \\ \midrule
    \makecell{Memory-efficient\\OXO}
    & $O(d' \log T)$
    & $\tilde{O}(d^\omega T)$
    & same
    & $\tilde{O}(d' d T + d^\omega \sqrt{T})$
    & intricate \\
    \bottomrule
  \end{tabular}
\end{table}

\section{Applications to Various Problems}
\label{sec:app}

ONS extends its influence far beyond the classical settings of OXO and SXO.
Since \lightons is designed to retain core updates of ONS with minimal modifications, it preserves its elegant structure.
This enables seamless integration into a range of applications where ONS serves as the computational core, not only maintaining its statistical advantages but also significantly enhancing efficiency.

We highlight three representative applications of ONS (gradient-norm adaptivity, parametric stochastic bandits, and memory-efficient OXO) where \lightons seamlessly fits, preserving statistical benefits while requiring minimal additional analytical effort.
In contrast, deploying OQNS in these settings is either infeasible or would require substantial and non-trivial modifications.
The main results are summarized in \cref{tab:app-compare}, comparing ONS with \lightons and OQNS.
Due to space constraints, we defer details and proofs to \cref{apd:app-proof}.

\subsection{Gradient-Norm Adaptivity}

As shown by~\citet{aistats12:exp-concave-smooth}, ONS achieves the following \emph{problem-dependent} regret bound that scales with the accumulated squared gradient norms $G_T$ instead of the time horizon $T$:
\begin{equation}
  \label{eq:gradient-norm-adaptive-oxo}
  \reg_T(\u) = O\sbr{ d \log G_T } , \quad
  \text{where} ~ G_T \triangleq \sum_{t=1}^T \norme{\gr\ftxt}^2 ,
\end{equation}
where $G_T$ can be $o(T)$ in benign environments, yielding regret far below worst-case bounds that scale with $T$.
Yet, since $G_T \leq G^2 T$, the gradient-norm adaptive bound safeguards the optimal regret $O(d \log T)$ against worst cases.
Prior work has leveraged \cref{eq:gradient-norm-adaptive-oxo} to achieve strong guarantees in various settings, including small-loss adaptivity for OXO with smoothness and comparator-norm adaptivity for unbounded OCO, as discussed below.

\paragraph{OXO with smoothness.}
\citet{aistats12:exp-concave-smooth} show that the gradient-norm adaptivity in~\cref{eq:gradient-norm-adaptive-oxo} can be enhanced to small-loss adaptivity~\citep{NIPS10:smooth,NIPS20:sword} in~\cref{eq:small-loss-quantity} under smoothness assumptions on loss functions $\lbr{f_t}_{t=1}^T$,
\begin{equation}
  \label{eq:small-loss-quantity}
  \reg_T = O\sbr{ d \log L_T } , \quad
  \text{where} ~ L_T \triangleq \min_{\u \in \XX} \sum_{t=1}^T \sbr{ f_t(\u) - \min_{\x \in \XX} f_t(\x) } ,
\end{equation}
where the regret scales with the cumulative loss of the best comparator in hindsight.

\paragraph{Unbounded OCO.}
For OCO with unbounded domains, i.e., $\XX = \reals^d$, \citet{COLT18:black-box-reduction} achieve a comparator-norm adaptive regret bound without prior knowledge of the comparator norm $\norme{\u}$ based on the gradient-norm adaptivity in~\cref{eq:gradient-norm-adaptive-oxo},
\begin{equation}
  \label{eq:unbounded-comparator-norm}
  \reg_T(\u) = \tilde{O}\sbr{\norme{\u} \sqrt{d G_T}} , \quad
  \text{for any unbounded comparator} ~ \u \in \reals^d .
\end{equation}
Such guarantees for unbounded comparators are unachievable via classical OCO algorithms that explicitly depend on the domain diameter $D$, for example, OGD demands an explicit $D$ to obtain $O(D G \sqrt{T})$ regret~\citep{ICML03:zinkvich,COLT08:lower-bound}.

\paragraph{Improvements by \lightons.}
Nevertheless, both~\citep{aistats12:exp-concave-smooth,COLT18:black-box-reduction} incur the worst-case $\tilde{O}(d^\omega T)$ runtime bottleneck of ONS.
Replacing ONS in these pipelines, \lightons preserves the same regret bounds while improving the runtime to $O(d^2 T + d^\omega \sqrt{T} \log T)$.

In contrast, OQNS can hardly achieve gradient-norm adaptivity, let alone the small-loss adaptivity and comparator-norm adaptivity.
Its computational efficiency relies on the log-barrier, which inevitably introduces an $O(\log T)$ term which precludes adaptation to $G_T$.
The OMD framework permits a flexible trade-off between stability and bias in the regret decomposition.
While \lightons retains this flexibility, the log-barrier in OQNS overly suppresses the stability term, leading to a large bias term that does not scale with the gradient norms.
Details are deferred to \cref{apd:app-proof-norm}.

\subsection{Parametric Stochastic Bandits}

Parametric stochastic bandits model decision-making problems with partial feedback, and generalized linear bandits (GLB) is a fundamental instance where the expected loss depends on an unknown parameter through a known link function.
This formulation introduces non-linearities into linear bandits, enhancing their expressivity while retaining tractable solutions.

\paragraph{Current results.}
A notable challenge of GLB is the dependence on the condition number $\kappa \propto \exp(D)$ in regret bounds.
Directly applying ONS to GLB yields an $\tilde{O}(\kappa d \sqrt{T})$ regret~\citep{NIPS17:online-GLB}, which becomes vacuous for $\kappa = \Omega(\sqrt{T})$.
Recently, \citet{ZYJ2025LogB} propose the first \emph{jointly efficient} GLB algorithm, which achieves an $\tilde{O}(d \sqrt{T} + \kappa d^2)$ regret (statistically efficient) and is one-pass (computationally efficient).
Their key technique is an ONS-based subroutine with carefully designed non-monotonic local norms, which enjoys the one-pass efficiency of ONS and eliminates $\kappa$ from leading terms in regret.

\paragraph{Improvements by \lightons.}
By replacing the ONS-based subroutine with a \lightons-based counterpart, we retain the same regret bound and improve the runtime from $\tilde{O}((d^2 K + d^\omega) T)$ to $\tilde{O}(d^2 K T + d^\omega \cdot \min\{\sqrt{\kappa dT} \log\kappa, T\})$, where $K$ is the number of arms for bandits.
When the regret is sublinear in $T$, i.e., $\kappa = o(T)$, our improvement yields a strictly lower runtime order.

OQNS can hardly be adapted to the method of~\citet{ZYJ2025LogB}.
The reasons are two-fold:
\rom{1} The GLB analysis treats the ONS-based subroutine as a white-box, exploiting negative terms in the OMD analysis, while OQNS fails to provide such analytical properties due to its deviation from OMD;
\rom{2} The GLB method relies on non-monotonic local norms, while OQNS is tailored for monotonic local norms as in \cref{eq:oqns-update-objective}, to which OQNS's analysis is deeply coupled.
On the contrary, the key ingredients of \lightons (deferred projection and domain conversion) are largely orthogonal to the core mirror-descent updates, seamlessly accommodating customized local norms and preserving necessary analytical properties for GLB.
Details are deferred to \cref{apd:app-proof-bandits}.

\subsection{Memory-Efficient OXO}

Another computational challenge of ONS is its $O(d^2)$ working memory, in contrast to OGD's $O(d)$ working memory.
\citet{NIPS16:sketch-ONS} propose Sketched Online Newton Step (SON), which mitigates this issue by incorporating matrix sketching into ONS.

\paragraph{Current results.}
SON achieves linear-in-$d$ runtime and working memory when the sketched dimension $d' \ll d$.
The value of $d'$ typically depends on the problem's intrinsic dimensionality, such as the number of non-zero eigenvalues in the Hessian-related matrix.
However, SON demands additional assumptions.
Its domain $\XX_t$ must be an intersection of two parallel half-spaces, onto which Mahalanobis projections admit closed-form with $O(d^2)$ time.
Under the standard OXO setting (\cref{asm:bounded-domain,asm:bounded-gradient,asm:exp-concave}), SON loses its computational advantage and reverts to the high computational cost of ONS.
The domain restriction further imposes a stronger assumption on loss functions.
The curvature parameter $\gamma_0 = \frac{1}{2} \min \{\frac{1}{DG}, \alpha\}$ in \cref{lem:exp-concave} collapses to zero for the unbounded intersection-of-parallel-half-spaces domain, reducing exp-concavity to convexity.
Consequently, SON assumes an explicit quadratic property mirroring \cref{lem:exp-concave}, rather than the standard \cref{asm:exp-concave}.

\paragraph{Improvements by \lightons.}
We propose \lightson (\cref{alg:lightson} in \cref{apd:app-proof-sketch}), which replaces \lightons's accesses to the Hessian-related matrix $A_t$ with the sketching primitives as in SON.
\lightson combines the projection efficiency of \lightons with the memory efficiency of SON.
This hybrid method extends linear-in-$d$ runtime and working memory to the standard OXO setting, while retaining SON's regret $O(d' \log T)$.
Theoretical guarantees of the hybrid method are summarized in the following theorem.

\begin{thm}[\lightons's improvement for memory-efficient OXO]
  \label{thm:lightson}
  Under \cref{asm:bounded-domain,asm:bounded-gradient,asm:exp-concave}, \\ with $\gamma_0 = \frac{1}{2} \min \{ \frac{1}{DG}, \alpha \}$, \lightson (\cref{alg:lightson}) satisfies that, for any $\u \in \XX$,
  \begin{subequations}
    \begin{align}
      \label{eq:lightson-regret}
      \reg_T(\u) & \leq \frac{d'}{\gamma_0} \log \sbr{1 + \frac{G^2}{2d' \eps} T} + \frac{\gamma_0 \eps D^2}{8} + \frac{\Delta_{1:T}}{2\gamma_0} , \\
      \label{eq:lightson-time}
      \runtime & \leq O\sbr{  \sbr{\EP_\XX + d' d \log T} T + d^\omega \sqrt{(d + \Delta_{1:T}) T/\eps} \log T  }  ,
    \end{align}
  \end{subequations}
  and the working memory is $O(d'd)$.
  With the surrogate gradient $\gr\gtyt$ as in \cref{lem:domain-conversion}, the cumulative sketching error $\Delta_{1:T}$ can be bounded as
  \begin{equation}
    \label{eq:lightson-delta}
    \Delta_{1:T} \leq \min_{j \in [d']} \frac{2d'}{(d'-j+1)\eps} \sum_{i=j}^{d} \lambda_i\sbr{\sum_{t=1}^{T} \gr\gtyt \gr\gtyt^\trs} .
  \end{equation}
\end{thm}

Although \citet{COLT23:OQNS} suggest integrating sketching into OQNS, their algorithm's considerable structural complexity may pose significant obstacles to algorithmic design and theoretical analysis.
Details are deferred to \cref{apd:app-proof-sketch}.

\section{Conclusion}
\label{sec:conclusion}

For online exp-concave optimization (OXO), we propose \lightons, a simple yet powerful variant of ONS.
By combining the deferred-projection mechanism with the improper-to-proper conversion, \lightons achieves significant reductions in worst-case total runtime while preserving the regret optimality of ONS.
These gains extend to the stochastic optimization setting, answering a COLT'13 open problem on efficient and optimal SXO~\citep{koren2013open}.
Moreover, due to its fidelity to the mirror-descent update of ONS, \lightons serves as an efficient drop-in replacement of ONS across diverse applications, including gradient-norm adaptivity, parametric stochastic bandits, and memory-efficient settings, all without compromising statistical guarantees.

Several important directions remain open.
First, while \lightons adapts to gradient norms, achieving adaptivity to \emph{gradient variation} for OXO within $\tilde{O}(d^2 T)$ runtime remains an open challenge~\citep{COLT12:variation-Yang,JMLR24:Sword++official}.
Second, applying the \lightons technique to other preconditioned online learning algorithms may not improve their (asymptotic) runtime if the bottleneck lies in other operations, such as explicit matrix factorization rather than Mahalanobis projections (e.g., in AdaGrad~\citep{duchi2011adaptive}).
Accelerating these algorithms remains unclear.
These future directions motivate the development of broader, general-purpose acceleration techniques for Hessian-related online learning algorithms.

\acks{This work was supported by National Science and Technology Major Project (2022ZD0114800), NSFC (62361146852), the Fundamental and Interdisciplinary Disciplines Breakthrough Plan of the Ministry of Education of China (No. JYB2025XDXM118), and the ``111 Center'' (No. B26023).}

\bibliography{./online_learning.bib}

\begin{thebibliography}{52}
\providecommand{\natexlab}[1]{#1}
\providecommand{\url}[1]{\texttt{#1}}
\expandafter\ifx\csname urlstyle\endcsname\relax
  \providecommand{\doi}[1]{doi: #1}\else
  \providecommand{\doi}{doi: \begingroup \urlstyle{rm}\Url}\fi

\bibitem[Abernethy et~al.(2008)Abernethy, Bartlett, Rakhlin, and Tewari]{COLT08:lower-bound}
Jacob Abernethy, Peter~L Bartlett, Alexander Rakhlin, and Ambuj Tewari.
\newblock Optimal strategies and minimax lower bounds for online convex games.
\newblock In \emph{Proceedings of the 21st Annual Conference on Learning Theory (COLT)}, pages 415--423, 2008.

\bibitem[Alman et~al.(2025)Alman, Duan, Williams, Xu, Xu, and Zhou]{mat-mul-exponent-best2024}
Josh Alman, Ran Duan, Virginia~Vassilevska Williams, Yinzhan Xu, Zixuan Xu, and Renfei Zhou.
\newblock More asymmetry yields faster matrix multiplication.
\newblock In \emph{Proceedings of the 2025 Annual ACM-SIAM Symposium on Discrete Algorithms (SODA)}, pages 2005--2039, 2025.

\bibitem[Boyd and Vandenberghe(2004)]{Book04:CvxOpt}
Stephen Boyd and Lieven Vandenberghe.
\newblock \emph{{Convex} {Optimization}}.
\newblock Cambridge University Press, 2004.

\bibitem[Bubeck et~al.(2018)Bubeck, Eldan, and Lehec]{bubeck2018sampling}
S{\'e}bastien Bubeck, Ronen Eldan, and Joseph Lehec.
\newblock Sampling from a log-concave distribution with projected langevin monte carlo.
\newblock \emph{Discrete \& Computational Geometry}, 59:\penalty0 757--783, 2018.

\bibitem[Cai et~al.(2010)Cai, Zhang, and Zhou]{cai2010covariance}
T~Tony Cai, Cun-Hui Zhang, and Harrison~H Zhou.
\newblock Optimal rates of convergence for covariance matrix estimation.
\newblock \emph{The Annals of Statistics}, 38\penalty0 (4):\penalty0 2118--2144, 2010.

\bibitem[Cesa-Bianchi and Lugosi(2006)]{book/Cambridge/cesa2006prediction}
Nicol\`{o} Cesa-Bianchi and G{\'a}bor Lugosi.
\newblock \emph{Prediction, {L}earning, and {G}ames}.
\newblock Cambridge {U}niversity {P}ress, 2006.

\bibitem[Chiang et~al.(2012)Chiang, Yang, Lee, Mahdavi, Lu, Jin, and Zhu]{COLT12:variation-Yang}
Chao-Kai Chiang, Tianbao Yang, Chia-Jung Lee, Mehrdad Mahdavi, Chi-Jen Lu, Rong Jin, and Shenghuo Zhu.
\newblock Online optimization with gradual variations.
\newblock In \emph{Proceedings of the 25th Conference On Learning Theory (COLT)}, pages 6.1--6.20, 2012.

\bibitem[Cover(1991)]{cover1991universal}
Thomas~M Cover.
\newblock Universal portfolios.
\newblock \emph{Mathematical Finance}, 1\penalty0 (1):\penalty0 1--29, 1991.

\bibitem[Cutkosky(2020)]{ICML20:Ashok}
Ashok Cutkosky.
\newblock Parameter-free, dynamic, and strongly-adaptive online learning.
\newblock In \emph{Proceedings of the 37th International Conference on Machine Learning (ICML)}, pages 2250--2259, 2020.

\bibitem[Cutkosky and Orabona(2018)]{COLT18:black-box-reduction}
Ashok Cutkosky and Francesco Orabona.
\newblock Black-box reductions for parameter-free online learning in {Banach} spaces.
\newblock In \emph{Proceedings of the 31st Conference on Learning Theory (COLT)}, pages 1493--1529, 2018.

\bibitem[Duchi et~al.(2011)Duchi, Hazan, and Singer]{duchi2011adaptive}
John Duchi, Elad Hazan, and Yoram Singer.
\newblock Adaptive subgradient methods for online learning and stochastic optimization.
\newblock \emph{Journal of Machine Learning Research}, 12\penalty0 (7):\penalty0 2121--2159, 2011.

\bibitem[Foster(1991)]{foster1991OLR}
Dean~P Foster.
\newblock Prediction in the worst case.
\newblock \emph{The Annals of Statistics}, pages 1084--1090, 1991.

\bibitem[Foster and Simchowitz(2020)]{ICML20:log-control}
Dylan~J. Foster and Max Simchowitz.
\newblock Logarithmic regret for adversarial online control.
\newblock In \emph{Proceedings of the 37th International Conference on Machine Learning (ICML)}, pages 3211--3221, 2020.

\bibitem[Foster et~al.(2018)Foster, Kale, Luo, Mohri, and Sridharan]{COLT18:improper-LR}
Dylan~J Foster, Satyen Kale, Haipeng Luo, Mehryar Mohri, and Karthik Sridharan.
\newblock Logistic regression: The importance of being improper.
\newblock In \emph{Proceedings of 31st Conference on Learning Theory (COLT)}, pages 167--208, 2018.

\bibitem[Garber and Kretzu(2023)]{colt23garber}
Dan Garber and Ben Kretzu.
\newblock Projection-free online exp-concave optimization.
\newblock In \emph{Proceedings of the 36th Annual Conference on Learning Theory (COLT)}, pages 1259--1284, 2023.

\bibitem[Ghashami et~al.(2016)Ghashami, Liberty, Phillips, and Woodruff]{SICOMP16:frequent-direct}
Mina Ghashami, Edo Liberty, Jeff~M. Phillips, and David~P. Woodruff.
\newblock Frequent directions: Simple and deterministic matrix sketching.
\newblock \emph{{SIAM} Journal on Computing}, 45\penalty0 (5):\penalty0 1762--1792, 2016.

\bibitem[Golub and Van~Loan(2013)]{golub-loan-mat-compute}
Gene~H Golub and Charles~F Van~Loan.
\newblock \emph{{Matrix} {Computations}}.
\newblock JHU Press, 2013.

\bibitem[Hazan(2016)]{hazan2016introductionOCO}
Elad Hazan.
\newblock {Introduction} to {Online} {Convex} {Optimization}.
\newblock \emph{Foundations and Trends{\textregistered} in Optimization}, 2\penalty0 (3-4):\penalty0 157--325, 2016.

\bibitem[Hazan and Kale(2011)]{hazan2014stochasticstronglyconvex}
Elad Hazan and Satyen Kale.
\newblock Beyond the regret minimization barrier: An optimal algorithm for stochastic strongly-convex optimization.
\newblock In \emph{Proceedings of the 24th Annual Conference on Learning Theory (COLT)}, pages 421--436, 2011.

\bibitem[Hazan and Megiddo(2023)]{HM10:EfficientIPMOCO}
Elad Hazan and Nimrod Megiddo.
\newblock An efficient interior-point method for online convex optimization.
\newblock \emph{ArXiv preprint}, arXiv:2307.11668, 2023.

\bibitem[Hazan et~al.(2007)Hazan, Agarwal, and Kale]{journals/ml/HazanAK07}
Elad Hazan, Amit Agarwal, and Satyen Kale.
\newblock Logarithmic regret algorithms for online convex optimization.
\newblock \emph{Machine Learning}, 69\penalty0 (2-3):\penalty0 169--192, 2007.

\bibitem[Hazan et~al.(2014)Hazan, Koren, and Levy]{COLT14:LR}
Elad Hazan, Tomer Koren, and Kfir~Y Levy.
\newblock Logistic regression: Tight bounds for stochastic and online optimization.
\newblock In \emph{Proceedings of 27th Conference on Learning Theory (COLT)}, pages 197--209, 2014.

\bibitem[Ibarra et~al.(1982)Ibarra, Moran, and Hui]{Ibarra-Moran-Hui-Reduction}
Oscar~H Ibarra, Shlomo Moran, and Roger Hui.
\newblock A generalization of the fast {LUP} matrix decomposition algorithm and applications.
\newblock \emph{{Journal} of {Algorithms}}, 3\penalty0 (1):\penalty0 45--56, 1982.

\bibitem[Jiang et~al.(2020)Jiang, Lee, Song, and Wong]{jiang2020cuttingplane}
Haotian Jiang, Yin~Tat Lee, Zhao Song, and Sam Chiu-wai Wong.
\newblock An improved cutting plane method for convex optimization, convex-concave games, and its applications.
\newblock In \emph{Proceedings of the 52nd Annual ACM SIGACT Symposium on Theory of Computing (STOC)}, pages 944--953, 2020.

\bibitem[Jun et~al.(2017)Jun, Bhargava, Nowak, and Willett]{NIPS17:online-GLB}
Kwang-Sung Jun, Aniruddha Bhargava, Robert~D. Nowak, and Rebecca Willett.
\newblock Scalable generalized linear bandits: Online computation and hashing.
\newblock In \emph{Advances in Neural Information Processing Systems 30 (NIPS)}, pages 99--109, 2017.

\bibitem[Kivinen and Warmuth(1999)]{kivinen1999averaging}
Jyrki Kivinen and Manfred~K Warmuth.
\newblock Averaging expert predictions.
\newblock In \emph{Proceedings of 4th European Conference on Computational Learning Theory (EuroCOLT)}, pages 153--167, 1999.

\bibitem[Koren(2013)]{koren2013open}
Tomer Koren.
\newblock Open problem: Fast stochastic exp-concave optimization.
\newblock In \emph{Proceedings of the 26th Annual Conference on Learning Theory (COLT)}, pages 1073--1075, 2013.

\bibitem[Koren and Levy(2015)]{NIPS15:fast-rate-exp-concave}
Tomer Koren and Kfir~Y. Levy.
\newblock Fast rates for exp-concave empirical risk minimization.
\newblock In \emph{Advances in Neural Information Processing Systems 28 (NIPS)}, pages 1477--1485, 2015.

\bibitem[Lee et~al.(2015)Lee, Sidford, and Wong]{lee2015cuttingplane}
Yin~Tat Lee, Aaron Sidford, and Sam Chiu-wai Wong.
\newblock A faster cutting plane method and its implications for combinatorial and convex optimization.
\newblock In \emph{Proceedings of the 56th Annual Symposium on Foundations of Computer Science (FOCS)}, pages 1049--1065, 2015.

\bibitem[Luo et~al.(2016)Luo, Agarwal, Cesa{-}Bianchi, and Langford]{NIPS16:sketch-ONS}
Haipeng Luo, Alekh Agarwal, Nicol{\`{o}} Cesa{-}Bianchi, and John Langford.
\newblock Efficient second order online learning by sketching.
\newblock In \emph{Advances in Neural Information Processing Systems 29 (NIPS)}, pages 902--910, 2016.

\bibitem[Mahdavi et~al.(2015)Mahdavi, Zhang, and Jin]{COLT15:exp-concave}
Mehrdad Mahdavi, Lijun Zhang, and Rong Jin.
\newblock Lower and upper bounds on the generalization of stochastic exponentially concave optimization.
\newblock In \emph{Proceedings of the 28th Conference on Learning Theory (COLT)}, pages 1305--1320, 2015.

\bibitem[Mehta(2017)]{AISTATS17:exp-concave-high-prob}
Nishant Mehta.
\newblock Fast rates with high probability in exp-concave statistical learning.
\newblock In \emph{Proceedings of the 20th International Conference on Artificial Intelligence and Statistics (AISTATS)}, pages 1085--1093, 2017.

\bibitem[Mhammedi and Gatmiry(2023)]{COLT23:OQNS}
Zakaria Mhammedi and Khashayar Gatmiry.
\newblock Quasi-newton steps for efficient online exp-concave optimization.
\newblock In \emph{Proceedings of The 36th Annual Conference on Learning Theory (COLT)}, pages 4473--4503, 2023.

\bibitem[Nesterov and Nemirovskii(1994)]{interior-point-method}
Yurii~E. Nesterov and Arkadii Nemirovskii.
\newblock \emph{{Interior}-{Point} {Polynomial} {Algorithms} in {Convex} {Programming}}.
\newblock {SIAM}, 1994.

\bibitem[Orabona(2019)]{book22:FO-book}
Francesco Orabona.
\newblock A {M}odern {I}ntroduction to {O}nline {L}earning.
\newblock \emph{ArXiv preprint}, arxiv:1912.13213, 2019.

\bibitem[Orabona et~al.(2012)Orabona, Cesa-Bianchi, and Gentile]{aistats12:exp-concave-smooth}
Francesco Orabona, Nicol\`{o} Cesa-Bianchi, and Claudio Gentile.
\newblock Beyond logarithmic bounds in online learning.
\newblock In \emph{Proceedings of the 15th International Conference on Artificial Intelligence and Statistics (AISTATS)}, pages 823--831, 2012.

\bibitem[Ordentlich and Cover(1998)]{ordentlich-cover-1998-portfolio}
Erik Ordentlich and Thomas~M Cover.
\newblock The cost of achieving the best portfolio in hindsight.
\newblock \emph{Mathematics of Operations Research}, 23\penalty0 (4):\penalty0 960--982, 1998.

\bibitem[Parlett(1998)]{parlett1998symmetriceigen}
Beresford~N Parlett.
\newblock \emph{{The} {Symmetric} {Eigenvalue} {Problem}}.
\newblock SIAM, 1998.

\bibitem[Petersen and Petersen(2012)]{Book12:Matrix}
Kaare Petersen and Michael Petersen.
\newblock \emph{{The} {Matrix} {Cookbook}}.
\newblock 2012.
\newblock URL \url{https://ece.uwaterloo.ca/\~ece602/MISC/matrixcookbook.pdf}.

\bibitem[Shalev-Shwartz and Ben-David(2014)]{shaishaiuml}
Shai Shalev-Shwartz and Shai Ben-David.
\newblock \emph{{Understanding} {Machine} {Learning}: From {Theory} to {Algorithms}}.
\newblock Cambridge University Press, 2014.

\bibitem[Simchowit(2020)]{NIPS20:max-control}
Max Simchowit.
\newblock Making non-stochastic control (almost) as easy as stochastic.
\newblock In \emph{Advances in Neural Information Processing Systems 33 (NeurIPS)}, pages 18318--18329, 2020.

\bibitem[Srebro et~al.(2010)Srebro, Sridharan, and Tewari]{NIPS10:smooth}
Nathan Srebro, Karthik Sridharan, and Ambuj Tewari.
\newblock Smoothness, low noise and fast rates.
\newblock In \emph{Advances in Neural Information Processing Systems 23 (NIPS)}, pages 2199--2207, 2010.

\bibitem[van~der Hoeven et~al.(2023)van~der Hoeven, Zhivotovskiy, and Cesa-Bianchi]{arXiv23:high-prob-risk-bound}
Dirk van~der Hoeven, Nikita Zhivotovskiy, and Nicol{\`o} Cesa-Bianchi.
\newblock High-probability risk bounds via sequential predictors.
\newblock \emph{ArXiv preprint}, arXiv:2308.07588, 2023.

\bibitem[Vovk(1997)]{vovk1997OLR}
Volodya Vovk.
\newblock Competitive on-line linear regression.
\newblock In \emph{Advances in Neural Information Processing Systems 10 (NIPS)}, pages 364--370, 1997.

\bibitem[Wan et~al.(2022)Wan, Wang, Tu, and Zhang]{JMLR22:ProjectionFreeDistributed}
Yuanyu Wan, Guanghui Wang, Wei-Wei Tu, and Lijun Zhang.
\newblock Projection-free distributed online learning with sublinear communication complexity.
\newblock \emph{Journal of Machine Learning Research}, 23\penalty0 (172):\penalty0 1--53, 2022.

\bibitem[Yang et~al.(2024)Yang, Wang, Zhao, and Zhang]{NeurIPS24:universal-1-projection}
Wenhao Yang, Yibo Wang, Peng Zhao, and Lijun Zhang.
\newblock Universal online convex optimization with 1 projection per round.
\newblock In \emph{Advances in Neural Information Processing Systems 37 (NeurIPS)}, pages 31438--31472, 2024.

\bibitem[Zhang et~al.(2016)Zhang, Yang, Jin, Xiao, and Zhou]{ICML16:Zhang-one-bit}
Lijun Zhang, Tianbao Yang, Rong Jin, Yichi Xiao, and Zhi-Hua Zhou.
\newblock Online stochastic linear optimization under one-bit feedback.
\newblock In \emph{Proceedings of the 33rd International Conference on Machine Learning (ICML)}, pages 392--401, 2016.

\bibitem[Zhang et~al.(2025)Zhang, Xu, Zhao, and Sugiyama]{ZYJ2025LogB}
Yu-Jie Zhang, Sheng-An Xu, Peng Zhao, and Masashi Sugiyama.
\newblock Generalized linear bandits: Almost optimal regret with one-pass update.
\newblock In \emph{Advances in Neural Information Processing Systems 38 (NeurIPS)}, pages 69244--69277, 2025.

\bibitem[Zhao et~al.(2020)Zhao, Zhang, Zhang, and Zhou]{NIPS20:sword}
Peng Zhao, Yu-Jie Zhang, Lijun Zhang, and Zhi-Hua Zhou.
\newblock Dynamic regret of convex and smooth functions.
\newblock In \emph{Advances in Neural Information Processing Systems 33 (NeurIPS)}, pages 12510--12520, 2020.

\bibitem[Zhao et~al.(2024)Zhao, Zhang, Zhang, and Zhou]{JMLR24:Sword++official}
Peng Zhao, Yu-Jie Zhang, Lijun Zhang, and Zhi-Hua Zhou.
\newblock Adaptivity and non-stationarity: Problem-dependent dynamic regret for online convex optimization.
\newblock \emph{Journal of Machine Learning Research}, 25\penalty0 (98):\penalty0 1--52, 2024.

\bibitem[Zhao et~al.(2025)Zhao, Xie, Zhang, and Zhou]{JMLR25:efficient-nonstationary}
Peng Zhao, Yan-Feng Xie, Lijun Zhang, and Zhi-Hua Zhou.
\newblock Efficient methods for non-stationary online learning.
\newblock \emph{Journal of Machine Learning Research}, 25\penalty0 (208):\penalty0 1--66, 2025.

\bibitem[Zinkevich(2003)]{ICML03:zinkvich}
Martin Zinkevich.
\newblock Online convex programming and generalized infinitesimal gradient ascent.
\newblock In \emph{Proceedings of the 20th International Conference on Machine Learning (ICML)}, pages 928--936, 2003.

\end{thebibliography}

\newpage
\appendix
\crefalias{section}{appendix}
\crefalias{subsection}{appendix}
\crefalias{subsubsection}{appendix}
\section{Empirical Validation}
\label{apd:experiment}

We conduct numerical experiments to validate the theoretical guarantees of \lightons, especially its negligible statistical gap from ONS and its non-asymptotic statistical advantages over OQNS.

\paragraph{Setup.}
We evaluate on two fundamental tasks, linear and logistic regression, over the Euclidean ball domain $\XX = \BB(D/2) \subset \reals^d$ over a time horizon of $T$.
With $\{ (\x_t, y_t) \}_{t=1}^T$ sampled i.i.d. from the folded standard Gaussian distribution, i.e., each entry of $\x_t$ and $y_t$ is the absolute value of a standard Gaussian random variable, the (online) loss functions are designed as
\begin{equation*}
  \ell^{\text{linear}}_t(\w) = \frac{1}{2} \sbr{ \sqrt{\frac{G}{D}} \x_t^\trs \w + \frac{\sqrt{DG}}{2} y_t }^2 , \quad
  \ell^{\text{logistic}}_t(\w) = \log \sbr{ 1 + \exp \sbr{ G \x_t^\trs \w } } .
\end{equation*}
Both loss functions are $G$-Lipschitz and $\alpha$-exp-concave with $\alpha = 1/(DG)$ and $\alpha = \exp(-DG)$, respectively.
Our logistic regression setup slightly differs from the standard formulation, as we omit the binary label $y_t \in \{0,1\}$ for simplicity, which does not affect exp-concavity.

Three algorithms, \lightons, ONS, and OQNS, are configured with their theoretically optimal parameters.
The implementation of ONS and \lightons follows \cref{alg:ons,alg:lightons}, respectively;
OQNS is implemented in accordance with Algorithm~3 of~\citet{COLT23:OQNS}.

\paragraph{Results.}
\cref{fig:experiment} displays all results averaged over 5 independent runs with the same seeds, where the averaged performance is shown as a dark line, while individual runs are in transparent lines.

Our experiments confirm that computational gains of \lightons incur negligible loss in statistical performance.
As shown in \cref{fig:lin-ir,fig:log-ir,fig:lin-regret,fig:log-regret}, curves of \lightons (red dotted) overlap with those of ONS (green solid).
This provides empirical evidence that \lightons retains the sharp regret constants of ONS.
In contrast, OQNS (blue dashed) shows noticeably higher regret, indicating that its theoretical bounds may carry greater constants.

\begin{figure}[!h]
  \centering

  \subfigure[Instantaneous regret.\label{fig:lin-ir}]{
    \includegraphics[width=0.3\textwidth]{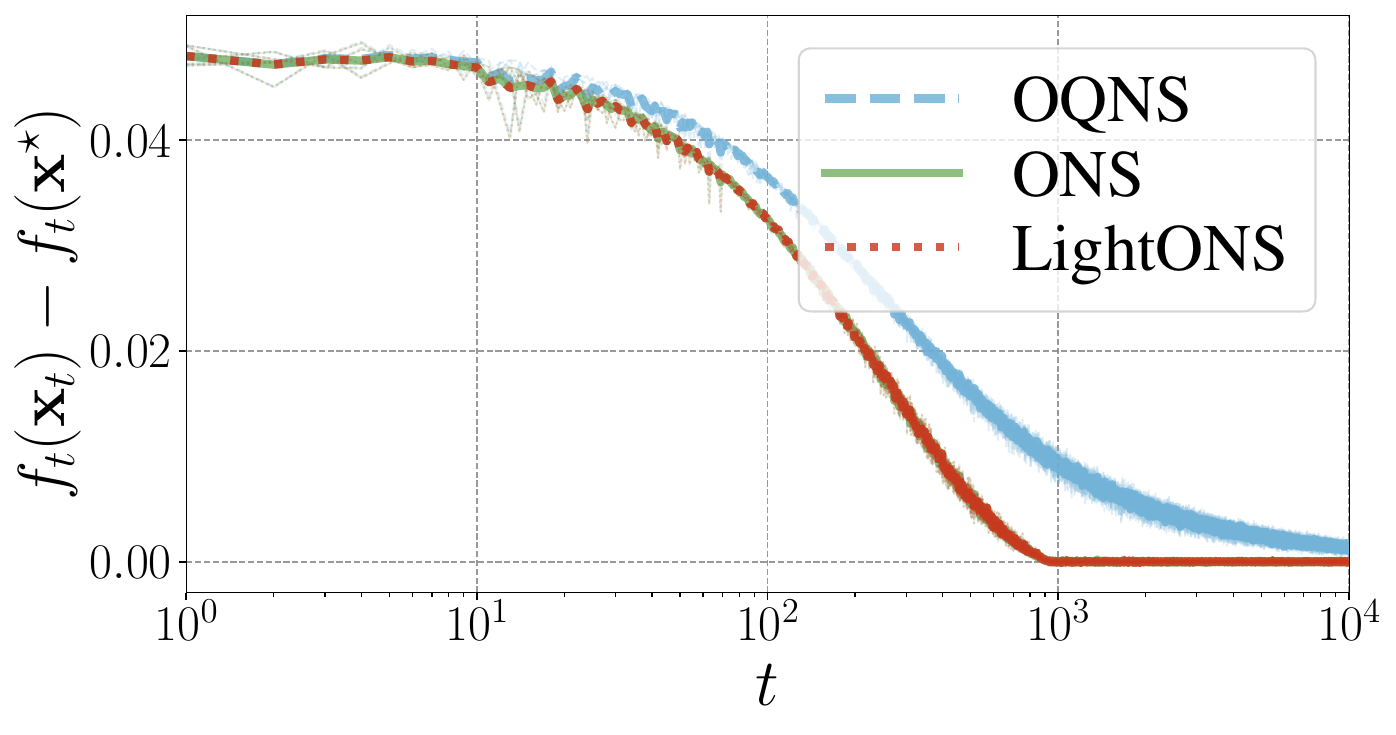}
  }
  \subfigure[Regret.\label{fig:lin-regret}]{
    \includegraphics[width=0.3\textwidth]{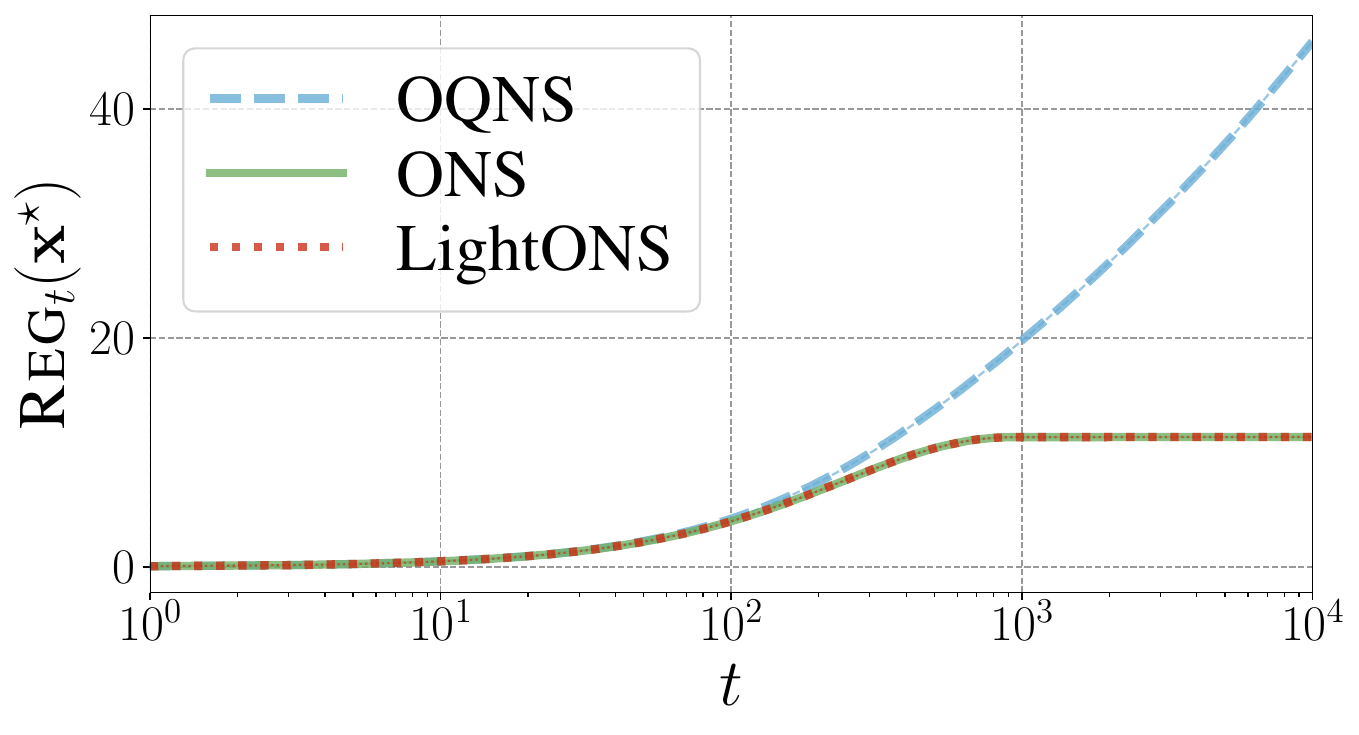}
  }
  \subfigure[Norm.\label{fig:lin-norm}]{
    \includegraphics[width=0.3\textwidth]{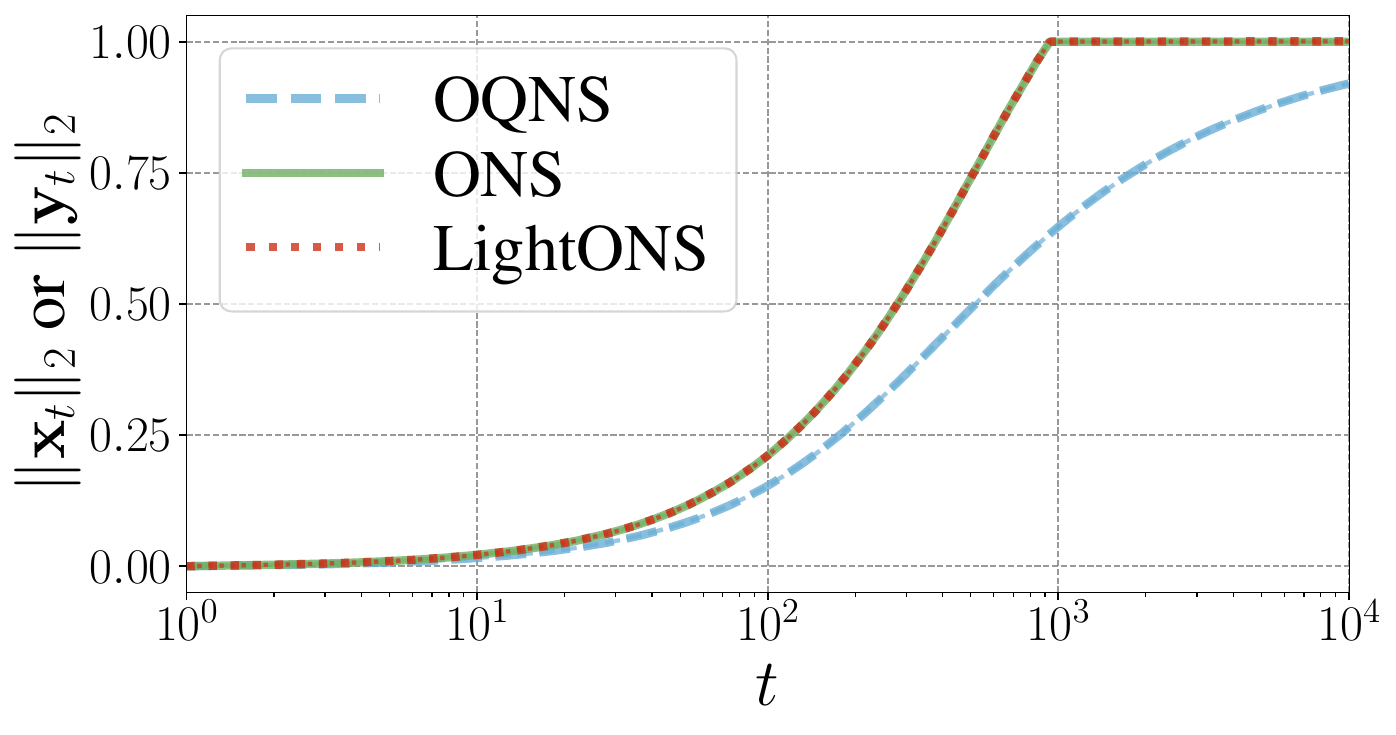}
  }

  \subfigure[Instantaneous regret.\label{fig:log-ir}]{
    \includegraphics[width=0.3\textwidth]{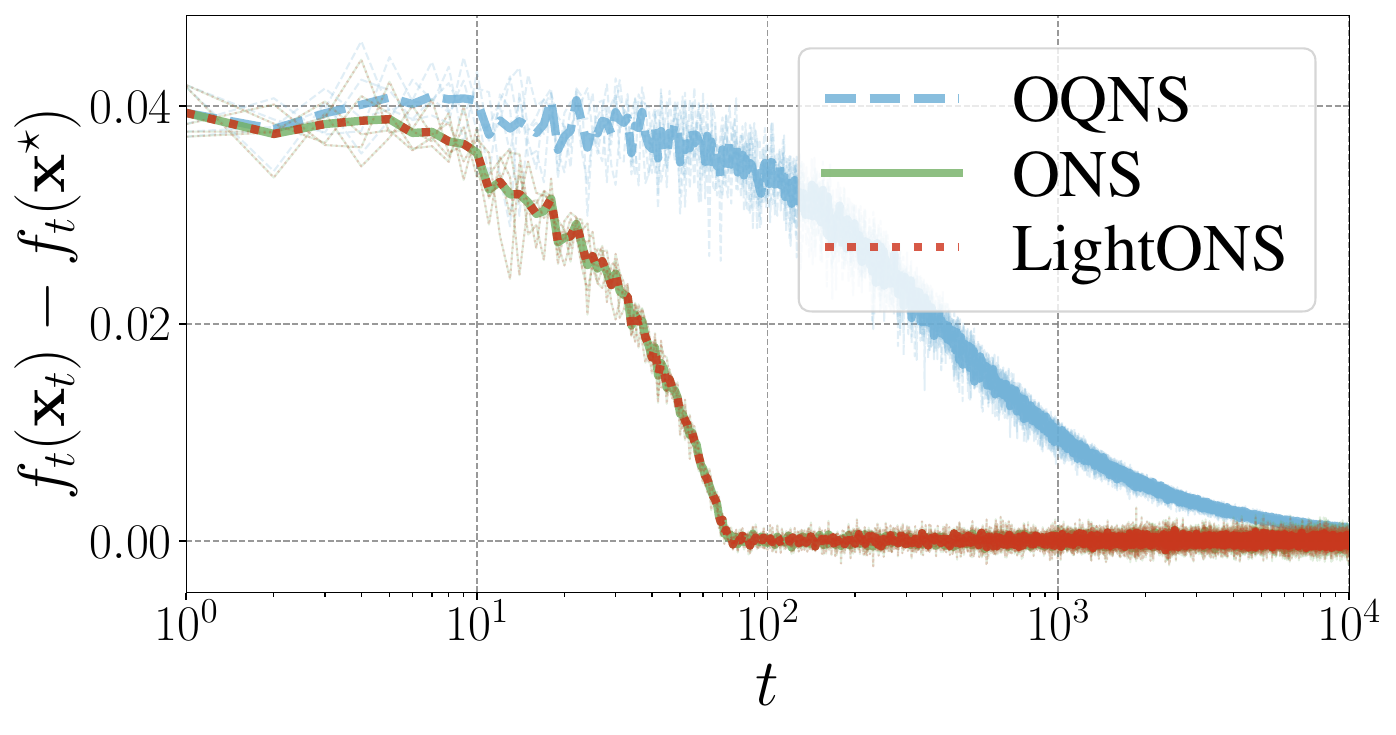}
  }
  \subfigure[Regret.\label{fig:log-regret}]{
    \includegraphics[width=0.3\textwidth]{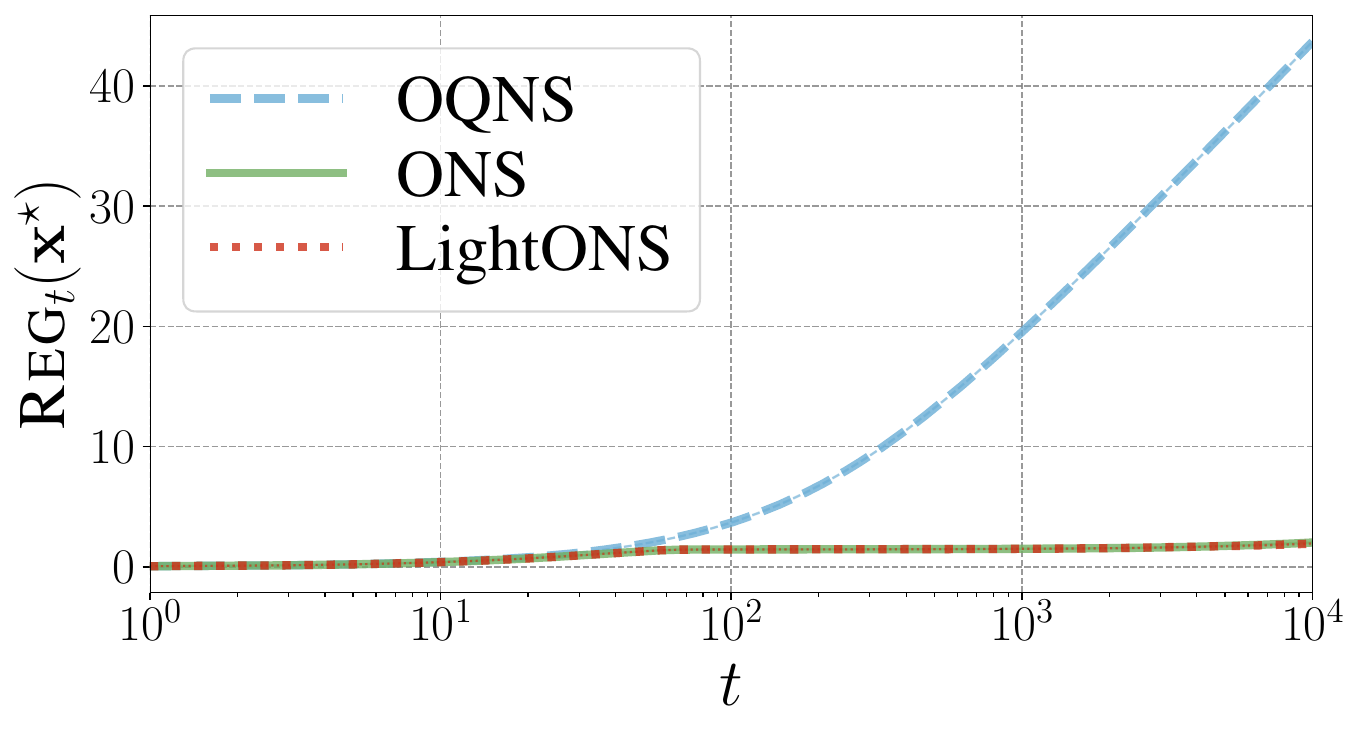}
  }
  \subfigure[Norm.\label{fig:log-norm}]{
    \includegraphics[width=0.3\textwidth]{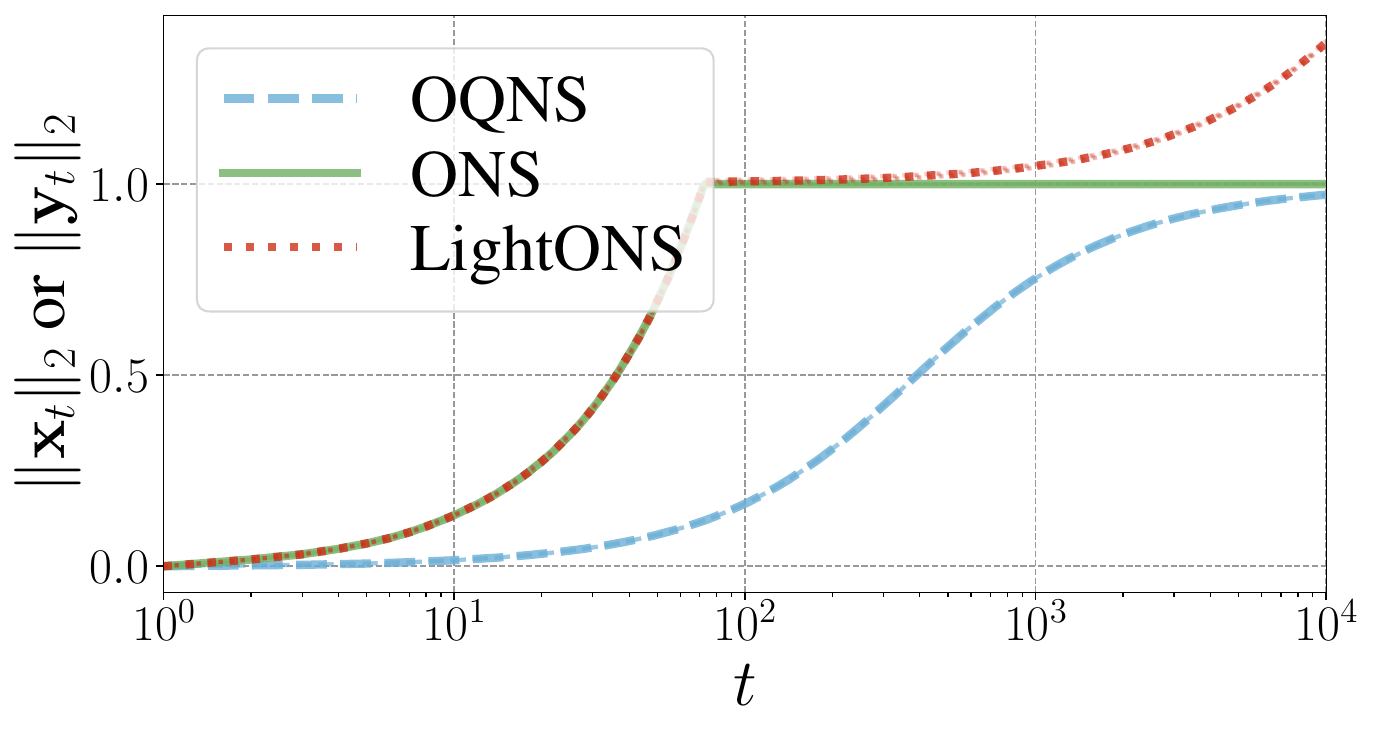}
  }

  \caption{Numerical results with $T=10^4$, $d=10$, and $D=2$. The first row shows linear regression results with $G=1/10$, $\alpha=5$ and the second row shows logistic regression results with $G=1/10$, $\alpha=\exp\sbr{-1/5}$.}
  \label{fig:experiment}
\end{figure}

The key advantage of \lightons, its computational efficiency, is evident in \cref{fig:lin-norm,fig:log-norm}.
After an initial convergence phase (e.g., the first $10^2$ rounds in \cref{fig:log-norm}), both ONS and \lightons stabilize near the offline optimal solution.
Then ONS frequently pushes decisions outside the domain, triggering costly Mahalanobis projections.
In contrast, \lightons consistently outputs decisions with $\norm{\y_t} \leq kD/2 = 2$ and avoids Mahalanobis projections.

\section{Technical Lemmas}
\label{apd:technical-lemmas}

In this section, we present technical lemmas used in the subsequent analysis.

\subsection{Proof of \texorpdfstring{\cref{lem:exp-concave}}{Lemma~\ref{lem:exp-concave}}}
\label{apd:alg-proof-exp-concave}

\begin{proof}
  Since $g(\x) = \rme^{- \alpha f(\x)}$ is concave and positive, we have that $h(\x) = \rme^{- 2 \gamma_0 f(\x)}$ is also concave.
  Specifically, for any $\lambda \in [0,1]$ and any $(\x,\y) \in \XX^2$, it holds that
  \begin{equation*}
    \begin{aligned}
      h(\lambda \x + (1-\lambda) \y) & = g(\lambda \x + (1-\lambda) \y)^{2\gamma_0/\alpha} \geq \sbr{\lambda g(\x) + (1-\lambda) g(\y)}^{2\gamma_0/\alpha} \\
      & \geq \lambda g(\x)^{2\gamma_0/\alpha} + (1-\lambda) g(\y)^{2\gamma_0/\alpha} = \lambda h(\x) + (1-\lambda) h(\y) ,
    \end{aligned}
  \end{equation*}
  where the first inequality follows from the concavity of $g$ and the second from the concavity of $a \mapsto a^{2\gamma_0/\alpha}$ (where $2\gamma_0/\alpha \leq 1$).
  Next, by the concavity of $h$, for any $(\x,\u) \in \XX^2$, we have
  \begin{equation*}
    \rme^{-2\gamma_0 f(\u)} \leq \rme^{-2\gamma_0 f(\x)} + \sbr{ -2\gamma_0 \rme^{-2\gamma_0 f(\x)} \gr f(\x) }^\trs (\u-\x) .
  \end{equation*}
  Rearranging the preceding inequality yields
  \begin{equation*}
    f(\x)-f(\u) \leq \frac{1}{2\gamma_0} \log\sbr{ 1 - 2\gamma_0 \gr f(\x)^\trs (\u-\x) } .
  \end{equation*}
  To prove \cref{eq:exp-concave-taylor}, it suffices to apply that $\log(1+a) \leq a - a^2/4$ (which holds for any $\abs{a} \leq 1$) to the right-hand side of the preceding inequality.
  We note that $\abs{2\gamma_0 \gr f(\x)^\trs (\x-\u)} \leq 1$ due to the selection of the curvature parameter $\gamma_0 = \frac{1}{2} \min \{ \frac{1}{DG}, \alpha \}$.
\end{proof}

\subsection{Elliptical Potential Lemmas}
\label{apd:alg-proof-elliptical-potential}

We introduce two lemmas tackling Hessian-related matrices, which are essential for both regret and runtime analysis.
\cref{lem:log-det} is standard in the ONS literature (e.g., \citep{journals/ml/HazanAK07,NIPS16:sketch-ONS}).
The proof of \cref{lem:neg-tr} mirrors that of \cref{lem:log-det}, differing only in the potential function: $F(A) = -\ttr{A^{-1}}$ versus $F(A) = \log \tdet{A}$.

\begin{lem}
  \label{lem:log-det}
  Let $A_i = \lambda I + \sum_{j=1}^{i} \v_j \v_j^\trs$, if $\norme{\v_i} \leq L$ for any $i \in [n]$, then
  \begin{equation}
    \label{eq:log-det}
    \sum_{i=1}^{n} \v_i^\trs A_i^{-1} \v_i \leq \log\tdet{A_n} - \log\tdet{A_0} \leq d \log\sbr{ 1 + \frac{L^2}{d \lambda} n }.
  \end{equation}
\end{lem}

\begin{lem}
  \label{lem:neg-tr}
  Let $A_i = \lambda I + \sum_{j=1}^{i} \v_j \v_j^\trs$, then
  \begin{equation}
    \label{eq:neg-tr}
    \sum_{i=1}^{n} \v_i^\trs A_i^{-2} \v_i \leq \ttr{A_0^{-1}}-\ttr{A_n^{-1}} \leq \frac{d}{\lambda} .
  \end{equation}
\end{lem}

\begin{proof}\textbf{of \cref{lem:log-det,lem:neg-tr}}
  Let $\mathbb{S}_{++}$ denote the set of $d$-dimensional positive-definite and symmetric matrices, and let $\innerf{\cdot}{\cdot}$ denote the inner product induced by the Frobenius matrix norm.
  Then for any concave function $F: \mathbb{S}_{++} \to \reals$, we have
  \begin{equation}
    \label{eq:general-elliptical-potential-telescope}
    \sum_{i=1}^{n} \v_i^\trs \gr{F}(A_i) \v_i = \sum_{i=1}^{n} \innerf{\gr{F}(A_i)}{A_i-A_{i-1}} \leq \sum_{i=1}^{n} \sbr{F(A_i) - F(A_{i-1})} = F(A_n) - F(A_0) .
  \end{equation}
  The first equality follows from $A_i - A_{i-1} = \v_i \v_i^\trs$ (by the definition of $A_i$) and $\u^\trs A \v = \innerf{A}{\v\u^\trs}$ (by the definition of the $\innerf{\cdot}{\cdot}$);
  The second inequality from the concavity of $F$;
  The last equality from telescoping.
  Substituting the respective potentials into \cref{eq:general-elliptical-potential-telescope} yields \cref{lem:log-det,lem:neg-tr}.
  Concretely,
  \begin{itemize}
    \item For \cref{lem:log-det}, $F(A) = \log \tdet{A}$ and $\gr{F}(A) = A^{-1}$.
    \item For \cref{lem:neg-tr}, $F(A) = -\ttr{A^{-1}}$ and $\gr{F}(A) = A^{-2}$.
  \end{itemize}
  Concavity of $F$ is standard~\citep{Book04:CvxOpt} and gradients of $F$ follow from matrix calculus~\citep{Book12:Matrix}.
\end{proof}

\begin{algorithm}[t]
  \caption{\fastproj onto Euclidean ball $\BB(R)$ with Mahalanobis norm $\norm{\cdot}_A$}
  \label{alg:mp}
  \begin{algorithmic}[1]
    \REQUIRE point $\y \notin \BB(R)$, error tolerance $\zeta$, range of $A$'s eigenvalues $[\underline{\lambda}, \overline{\lambda}]$.
    \ENSURE approximate Mahalanobis projection $\x \approx \Pi^A_{\BB(R)}[\y]$.
    \STATE (Choice 1.) $\v = A \y$, let $\rho(\mu) = \| \sbr{A + \mu I}^{-1} \v \|_2^2 - R^2$.
    \STATE (Choice 2.) Tridiagonalize $Q C Q^\trs = A$, $\w = C Q^\trs \y$, let $\rho(\mu) = \| \sbr{C + \mu I}^{-1} \w \|_2^2 - R^2$.
    \STATE $a_1 = (\frac{\norme{\y}}{R}-1)\underline{\lambda}$, $b_1 = (\frac{\norme{\y}}{R}-1)\overline{\lambda}$, $n = \big\lceil \log_2 ( \frac{1}{\zeta}(\frac{~\overline{\lambda}~}{~\underline{\lambda}~}-1)\norme{\y}(\frac{\norme{\y}}{R}-1) ) \big\rceil$. \\
    (Without loss of generality, we assume $n \geq 1$.)
    \FOR{$i = 1, \dots, n$}
    \STATE $(a_{i+1}, b_{i+1}) =
    \begin{cases}
      (\frac{a_i + b_i}{2}, b_i), & ~~ \text{if} ~ \rho(\frac{a_i + b_i}{2}) \geq 0 ; \\
      (a_i, \frac{a_i + b_i}{2}), & ~~ \text{otherwise.} \\
    \end{cases}$
    \ENDFOR
    \STATE $\x = \min\{ 1, \frac{ R }{ \norme{\tilde{\x}} } \} \cdot \tilde{\x}$, where  $\tilde{\x} = (A + \frac{a_{n+1} + b_{n+1}}{2} I)^{-1} A \y$.
  \end{algorithmic}
\end{algorithm}

\subsection{Numerical Lemmas}
\label{apd:alg-impl}

The overall efficiency of \lightons hinges on two key operations, namely the matrix inversion $A_t^{-1}$ and the (infrequent) Mahalanobis projection $\Pi^{A_t}_{\BB(D/2)}[\hat{\y}_{t+1}]$.
In this subsection, we detail efficient numerical approaches for the two operations, respectively.

\paragraph{Matrix inversion.}
The per-round update of the matrix $A_t$ is a rank-one update.
Instead of inverting from scratch at a cost of $O(d^\omega)$, we can update $A_{t-1}^{-1}$ to obtain $A_t^{-1}$ in only $O(d^2)$ time with the Sherman-Morrison-Woodbury formula.
The following equation ensures $V_t = A_t^{-1}$ for any $t \in [T]$.
\begin{equation}
  \label{eq:rank-one-update-inversion}
  V_0 = \frac{1}{\eps} I; \quad
  V_t = V_{t-1} - \frac{1}{1 + \norm{\bgr_t}_{V_{t-1}}^2} V_{t-1} \bgr_t \bgr_t^\trs V_{t-1} , ~~
  \text{where} \bgr_t = c_f \gr\gtyt.
\end{equation}

\paragraph{Mahalanobis projection onto Euclidean ball.}
In \lightons, all Mahalanobis projections are onto the Euclidean ball $\BB(D/2)$ rather than the potentially complex domain $\XX$.
This geometry enables customized numerical approaches faster than generic solvers~\citep{lee2015cuttingplane,jiang2020cuttingplane}.
Specifically, the dual problem of the Mahalanobis projection reduces to a one-dimensional root-finding problem (Exercise 4.22 of~\citet{Book04:CvxOpt}), solvable via bisection.

We propose \fastproj in \cref{alg:mp}, which offers two implementation choices for different purposes.
Choice 1 attains superior theoretical dependence on $d$ by exploiting fast matrix multiplication with exponent $\omega < 2.3714$~\citep{mat-mul-exponent-best2024}.
Choice 2 offers stronger practical performance via tridiagonalization tricks~\citep{parlett1998symmetriceigen,golub-loan-mat-compute}, which are more favorable in practice where $\omega=3$.
The following lemma characterizes the error and runtime of \fastproj.

\begin{lem}
  \label{lem:mp}
  Let $\x$ be the output of \fastproj (\cref{alg:mp}), and $\x^\ast = \Pi^A_{\BB(R)}[\y]$ be the exact Mahalanobis projection, then $\x \in \BB(R)$ and $\norme{\x - \x^\ast} \leq \zeta$.
  The number of bisections is $O( \log (\frac{1}{\zeta}\frac{~\overline{\lambda}~}{~\underline{\lambda}~}) )$.
  With choice 1, the runtime is $O(d^\omega n)$, and with choice 2, the runtime is $O(d^3 + d n)$.
\end{lem}

In the context of \lightons (\cref{alg:lightons}), it suffices to set the error tolerance as $\zeta_t = O(1/t^2)$ at the $t$-th round.
The accumulated truncation error contributes only a negligible additive constant to the regret of \lightons.
A detailed discussion is deferred to \cref{lem:regret-decompose-inexact} in \cref{apd:alg-proof-main-theorem}.
Moreover, with $\zeta_t = \Omega(1/t^3)$, $\underline{\lambda}_t = \eps$ and $\overline{\lambda}_t = \eps + c_f^2 c_g^2 G^2 t$, we have $n_t = O(\log t)$ at the $t$-th round.

To prove \cref{lem:mp}, we first present two supporting lemmas.
\cref{lem:mp-kkt} reduces the Mahalanobis projection onto a Euclidean ball to a one-dimensional root-finding problem.
\cref{lem:mp-error} bridges the truncation error of root-finding and the truncation error of Mahalanobis projection.

\begin{lem}
  \label{lem:mp-kkt}
  $\Pi^A_{\BB(R)}[\y] = (A + \mu^\ast I)^{-1} A \y$, where $\mu^\ast$ is the only positive zero of the following function
  \begin{equation}
    \label{eq:mp-kkt}
    \rho(\mu) = \norme{\sbr{A + \mu I}^{-1} A \y}^2 - R^2 = \sum_{i=1}^d \frac{v_i^2}{\sbr{1+\frac{\mu}{\lambda_i}}^2} - R^2 ,
  \end{equation}
  and $\mu^\ast$ satisfies that
  \begin{equation}
    \label{eq:mp-kkt-range}
    \sbr{\frac{\norme{\y}}{R}-1}\lambda_d \leq \mu^\ast \leq \sbr{\frac{\norme{\y}}{R}-1}\lambda_1 ,
  \end{equation}
  where $v_i = \e_i^\trs Q^\trs \y$, $\e_i$ is the $i$-th standard basis vector, and $A = Q \Lambda Q^\trs$ is the eigendecomposition, $\Lambda = \diag{\lambda_1, \dots, \lambda_d}$ with $\lambda_1 \geq \dots \geq \lambda_d > 0$.
\end{lem}

\begin{proof}\textbf{of \cref{lem:mp-kkt}}
  The Mahalanobis projection onto a Euclidean ball is formulated as
  \begin{equation*}
    \begin{aligned}
      \min_{\x \in \mathbb{R}^d} ~~ & (\x-\y)^\trs A (\x-\y) \\
      \text{s.t.} ~~ & \x^\trs \x \leq R^2 \\
    \end{aligned}
  \end{equation*}
  The Lagrangian of this quadratic program is $\LL = (\x-\y)^\trs A (\x-\y) + \mu (\x^\trs \x - R^2)$.
  According to the KKT conditions, $\gr\LL = 2 A (\x-\y) + 2 \mu \x = \bm{0}$ with $\mu > 0$.
  (Otherwise, $\mu = 0$ implies $\y \in \BB(R)$ and the projection is trivial.)
  Rearranging $\gr\LL = \bm{0}$ yields $\x = (A + \mu I)^{-1} A \y$, thus $\x^\trs\x = R^2$ is equivalent to $\rho(\mu) = 0$, which proves \cref{eq:mp-kkt}.

  To prove \cref{eq:mp-kkt-range}, since $\rho$ monotonically decreases on $[0,\infty)$, it suffices to bound $\rho(\mu)$.
  By the orthogonality of $Q$, we have $\sum_{i=1}^{d} v_i^2 = \norme{Q^\trs \y}^2 = \norme{\y}^2$.
  Therefore,
  \begin{equation*}
    \frac{\norme{\y}^2}{(1+\frac{\mu}{\lambda_d})^2} - R^2 \leq \rho(\mu) = \sum_{i=1}^d \frac{v_i^2}{(1+\frac{\mu}{\lambda_i})^2} - R^2 \leq \frac{\norme{\y}^2}{(1+\frac{\mu}{\lambda_1})^2} - R^2 .
  \end{equation*}
  Then it is straightforward that $\rho((\frac{\norme{\y}}{R}-1)\lambda_d) \geq 0$ and $\rho((\frac{\norme{\y}}{R}-1)\lambda_1) \leq 0$, yielding that $\mu^\ast \in [ (\frac{\norme{\y}}{R}-1)\lambda_d, (\frac{\norme{\y}}{R}-1)\lambda_1 ]$.
  Since $\rho$ is strictly decreasing on $[0, \infty)$ and $\mu^\ast$ is its unique positive zero.
\end{proof}

\begin{lem}
  \label{lem:mp-error}
  Let $\x^\ast = (A + \mu^\ast I)^{-1} A \y$, where $\rho(\mu^\ast) = 0$ as in \cref{lem:mp-kkt}.
  If $\tilde{\mu} \geq 0$ and $\abs{\tilde{\mu} - \mu^\ast} \leq \frac{\lambda_d}{\norme{\y}} \zeta$, then $\x \in \BB(R)$ and $\norme{\x - \x^\ast} \leq \zeta$, where $\x$ is constructed from $\tilde{\mu}$ as
  \begin{equation}
    \label{eq:enforce-feasibility-of-fastprojection}
    \x = \Pi_{\BB(R)} \mbr{\tilde{\x}} = \min\lbr{ 1, \frac{ R }{ \norme{\tilde{\x}} } } \cdot \tilde{\x} ,
    \quad
    \tilde{\x} = \sbr{A + \tilde{\mu} I}^{-1} A \y .
  \end{equation}
\end{lem}

\begin{proof}\textbf{of \cref{lem:mp-error}}
  \cref{eq:enforce-feasibility-of-fastprojection} immediately implies $\x \in \BB(R)$.

  To prove $\norme{\x - \x^\ast} \leq \zeta$, since $\norme{\x - \x^\ast} \leq \norme{\tilde{\x} - \x^\ast}$ by \cref{lem:folklore}, it suffices to prove $\norme{\tilde{\x} - \x^\ast} \leq \zeta$.
  By the definition of $v_i$ in \cref{lem:mp-kkt}, we have
  \begin{equation*}
    \begin{aligned}
      \norme{\tilde{\x} - \x^\ast}^2
      & = \norme{\sbr{A + \tilde{\mu} I}^{-1} A \y - \sbr{A + \mu^\ast I}^{-1} A \y}^2
      = \sum_{i=1}^{d} v_i^2 \sbr{  \frac{1}{1 + \frac{\tilde{\mu}}{\lambda_i}}  -  \frac{1}{1 + \frac{\mu^\ast}{\lambda_i}}  }^2 \\
      & \leq \sum_{i=1}^{d} v_i^2 \sbr{  1  -  \frac{1}{1 + \frac{\abs{\tilde{\mu}-\mu^\ast}}{\lambda_i}}  }^2
      \leq \sum_{i=1}^{d} v_i^2 \sbr{\frac{\tilde{\mu}-\mu^\ast}{\lambda_i}}^2
      \leq \norme{\y}^2 \sbr{\frac{\tilde{\mu}-\mu^\ast}{\lambda_d}}^2 \leq \zeta^2 .
    \end{aligned}
  \end{equation*}
  The first inequality uses that $\lambda_i > 0$, $\tilde{\mu} \geq 0$ and $\mu^\ast \geq 0$;
  The second inequality comes from the following inequality, $(1-\frac{1}{1+a})^2 = \frac{a^2}{(a+1)^2} \leq a^2$ for any $a \geq 0$;
  The third inequality uses that $\sum_{i=1}^{d} v_i^2 = \norme{\y}^2$;
  The last inequality uses that $|\tilde{\mu} - \mu^\ast| \leq \frac{\lambda_d}{\norme{\y}} \zeta$.
\end{proof}

With the help of these lemmas, we can prove \cref{lem:mp}.

\begin{proof}\textbf{of \cref{lem:mp}}
  First, by \cref{lem:mp-kkt} and the supplied eigenvalue bounds, the wanted zero satisfies
  \begin{equation*}
    a_1 \leq \sbr{\frac{\norme{\y}}{R}-1}\lambda_d
    \leq \mu^\ast
    \leq \sbr{\frac{\norme{\y}}{R}-1}\lambda_1
    \leq b_1,
  \end{equation*}
  which justifies the initial bisection interval (Line~3 in \cref{alg:mp}).
  Furthermore, $\rho$ monotonically decreases on $[0, \infty)$, which implies the selection of $a_{i+1}$ and $b_{i+1}$ (Lines~4--6 in \cref{alg:mp}).

  Next, we show the convergence of the bisection.
  With the interval length halving each iteration, let $\tilde{\mu}=(a_{n+1}+b_{n+1})/2$.
  Based on the value of $n$ (Line~3 in \cref{alg:mp}), we have
  \begin{equation*}
    |\tilde{\mu} - \mu^\ast|
    \leq \frac{b_1 - a_1}{2^n}
    = \frac{(\frac{\norme{\y}}{R}-1)(\overline{\lambda}-\underline{\lambda})}{2^n}
    \leq \frac{\underline{\lambda}}{\norme{\y}} \zeta
    \leq \frac{\lambda_d}{\norme{\y}} \zeta
    .
  \end{equation*}
  Then by \cref{lem:mp-error}, \cref{alg:mp} achieves an error $\norme{\x - \x^\ast} \leq \zeta$.

  Finally, we bound the runtime.
  We note that choice 1 and 2 are equivalent, since
  \begin{equation*}
    \norme{\sbr{A + \mu I}^{-1} A \y}^2 = \norme{Q \sbr{C + \mu I}^{-1} Q^\trs Q C Q^\trs \y }^2 = \norme{\sbr{C + \mu I}^{-1} \w }^2 .
  \end{equation*}
  With choice 1, each iteration requires $d^\omega$ arithmetic operations, where the computational bottleneck lies in the matrix inversion $(A + \mu I)^{-1}$, resulting in a runtime of $O( d^\omega n )$.
  With choice 2, the tridiagonalization of $A$ requires $O(d^3)$ arithmetic operations~\citep{golub-loan-mat-compute}, and each iteration can evaluate $(C + \mu I)^{-1} \w$ via the Thomas algorithm with only $O(d)$ arithmetic operations~\citep{golub-loan-mat-compute}, resulting in a runtime of $O( d^3 + d n )$.
\end{proof}

\section{Proofs for \texorpdfstring{\cref{sec:alg}}{Section~\ref{sec:alg}}}
\label{apd:alg-proof}

In this section, we prove the theoretical guarantees of \lightons.

\subsection{Proof of \texorpdfstring{\cref{lem:proj-count}}{Lemma~\ref{lem:proj-count}}}
\label{apd:alg-proof-proj-count}

\begin{proof}
  Let $\Phi_T$ denote the sum of squared norms of updates, as in the left-hand side of \cref{eq:neg-tr} in \cref{lem:neg-tr}.
  Note that $A_t = \eps I + \sum_{i=1}^{t} \gr\ftxt \gr\ftxt^\trs$ in \cref{alg:lightonscore} and $\norme{\gr\ftxt} \leq G$ in \cref{asm:bounded-gradient}.
  Then by \cref{lem:neg-tr} we have
  \begin{equation*}
    \Phi_T \triangleq \sum_{t=1}^{T} \norme{\frac{1}{\gamma} A_t^{-1} \gr\ftxt}^2 \leq \frac{1}{\gamma^2} \frac{d}{\eps} .
  \end{equation*}

  We prove \cref{lem:proj-count} by bounding $N$ with $\Phi_T$.
  Let $\{ \tau_i \}_{i=1}^{N} \subseteq \{2, \dots, T+1\}$ denote the indices of all decisions produced by a Mahalanobis projection.
  With $\tau_0 = 1$, we have that for any $i \in [N]$,
  \begin{equation*}
    - \sum_{t=\tau_{i-1}}^{\tau_i-1} \frac{1}{\gamma} A_t^{-1} \gr\ftxt = \hat{\x}_{\tau_i} - \x_{\tau_{i-1}} .
  \end{equation*}
  Since $\norme{\hat{\x}_{\tau_i}} > k D/2$ and $\norme{\x_{\tau_{i-1}}} \leq D/2$ (because $\x_{\tau_0}=\x_1=\bm{0}$ and every later $\x_{\tau_{i-1}}$ is a projection onto $\XX$), we have
  \begin{equation*}
    \norme{ \sum_{t=\tau_{i-1}}^{\tau_i-1} \frac{1}{\gamma} A_t^{-1} \gr\ftxt } = \norme{\hat{\x}_{\tau_i} - \x_{\tau_{i-1}}} > (k - 1) \frac{D}{2} .
  \end{equation*}
  By applying the inequality $n \sum_{i=1}^{n} \norme{\v_i}^2 \geq \norme{\sum_{i=1}^{n}\v_i}^2$, we obtain
  \begin{equation*}
    (\tau_i-\tau_{i-1}) \sum_{t=\tau_{i-1}}^{\tau_i-1} \norme{\frac{1}{\gamma} A_t^{-1} \gr\ftxt}^2 \geq \norme{\sum_{t=\tau_{i-1}}^{\tau_i-1} \frac{1}{\gamma} A_t^{-1} \gr\ftxt}^2 > (k - 1)^2 \frac{D^2}{4} .
  \end{equation*}
  Rearranging the preceding inequality yields
  \begin{equation*}
    \frac{1}{\tau_i-\tau_{i-1}} < \frac{4}{(k - 1)^2 D^2} \sum_{t=\tau_{i-1}}^{\tau_i-1} \norme{\frac{1}{\gamma} A_t^{-1} \gr\ftxt}^2 .
  \end{equation*}
  Since $\tau_i-\tau_{i-1} \geq 1$, applying the ``harmonic mean $\leq$ arithmetic mean'' inequality $\frac{n}{\sum_{i=1}^{n} \frac{1}{a_i}} \leq \frac{\sum_{i=1}^{n} a_i}{n}$ for $a_i = \frac{1}{\tau_i - \tau_{i-1}} > 0$, we have
  \begin{equation*}
    \frac{N^2}{T} \leq \frac{N^2}{\tau_N-\tau_0} \leq \sum_{i=1}^{N} \frac{1}{\tau_i-\tau_{i-1}} .
  \end{equation*}
  Combining the preceding two inequalities, we obtain
  \begin{equation*}
    \frac{N^2}{T} \leq \sum_{i=1}^{N} \frac{1}{\tau_i-\tau_{i-1}} < \frac{4}{(k - 1)^2 D^2} \sum_{t = 1}^{\tau_N-1} \norme{\frac{1}{\gamma} A_t^{-1} \gr\ftxt}^2 \leq \frac{4}{(k - 1)^2 D^2} \Phi_T .
  \end{equation*}
  Finally, rearranging the preceding inequality yields the desired result
  \begin{equation*}
    N < \frac{2\sqrt{T\Phi_T}}{(k-1) D} \leq \frac{2}{(k-1) D \gamma} \sqrt{\frac{d}{\eps}T} .
  \end{equation*}
\end{proof}

We remark that \cref{lem:proj-count} is tight in terms of $T$ up to logarithmic factors.
Consider $d=1$ and $\gr\ftxt = 1/\sqrt{T}$, then $A_t^{-1} \gr \ftxt = \Omega(1 / \sqrt{T})$ and $\| \sum_{t=1}^{T} A_t^{-1} \gr \ftxt \|_2 = \Omega(\sqrt{T})$.
We notice that Theorem~2.3 of~\citet{NIPS20:max-control} corroborates this point.

\subsection{Proof of \texorpdfstring{\cref{thm:lightonscore}}{Theorem~\ref{thm:lightonscore}}}
\label{apd:alg-proof-improper-theorem}

\begin{proof}
  The proof of \cref{thm:lightonscore} follows as a specialization of the proof of \cref{thm:lightons} given in \cref{apd:alg-proof-main-theorem}.
  The only substantive difference is that, in the present setting, the analysis is carried out directly on the original loss functions $f_t$, rather than on surrogate losses $g_t$.

  We therefore omit the repetitive details and highlight only the key intermediate arguments.
  \begin{itemize}
    \item \textbf{Choice of the expanded curvature parameter $\gamma$ in \lightonscore.}~
      The selection of $\gamma = \frac{1}{2} \min \{ \frac{2}{(k+1) DG}, \alpha \}$ directly follows from \cref{lem:exp-concave}.
    \item \textbf{Regret analysis with selective projections.}~
      The regret analysis follows the same structure as \cref{lem:regret-decompose-exact}, which accounts for projections triggered only when the iterate exits the expanded domain $\tilde{\XX}_k$.
      Details are given in \cref{apd:alg-proof-main-theorem}.
    \item \textbf{Runtime analysis with selective projections.}~
      The runtime bound is derived in the same manner as in the proof of \cref{thm:lightons} in \cref{apd:alg-proof-main-theorem}.
  \end{itemize}
\end{proof}

\subsection{Proof of \texorpdfstring{\cref{lem:surrogate-taylor}}{Lemma~\ref{lem:surrogate-taylor}}}
\label{apd:alg-proof-surrogate-taylor}

\begin{proof}
  By the selection of curvature parameters $\gamma_0$ (in \cref{lem:exp-concave}) and $\gamma'$ (in \cref{alg:lightons}), it holds that $\gamma' = \frac{1}{2} \min \{ \frac{1}{DG}, \alpha, \frac{4}{(k+1) c_f c_g DG} \} = \min \{ \gamma_0, \frac{2}{(k+1) c_f c_g DG} \} \leq \gamma_0$.
  Thus by \cref{lem:exp-concave} we have
  \begin{equation*}
    \begin{aligned}
      \ftxt - \ftu & \leq \gr\ftxt^\trs (\x_t-\u) - \frac{\gamma_0}{2} \sbr{ \gr\ftxt^\trs (\x_t-\u) }^2 \\
      & \leq \gr\ftxt^\trs (\x_t-\u) - \frac{\gamma'}{2} \sbr{ \gr\ftxt^\trs (\x_t-\u) }^2 .
    \end{aligned}
  \end{equation*}

  To finish the proof, it suffices to show that, with $U(a) = a - \frac{\gamma'}{2} a^2$,
  \begin{equation*}
    U\sbr{\gr\ftxt^\trs (\x_t-\u)} \leq U\sbr{c_f \gr\gtyt^\trs (\y_t-\u)} .
  \end{equation*}
  Note that $U'(a) = 1 - \gamma' a$, thus $U$ monotonically increases on $(-\infty, \frac{1}{\gamma'}]$.
  It can be verified that
  \begin{equation*}
    \gr\ftxt^\trs (\x_t-\u) \leq c_f \gr\gtyt^\trs (\y_t-\u) \leq c_f c_g G \cdot \frac{k + 1}{2} D \leq \frac{1}{\gamma'} .
  \end{equation*}
  The first inequality follows from \cref{cdt:domain-conversion};
  The second inequality from the Cauchy-Schwarz inequality $\u^\trs\v \leq \norme{\u}\norme{\v}$;
  The last inequality from the selection of $\gamma'$.
\end{proof}

\subsection{Proof of \texorpdfstring{\cref{thm:lightons}}{Theorem~\ref{thm:lightons}}}
\label{apd:alg-proof-main-theorem}

We first introduce a property of the Mahalanobis projection.
This lemma is often referred to as the Pythagorean theorem in Banach space or the non-expansiveness of projections.

\begin{lem}[Lemma~8 of~\citet{journals/ml/HazanAK07}]
  \label{lem:folklore}
  If $A \in \reals^{d \times d}$ is a positive-definite and symmetric matrix, then for any convex and compact domain $\XX \subseteq \reals^{d}$, any point $\y \in \reals^{d}$ and any point $\u \in \XX$, $\norm{\Pi^A_\XX[\y] - \u}_A \leq \norm{\y - \u}_A$.
\end{lem}

Based on \cref{lem:folklore}, we can decompose the regret of \lightons as the following lemma shows.

\begin{lem}
  \label{lem:regret-decompose-exact}
  Ignoring the truncation error of Mahalanobis projections, under the same assumptions as \cref{thm:lightons}, in \cref{alg:lightons}, for any $t \in [T]$ and all $\u \in \XX$, it holds that, with $\bgr_t = c_f \gr\gtyt$,
  \begin{equation}
    \label{eq:regret-decompose-exact}
    2 \bgr_t^\trs (\y_t-\u) \leq \frac{1}{\gamma'} \norm{\bgr_t}_{A_t^{-1}}^2 + \gamma' \norm{\y_t-\u}_{A_t}^2 - \gamma' \norm{\y_{t+1}-\u}_{A_t}^2 .
  \end{equation}
\end{lem}

\begin{proof}\textbf{of \cref{lem:regret-decompose-exact}}
  When $\norme{\hat{\y}_{t+1}} \leq k D/2$, no Mahalanobis projection is triggered, i.e., $\y_{t+1} = \hat{\y}_{t+1} = \y_t - \frac{1}{\gamma'} A_t^{-1} \bgr_t$, then for any $\u \in \reals^d$, we have
  \begin{equation*}
    \norm{ \y_{t+1} - \u }_{A_t}^2 = \norm{ \y_{t} - \frac{1}{\gamma'} A_t^{-1} \bgr_t - \u }_{A_t}^2 .
  \end{equation*}

  Otherwise, when $\norme{\hat{\y}_{t+1}} > k D/2$ and the Mahalanobis projection is triggered, by \cref{lem:folklore}, for any $\u \in \XX$,  we have
  \begin{equation*}
    \norm{ \y_{t+1} - \u }_{A_t}^2 = \norm{ \Pi_{\BB(D/2)}^{A_t} \mbr{\y_{t} - \frac{1}{\gamma'} A_t^{-1} \bgr_t} - \u }_{A_t}^2 \leq \norm{ \y_{t} - \frac{1}{\gamma'} A_t^{-1} \bgr_t - \u }_{A_t}^2 .
  \end{equation*}

  Combining both cases, we conclude that, for any $t \in [T]$ and any $\u \in \XX$,
  \begin{equation*}
    \norm{ \y_{t+1} - \u }_{A_t}^2 \leq \norm{ \sbr{\y_{t}-\u} - \frac{1}{\gamma'} A_t^{-1} \bgr_t }_{A_t}^2 .
  \end{equation*}
  Rearranging the preceding inequality yields the desired result of \cref{eq:regret-decompose-exact}.
\end{proof}

We note that \cref{lem:regret-decompose-exact} ignores the truncation error of the Mahalanobis projections.
The matrix factorization underlying Mahalanobis projections is related to eigendecomposition, equivalent to finding roots of a $d$-degree polynomial, which is not exactly solvable by finitely many arithmetic operations when $d \ge 5$.
The next lemma complements the analysis.
Thanks to \fastproj (\cref{alg:mp} in \cref{apd:alg-impl}), the truncation error only incurs a negligible additive $O(1/t^2)$ term in the regret decomposition of \cref{lem:regret-decompose-exact}, which can be safely ignored in the final regret bound of \lightons.

\begin{lem}
  \label{lem:regret-decompose-inexact}
  Let the Mahalanobis projection of \cref{alg:lightons} be implemented with \cref{alg:mp} with $A = A_t$, $R = D/2$, $\y = \hat{\y}_{t+1}$, $\zeta = \zeta_t$, $\underline{\lambda} = \eps$, $\overline{\lambda} = \eps + c_f^2 c_g^2 G^2 t$, where $\zeta_t$ is defined as
  \begin{equation}
    \label{eq:zeta-t-selection}
    \zeta_t = \min\lbr{  \frac{1}{(k+1) D \cdot \overline{\lambda} \cdot \gamma' t^2}  ,  \sqrt{ \frac{1}{\overline{\lambda} \cdot \gamma' t^2} }  } = \Omega\sbr{\frac{1}{t^3}} .
  \end{equation}
  Then under the same assumptions as \cref{thm:lightons}, in \cref{alg:lightons}, for any $t \in [T]$ and all $\u \in \XX$, with $\bgr_t = c_f \gr\gtyt$,
  \begin{equation}
    \label{eq:regret-decompose-inexact}
    2 \bgr_t^\trs (\y_t-\u) \leq \frac{1}{\gamma'} \norm{\bgr_t}_{A_t^{-1}}^2 + \gamma' \norm{\y_t-\u}_{A_t}^2 - \gamma' \norm{\y_{t+1}-\u}_{A_t}^2 + \frac{2}{t^2} .
  \end{equation}
\end{lem}

\begin{proof}\textbf{of \cref{lem:regret-decompose-inexact}}
  It suffices to consider the case when the Mahalanobis projection is triggered, i.e., $\norme{\hat{\y}_{t+1}} > k D/2$.
  Let $\y^\ast_{t+1} = \Pi_{\BB(D/2)}^{A_t} \mbr{\hat{\y}_{t+1}}$ be the exact Mahalanobis projection without truncation error, and $\bdelta_t = \y_{t+1} - \y^\ast_{t+1}$ be the truncation error.
  By \cref{lem:mp-error}, $\norme{\bdelta_t} = \norme{\y_{t+1} - \y^\ast_{t+1}} \leq \zeta_t$;
  By \cref{lem:folklore}, for any $\u \in \XX$,
  \begin{equation*}
    \begin{aligned}
      \norm{ \y_{t+1} - \u }_{A_t}^2
      & = \norm{ \y^\ast_{t+1} - \u }_{A_t}^2 + 2 \sbr{ \y^\ast_{t+1} - \u }^\trs A_t \bdelta_t + \norm{ \bdelta_t }_{A_t}^2 \\
      & \leq \norm{ \y^\ast_{t+1} - \u }_{A_t}^2 + (k+1) D \cdot \overline{\lambda} \cdot \norme{\bdelta_t} + \overline{\lambda} \cdot \norme{\bdelta_t}^2 \\
      & \leq \norm{ \y_{t} - \frac{1}{\gamma'} A_t^{-1} \bgr_t - \u }_{A_t}^2 + \frac{2}{\gamma' t^2} .
    \end{aligned}
  \end{equation*}
  The first inequality is because the operator norm of $A_t$ is at most $\overline{\lambda} = c_f^2 c_g^2 G^2 t + \eps$, which comes from \cref{cdt:domain-conversion}, i.e., $\norme{\bgr_t} = \norme{c_f \gr\gtyt} \leq c_f c_g G$;
  The second inequality uses the selection of $\zeta_t$ in \cref{eq:zeta-t-selection}.
\end{proof}

For conciseness, we ignore the truncation error of the Mahalanobis projections in the proof of \cref{thm:lightons}, since it only incurs an additive $O(\sum_{t=1}^{T} 1/t^2) = O(1)$ term in the final regret bound.

\begin{proof}\textbf{of \cref{thm:lightons}}
  Recall that $\bgr_t = c_f \gr\gtyt$.
  Plugging \cref{lem:regret-decompose-exact} into \cref{lem:surrogate-taylor} yields:
  \begin{equation*}
    \begin{aligned}
      & \ftxt - \ftu \overset{\text{\eqref{eq:surrogate-taylor}}}{\leq} ~ \bgr_t^\trs (\y_t-\u) - \frac{\gamma'}{2} \sbr{ \bgr_t^\trs (\y_t-\u) }^2 \\
      \overset{\text{\eqref{eq:regret-decompose-exact}}}{\leq} ~ & \frac{1}{2} \sbr{ \frac{1}{\gamma'} \norm{\bgr_t}_{A_t^{-1}}^2 + \gamma' \norm{\y_t-\u}_{A_t}^2 -\gamma' \norm{\y_{t+1}-\u}_{A_t}^2 } - \frac{\gamma'}{2} \sbr{ \bgr_t^\trs (\y_t-\u) }^2 \\
      = ~ & \frac{1}{2\gamma'} \norm{\bgr_t}_{A_t^{-1}}^2 + \frac{\gamma'}{2} \norm{\y_t-\u}_{A_{t-1}}^2 - \frac{\gamma'}{2} \norm{\y_{t+1}-\u}_{A_t}^2 .
    \end{aligned}
  \end{equation*}
  The equality is because $A_t = A_{t-1} + \bgr_t \bgr_t^\trs$ in \cref{alg:lightons}.
  Summing over the time horizon and telescoping the right-hand side establishes the desired regret of \cref{eq:lightons-regret}:
  \begin{equation*}
    \begin{aligned}
      \sum_{t=1}^{T} \sbr{\ftxt-\ftu}
      & \leq \sum_{t=1}^{T} \sbr{ \frac{1}{2\gamma'} \norm{\bgr_t}_{A_t^{-1}}^2 + \frac{\gamma'}{2} \norm{\y_t-\u}_{A_{t-1}}^2 - \frac{\gamma'}{2} \norm{\y_{t+1}-\u}_{A_t}^2 } \\
      & = \sbr{ \frac{1}{2\gamma'} \sum_{t=1}^{T} \norm{\bgr_t}_{A_t^{-1}}^2 } + \frac{\gamma'}{2} \norm{\y_1-\u}_{A_0}^2 - \frac{\gamma'}{2} \norm{\y_{T+1}-\u}_{A_T}^2 \\
      & \overset{\text{\eqref{eq:log-det}}}{\leq} \frac{d}{2\gamma'} \log\sbr{ 1 + \frac{c_f^2 c_g^2 G^2}{d \eps} T } + \frac{\gamma' \eps D^2}{8} .
    \end{aligned}
  \end{equation*}
  The last inequality uses \cref{lem:log-det} and \cref{cdt:domain-conversion}, i.e., $\norme{\bgr_t} = \norme{c_f \gr\gtyt} \leq c_f c_g G$.

  Finally, the desired runtime of \cref{eq:lightons-time} follows from the following two parts:
  \begin{itemize}
    \item \textbf{Runtime aside from \fastproj.}~
      In each round, the domain conversion of \cref{lem:domain-conversion} takes $O(\EP_\XX + d)$ time, updating and inverting $A_t$ as \cref{eq:rank-one-update-inversion} takes $O(d^2)$ time, and other operations, such as computing the surrogate gradient as \cref{eq:grad-g}, take only $O(d)$ time.
      The overall runtime aside from \fastproj is $O( (\EP_\XX + d^2) T )$.
    \item \textbf{Runtime of \fastproj.}~
      By \cref{lem:mp}, the number of bisections $n_t = O(\log t) = O(\log T)$ and each bisection takes $O(d^\omega)$ time with choice 1.
      Since \lightonscore is a subroutine in \lightons, the number of calls to \fastproj is at most $O( (k - 1)^{-1} d^{0.5} \sqrt{T/\eps} )$ by \cref{lem:proj-count}.
      Specifically, since $A_t = \eps I + \sum_{i=1}^{t} \bgr_i \bgr_i^\trs$ in \cref{alg:lightons}, we have
      \begin{equation*}
        \Phi_T \triangleq \sum_{t=1}^{T} \norme{\frac{1}{\gamma'} A_t^{-1} \bgr_t}^2 \leq \frac{1}{\gamma'^2} \frac{d}{\eps} , \quad
        N < \frac{2\sqrt{T\Phi_T}}{(k-1) D} \leq \frac{2}{(k-1)D\gamma'} \sqrt{\frac{d}{\eps}T} .
      \end{equation*}
      The first inequality uses \cref{lem:neg-tr}, which controls $\Phi_T$ as a constant unrelated to the gradients;
      The second inequality follows from the proof of \cref{lem:proj-count}, which remains valid as long as the domain to project onto and the extended domain are separated by a margin of at least $(k-1)D/2$.
      The overall runtime of \fastproj is $O( (k - 1)^{-1} d^{0.5} \sqrt{T/\eps} \cdot d^\omega \log T )$.
  \end{itemize}
\end{proof}

\section{Proofs for \texorpdfstring{\cref{sec:stochastic}}{Section~\ref{sec:stochastic}}}
\label{apd:stochastic-proof}

In this section, we provide the proofs and details for \cref{sec:stochastic}.

\subsection{Proof of \texorpdfstring{\cref{thm:lightons-stochastic}}{Theorem~\ref{thm:lightons-stochastic}}}
\label{apd:stochastic-proof-lightons}

Before proving \cref{thm:lightons-stochastic}, we first give a more detailed presentation of the convergence rates induced by \lightons.
Throughout this subsection, we use the regime $\veps=O(1/d)$ stated in \cref{sec:stochastic-guarantee}.
With $T = \Theta\sbr{\frac{d}{\veps} \log \frac{d}{\veps} \log \frac{1}{\delta}}$, we have
\begin{subequations}
  \begin{align}
    \label{eq:lightons-stochastic-high-prob}
    & F(\bar{\x}_T) - \min_{\x \in \XX} F(\x) \leq \frac{1}{T} \sbr{ \reg_T + 4 \sqrt{ \frac{\reg_T}{2\gamma_0} \log \frac{4 \log T}{\delta} } + \frac{8}{\gamma_0} \log \frac{4 \log T}{\delta} } = O(\veps) , \\
    \label{eq:lightons-stochastic-high-prob-time}
    & \runtime \leq O\sbr{ \frac{d^3}{\veps} \log \frac{d}{\veps} \log \frac{1}{\delta} + \frac{d^{3.5}}{\sqrt{\veps}} \log \frac{d}{\veps} \log \frac{1}{\delta} } = \tilde{O}\sbr{\frac{d^3}{\veps}} ,
  \end{align}
\end{subequations}
where $\reg_T$ follows from \cref{eq:lightons-regret-specified} in \cref{cor:lightons}.
With $\frac{1}{\delta} = T' = \Theta\sbr{\frac{d}{\veps} \log \frac{d}{\veps}}$, we have
\begin{subequations}
  \begin{align}
    \label{eq:lightons-stochastic-expected}
    & \expectation{}{F(\bar{\x}_{T'}) - \min_{\x \in \XX} F(\x)} \leq O(\veps) , \\
    \label{eq:lightons-stochastic-expected-time}
    & \runtime \leq O\sbr{ \frac{d^3}{\veps} \log \frac{d}{\veps} + \frac{d^{3.5}}{\sqrt{\veps}} \log \frac{d}{\veps} } = \tilde{O}\sbr{\frac{d^3}{\veps}} .
  \end{align}
\end{subequations}
We remark that shifting the online loss function from $f_t(\u) = f(\u; \xi_t)$ to $h_t(\u) = f(\frac{\u+\x_t}{2}; \xi_t)$ improves the high-probability excess risk bound of \cref{eq:lightons-stochastic-high-prob} to $\frac{2}{T} \sbr{\reg_T + \frac{2}{\gamma_0} \log \frac{1}{\delta}}$.
This removes the $\log \log T$ term, though the asymptotic rate in~\cref{thm:lightons-stochastic} remains unchanged.
We refer readers to Theorem~1 of~\citet{arXiv23:high-prob-risk-bound} for details.

\begin{proof}
  First, we prove \cref{eq:lightons-stochastic-high-prob}.
  Without loss of generality, we consider $T \geq 3$.
  Then the high-probability excess risk bound of \cref{eq:lightons-stochastic-high-prob} directly follows from Corollary~2 of~\citet{AISTATS17:exp-concave-high-prob} by substituting \lightons's regret in \cref{cor:lightons}.

  Then, we verify that the choice of $T = \Theta\sbr{\frac{d}{\veps} \log \frac{d}{\veps} \log \frac{1}{\delta}}$ yields a high-probability excess risk of $O(\veps)$.
  Let $\Gamma_\delta = \frac{4 \log T}{\delta} = O\sbr{\frac{1}{\delta} \log \frac{d}{\veps} + \frac{1}{\delta} \log \log \frac{1}{\delta}}$.
  We have
  \begin{equation*}
    F(\bar{\x}_T) - \min_{\x \in \XX} F(\x) \leq O\sbr{\frac{\reg_T + \sqrt{\reg_T \cdot \log \Gamma_\delta} + \log \Gamma_\delta}{T}} \leq O\sbr{\frac{\reg_T + \log \Gamma_\delta}{T}} .
  \end{equation*}
  The first inequality is essentially Corollary~2 of~\citet{AISTATS17:exp-concave-high-prob}, and the second inequality follows from the fact that $\sqrt{ab} \leq \frac{a+b}{2} = O(a+b)$ for any positive terms $a$ and $b$.
  Since $\reg_T = O(d \log T) = O\sbr{d \log \frac{d}{\veps} + d \log \log \frac{1}{\delta}}$, it suffices to show that
  \begin{equation*}
    \begin{aligned}
      O\sbr{\frac{\reg_T}{T}} & = O\sbr{\frac{  d  \sbr{ \log \frac{d}{\veps} + \log \log \frac{1}{\delta} }  }{\frac{d}{\veps} \log \frac{d}{\veps} \log \frac{1}{\delta}}} = O(\veps) , ~~ \text{and} \\
      O\sbr{\frac{\log \Gamma_\delta}{T}} & = O\sbr{\frac{  \log \frac{1}{\delta} + \log \sbr{ \log \frac{d}{\veps} + \log \log \frac{1}{\delta} }  }{\frac{d}{\veps} \log \frac{d}{\veps} \log \frac{1}{\delta}}} = O\sbr{\frac{\veps}{d}} = O(\veps) .
    \end{aligned}
  \end{equation*}

  Next, we prove \cref{eq:lightons-stochastic-expected} from \cref{eq:lightons-stochastic-high-prob}.
  By the definition of expectation, we have
  \begin{equation*}
    \begin{aligned}
      \expectation{}{F(\bar{\x}_{T'}) - \min_{\x \in \XX} F(\x)} & \leq \sbr{1 - \frac{1}{T'}} \cdot O\sbr{\frac{\reg_{T'} + \sqrt{\reg_{T'} \cdot \log \Gamma_{1/T'}} + \log \Gamma_{1/T'}}{T'}} \\
      & \quad + \frac{1}{T'} \sbr{\max_{\y \in \XX} F(\y) - \min_{\x \in \XX} F(\x)} \\
      & \leq O\sbr{\frac{\reg_{T'} + \log \Gamma_{1/T'}}{T'}} + \frac{DG}{T'} \leq O\sbr{\frac{\reg_{T'} + \log \Gamma_{1/T'}}{T'}} .
    \end{aligned}
  \end{equation*}
  The above inequalities use the Lipschitzness of $F$ and the boundedness of $\XX$.
  Similarly, we verify that the choice of $T' = \Theta\sbr{\frac{d}{\veps} \log \frac{d}{\veps}}$ yields an in-expectation excess risk of $O(\veps)$ with $\frac{1}{\delta} = T' = \Theta\sbr{\frac{d}{\veps} \log \frac{d}{\veps}}$.
  It suffices to note that $\Gamma_{1/T'} = 4 T' \log T' = O\sbr{\frac{d}{\veps} \log^2 \frac{d}{\veps}}$ and that
  \begin{equation*}
    O\sbr{\frac{\log \Gamma_{1/T'}}{T'}} = O\sbr{\frac{\log \frac{d}{\veps}}{\frac{d}{\veps} \log \frac{d}{\veps}}} = O\sbr{\frac{\veps}{d}} = O(\veps) .
  \end{equation*}

  Finally, the total runtime in \cref{eq:lightons-stochastic-high-prob-time,eq:lightons-stochastic-expected-time} follows from \cref{thm:lightons}.
\end{proof}

\subsection{Proofs and Details for ERM-based SXO Methods}
\label{apd:stochastic-proof-erm}

In this subsection, we construct an SXO instance where applying cutting-plane methods incurs $\tilde{O}(d^3/\veps)$ runtime, and we give an intuitive explanation of the runtime barrier $\tilde{O}(d^3/\veps)$.

\begin{prp}
  \label{prp:sxo-erm-cpm-time}
  Under \cref{asm:exp-concave-stochastic} and that $\XX=\BB(1)$, let the stochastic functions take the form $f(\x;\xi) = \phi\sbr{\w(\xi)^\trs\x}$, where $\phi:\reals\to\reals$ is a black-box function, and $\w: \Xi \to \reals^d$ is a fixed mapping.
  By~\citep{NIPS15:fast-rate-exp-concave,AISTATS17:exp-concave-high-prob}, obtaining an $\veps$-optimal solution for SXO reduces to obtaining an $O(\veps)$-optimal minimizer of the offline $\alpha$-exp-concave objective
  \begin{equation*}
    \hat{F}(\x) = \frac{1}{T} \sum_{t=1}^{T} \phi\sbr{\w(\xi_t)^\trs \x} ,
  \end{equation*}
  where $T=\tilde{O}(d/\veps)$ is the necessary sample size.
  The best-known cutting-plane methods of~\citet{lee2015cuttingplane,jiang2020cuttingplane} solve this offline problem to $O(\veps)$-accuracy in $\tilde{O}(d^3/\veps)$ time.
\end{prp}

\begin{proof}\textbf{of \cref{prp:sxo-erm-cpm-time}}
  Cutting-plane methods query the gradient of $\hat{F}$ for $\tilde{O}(d)$ times, and each gradient query costs $O(dT)$ time due to the finite-sum structure of $\hat{F}$ and the black-box nature of $\phi$.
  Therefore, the total runtime is $\tilde{O}(d^2 T + d^3) = \tilde{O}(d^3/\veps)$.
\end{proof}

\paragraph{An intuitive explanation of the runtime $\tilde{O}(d^3/\veps)$.}
With the online-to-batch conversion, OGD's regret $O(\sqrt{T})$ for OCO translates to a sample complexity of $O(1/\veps^2)$ and a runtime of $\tilde{O}(d/\veps^2)$ for SXO.
When the target excess risk $\veps$ is relatively large, i.e., $\veps = \tilde{\Omega}(1/d^2)$, the OGD-based SXO method may circumvent the $\tilde{O}(d^3/\veps)$ runtime barrier.
Thus the computational challenge of SXO mainly manifests when pursuing the (near) optimal sample complexity of $\tilde{O}(d/\veps)$.

\cref{lem:exp-concave} indicates that an exp-concave function can be regarded as a strongly convex function in the direction of its gradient.
Comparing the sample complexity lower bound $\Omega(d/\veps)$ for SXO~\citep{COLT15:exp-concave} to the $\Omega(1/\veps)$ rate for stochastic strongly convex optimization~\citep{hazan2014stochasticstronglyconvex}, the exp-concave offline objective $F$ can be viewed as the sum of $d$ strongly convex objectives on $d$ directions, one along each dimension corresponding to the eigenvectors of the Hessian-related matrix $A_t$ in ONS.

Our intuition is as follows:
the difficulty of pursuing the (near) optimal sample complexity arises from exploiting the directional strong convexity, which necessitates estimating the covariance structure of the stochastic gradients.
\citet{cai2010covariance} establish a minimax sample complexity for covariance matrix estimation under squared Frobenius norm, which coincides with the sample complexity lower bound for SXO.
Specifically, for some family of covariance matrices $\SS$, with $T$ samples i.i.d. drawn from a distribution with covariance $\Sigma \in \SS$, the estimator $\hat{\Sigma}$ satisfies
\begin{equation*}
  \inf_{\hat{\Sigma}} \sup_{\Sigma \in \SS} \expectation{}{\tnormf{\hat{\Sigma} - \Sigma}^2} = O\sbr{\frac{d}{T}} .
\end{equation*}
We refer readers to Theorem~1 of~\citet{cai2010covariance} for a formal statement.
Since computing the sample covariance matrix requires $O(d^2 T)$ time, estimating the covariance structure of stochastic gradients plausibly incurs at least $O(d^2 T)$ runtime, leading to the $\tilde{O}(d^3/\veps)$ runtime barrier.

\section{Proofs for \texorpdfstring{\cref{sec:app}}{Section~\ref{sec:app}}}
\label{apd:app-proof}

In this section, we provide algorithms and proofs for \lightons's applications listed in \cref{sec:app}.

\subsection{Proofs and Details for Gradient-Norm Adaptivity}
\label{apd:app-proof-norm}

In this subsection, we show how \lightons recovers ONS's gradient-norm adaptive regret bound in \cref{eq:gradient-norm-adaptive-oxo}, and thereby leads to small-loss adaptivity for OXO with smoothness, as well as comparator-norm adaptivity for unbounded OCO.
We state the theoretical guarantees based on \lightons and omit those of the original ONS-based methods as they are identical except for the runtime.

We say $f$ is $H$-smooth on $\XX$ if for any $(\x, \y) \in \XX^2$, $\norme{\gr f(\x) - \gr f(\y)} \leq H \norme{\x - \y}$.

\begin{thm}[\lightons's improvement for OXO with smoothness]
  \label{thm:ons-for-sl}
  Under \cref{asm:bounded-domain,asm:bounded-gradient,asm:exp-concave}, and that $f_t$ is $H$-smooth for any $t \in [T]$, \lightons (\cref{alg:lightons}) satisfies that,
  \begin{equation*}
    \begin{aligned}
      \reg_T & \leq \frac{d}{2\gamma_0} \log \sbr{ \frac{8H}{d \eps} L_T + \frac{4H}{\gamma_0 \eps} \log \frac{4H}{\rme \gamma_0 \eps} + \frac{\gamma_0 D^2 H}{d} + 2 } + \frac{\gamma_0 \eps D^2}{8} , \\
      \runtime & \leq O\sbr{ \sbr{\EP_\XX + d^2} T + d^\omega \sqrt{T} \log T } ,
    \end{aligned}
  \end{equation*}
  where $L_T$ is defined in \cref{eq:small-loss-quantity}, and $\gamma_0$ is defined in \cref{cor:lightons}.
\end{thm}

\begin{proof}\textbf{of \cref{thm:ons-for-sl}}
  Based on the proof of Theorem~1 of~\citet{aistats12:exp-concave-smooth}, it suffices to bound the gradient norms of the surrogate loss by those of the original loss.
  We note that Appendix~B.4.2 of~\citet{NeurIPS24:universal-1-projection} also discusses the small-loss bounds of ONS with domain conversion, although with a different domain conversion.
  Since $\norme{\gr\gtyt} \leq \norme{\gr\ftxt}$ by \cref{lem:domain-conversion}, we have
  \begin{equation*}
    G_{T,g} \triangleq \sum_{t=1}^T \norme{\gr\gtyt}^2 \leq \sum_{t=1}^T \norme{\gr\ftxt}^2 \triangleq G_{T,f} .
  \end{equation*}
  \cref{thm:lightons} implies the runtime with $\eps = d$ and the regret of
  \begin{equation*}
    \reg_T(\u) \leq \frac{d}{2\gamma_0} \log \sbr{1 + \frac{G_{T,g}}{d \eps}} + \frac{\gamma_0 \eps D^2}{8} \leq \frac{d}{2\gamma_0} \log \sbr{1 + \frac{G_{T,f}}{d \eps}} + \frac{\gamma_0 \eps D^2}{8} .
  \end{equation*}
  The preceding inequality uses Jensen's inequality and $\ttr{A_T} = d \eps +  G_{T,g}$, i.e.,
  \begin{equation}
    \label{eq:jensen-det2tr}
    \log \tdet{A_T} - \log \tdet{A_0} \leq d \log \frac{\ttr{A_T}}{d} - \log \tdet{A_0} = d \log \sbr{1 + \frac{G_{T,g}}{d \eps}} .
  \end{equation}
  Then Corollary~5 of~\citet{aistats12:exp-concave-smooth} converts the gradient-norm adaptive bound to the small-loss bound with respect to $L_T$ and completes the proof.
\end{proof}

\begin{thm}[\lightons's improvement for unbounded OCO]
  \label{thm:ons-for-pf}
  If the loss function $f_t: \reals^d \to \reals$ is convex and $1$-Lipschitz for any $t \in [T]$, then there exists an algorithm satisfying that, for any $\u \in \reals^d$,
  \begin{equation*}
    \begin{aligned}
      \reg_T(\u) & \leq \tilde{O}\sbr{ \sqrt{d \sum_{t=1}^{T} \sbr{\gr\ftxt^\trs\u}^2} } \leq \tilde{O}\sbr{\norme{\u} \sqrt{d G_T}} , \\
      \runtime & \leq O\sbr{ d^2 T + d^\omega \sqrt{T} \log T } .
    \end{aligned}
  \end{equation*}
\end{thm}

\begin{proof}\textbf{of \cref{thm:ons-for-pf}}
  Theorem~8 of~\citet{COLT18:black-box-reduction} hinges on their Algorithm~7, Lemmas~16~and~17, apart from the coin-betting framework.
  To prove \cref{thm:ons-for-pf}, we show how \lightons adapts their analysis with minimal changes.
  \begin{itemize}
    \item \textbf{Modifications to their Algorithm~7.}~
      Since the decisions of their ONS are intermediate decisions to maximize the ``wealth'' in the coin-betting framework instead of true decisions to minimize regret, we can ignore the improper-to-proper conversion and replace their ONS with \lightonscore.
      Specifically, their ONS runs on $\XX = \BB(1/2)$ while \lightonscore runs on $\YY = \BB(3/4)$ with the deferral coefficient $k=3/2$.
    \item \textbf{Modifications to their Lemma~16.}~
      Their ONS's domain $\XX = \BB(1/2)$ implies a curvature parameter $\gamma = \frac{2 - \log 3}{2}$ while \lightonscore's domain $\YY = \BB(3/4)$ implies $\gamma = \frac{6 - \log 7}{18}$.
    \item \textbf{Modifications to their Lemma~17.}~
      Relaxing the radius from $1/2$ to $3/4$ enlarges constants in the regret.
      Nonetheless, the numerical constants in their Lemma~17 are loose enough to accommodate our changes, greatly simplifying our analysis.
      Following their proof, let $\z_t$ denote the gradient that \lightonscore receives at time $t$, \cref{thm:lightons} and \cref{eq:jensen-det2tr} imply
      \begin{equation*}
        \reg_T \leq \frac{d}{2\gamma} \log \sbr{1 + \frac{\sum_{t=1}^{T} \norme{\z_t}^2}{d \eps}} + \frac{\gamma \eps D^2}{8} .
      \end{equation*}
      Plugging $D = 1$, $\gamma = \frac{6 - \log 7}{18} \in (0.225,0.226)$, $\eps = d$, and $\norme{\z_t}^2 \leq 16 \norme{\gr\ftxt}^2$ yields
      \begin{equation*}
        \reg_T \leq d \sbr{ \frac{5}{2} \log \sbr{1 + \frac{16}{d^2} \sum_{t=1}^{T} \norme{\gr\ftxt}^2} + \frac{1}{35} } .
      \end{equation*}
      The preceding regret bound fully recovers their Lemma~17 when $d \geq 2$.
  \end{itemize}
  The runtime follows from \cref{thm:lightons} with $\eps = d$.
\end{proof}

\paragraph{Comparison with OQNS.}
We remark that OQNS can hardly achieve full gradient-norm adaptivity, as the log-barrier regularization introduces an unavoidable $O(\log T)$ bias term independent of $G_T$.
From the perspective of OMD, the regret can be decomposed into a stability term and a bias term.
Adopted from~\citet{hazan2016introductionOCO,book22:FO-book}, the regret decomposition of a typical OMD algorithm is as follows:
\begin{equation}
  \label{eq:omd-regret-decomposition-conceptual}
  \reg_T(\u) \leq \underbrace{\frac{\eta}{2} \sum_{t=1}^{T} \norm{\gr\ftxt}_{\ast}^2}_{\text{stability}} + \underbrace{\frac{\norm{\u - \x_1}^2}{2\eta}}_{\text{bias}} .
\end{equation}
While ONS and \lightons, as instances of OMD, can flexibly balance stability and bias,
\footnote{For ONS and \lightons, the learning rate $\eta$ is replaced with the time-varying Hessian-related matrix's inverse $A_t^{-1}$.}
OQNS sacrifices such flexibility for computational efficiency.
OQNS employs a highly stable log-barrier, which suppresses the stability term but inflates the bias term in a manner resistant to small gradient norms.
Specifically, the log-barrier contributes a term $- \log (D^2 - \norme{\u}^2)$ in the regret for any comparator $\u \in \BB(D)$.
As $\norme{\u}$ approaches $D$, this term diverges to infinity.
This issue is mitigated with the fixed-share trick~\citep{COLT23:OQNS}, an illustration of which is as follows:
\begin{equation*}
  \sum_{t=1}^{T} \sbr{f_t(\x_t) - f_t(\u)}
  \leq \underbrace{\sum_{t=1}^{T} \sbr{f_t(\x_t) - f_t(\v)}}_{\text{fixed-share regret}} + \underbrace{\sum_{t=1}^{T} \sbr{f_t(\v) - f_t(\u)}}_{\text{fixed-share margin}} , ~~
  \text{where} ~ \v = \sbr{1 - \frac{1}{T}} \u .
\end{equation*}
The fixed-share margin term $f_t(\v) - f_t(\u) \leq O(GD/T)$ due to boundedness and Lipschitzness, summing to $O(1)$ over $T$ rounds.
Since $1 - \norme{\v}^2 \geq \Omega(D^2/T)$, the fixed-share regret term contributes $- \log (D^2 - \norme{\v}^2) = O(\log T)$ to the regret, which prevents full gradient-norm adaptivity.

\subsection{Proofs and Details for Logistic Bandits}
\label{apd:app-proof-bandits}

In this subsection, we focus on logistic bandits as a representative instance of generalized linear bandits (GLB) to concretely illustrate the applicability of \lightons.

\paragraph{Problem setting.}
Logistic bandits can be interpreted as interactions between a learner and an adversary, which unfolds as follows:
At each round $t \in [T]$, the learner selects an arm $\x_t$ from an arm set $\XX_t \subseteq \BB(1)^K$ with $K$ arms, and suffers a stochastic loss $y_t \in \{0,1\}$, where $\probability{y_t = 1 \mid \x_t} = \sigma(\x_t^\trs \w^\star)$.
The true parameter $\w^\star \in \WW = \BB(D/2)$ is unknown, and $\sigma(z) = 1 / (1 + \exp(-z))$ denotes the sigmoid function.
The performance of the learner is measured by its pseudo-regret, quantifying the cumulative expected loss against the optimal arms in hindsight, which is defined as
\begin{equation*}
  \reg_T = \sum_{t=1}^T \sbr{ \sigma(\x_t^\trs \w^\star) - \sigma({\x^\star_t}^\trs \w^\star) } ,
\end{equation*}
where $\x^\star_t = \argmax_{\x \in \XX_t} \x^\trs \w^\star$.

A challenge in GLB lies in its dependence on $\kappa = \max_{\x \in \XX, \w \in \WW} 1 / \sigma'(\x^\trs \w)$, a problem-dependent constant that may grow exponentially with $D$.

\paragraph{Jointly efficient GLB algorithm.}
\citet{ZYJ2025LogB} propose the first one-pass (constant memory and constant per-round time) GLB algorithm achieving $\kappa$-free-leading-term pseudo-regret.
Their pseudo-regret is $O(d \sqrt{T} \log T + \kappa (d \log T)^2)$, which remains sublinear in $T$ for $\kappa = o(T)$.

Their key idea is to estimate the parameter $\w^\star$ with a ``look-ahead'' variant of ONS, whose analytical properties can remove the dependence on $\kappa$ in the pseudo-regret's leading terms while retaining the constant-memory and constant-time efficiency of ONS.
An illustration of their ``look-ahead'' ONS is provided as follows:
\begin{equation}
  \label{eq:omd-logb-original-mp}
  \begin{aligned}
    \tilde{H}_t & = \lambda I + \sbr{ \sum_{i=1}^{t-1} \gr^2 \ell_i(\w_{i+1}) } + \eta \gr^2 \ell_t(\w_t) , \\
    \hat{\w}_{t+1} & = \w_t - \frac{1}{\eta} \tilde{H}_t^{-1} \nabla \ell_t(\w_t) , \\
    \w_{t+1} & = \Pi^{\tilde{H}_t}_{\BB(D/2)}[\hat{\w}_{t+1}] ,
  \end{aligned}
\end{equation}
where $\ell_t(\w) = - y_t \log \sigma(\x_t^\trs \w) - (1-y_t) \log (1 - \sigma(\x_t^\trs \w))$ denotes the logistic loss for parameter estimation at round $t$.

Because of ONS, their algorithm incurs the worst-case $\tilde{O}((d^2 K + d^\omega) T)$ runtime bottleneck, where $K$ is the number of arms.
Below we show that plugging \lightons can improve the runtime while preserving the $\kappa$-free-leading-term pseudo-regret guarantee.

\paragraph{\lightons-based counterpart.}
We propose an analogous ``look-ahead'' variant of \lightons in \cref{alg:lightons-logb}, which replaces the ONS-based subroutine in~\citet{ZYJ2025LogB}.
Line~7 of \cref{alg:lightons-logb} introduces the deferred-projection mechanism of \lightonscore to replace the Mahalanobis projection in \cref{eq:omd-logb-original-mp}.

\cref{alg:lightons-logb} improves the runtime to $\tilde{O}(d^2 K T + d^\omega \cdot \min\{\sqrt{\kappa dT} \log\kappa, T\})$, while preserving the $\kappa$-free-leading-term pseudo-regret guarantee.
When $\kappa = o(T)$, as required for sublinear pseudo-regret, the improved algorithm is asymptotically faster than the original algorithm.
We state the theoretical guarantee of \cref{alg:lightons-logb} in the theorem below.

\begin{algorithm}[t]
  \caption{``look-ahead'' \lightons for~\citep{ZYJ2025LogB}}
  \label{alg:lightons-logb}
  \begin{algorithmic}[1]
    \REQUIRE domain $\WW = \BB(D/2)$, regularization coefficient $\lambda$, inverse step size $\eta$.
    \STATE Initialize $H_0 = \lambda I$, $\w = \bm{0}$.
    \FOR {$t = 1, \dots, T$}
    \STATE Update the lower confidence bound function as in Algorithm 1 of~\citet{ZYJ2025LogB}.
    \STATE Select the arm $\x_t$ as in Algorithm 1 of~\citet{ZYJ2025LogB} and observe the loss $y_t$.
    \STATE $\tilde{H}_t = H_{t-1} + \eta \gr^2\ell_t(\w_t)$.
    \STATE $\hat{\w}_{t+1} = \w_t - \frac{1}{\eta} \tilde{H}_t^{-1} \gr\ell_t(\w_t)$.
    \STATE $\w_{t+1} =
    \begin{cases}
      \hat{\w}_{t+1} & \text{if} ~ \norme{\hat{\w}_{t+1}} \leq k^\star D / 2 \\
      \Pi^{\tilde{H}_t}_{\BB(D/2)}[\hat{\w}_{t+1}] & \text{otherwise}
    \end{cases}$, where $k^\star = 1 + \frac{2}{\log\kappa} = \Theta(1 + 1/D)$.
    \STATE $H_t = H_{t-1} + \gr^2\ell_t(\w_{t+1})$.
    \ENDFOR
  \end{algorithmic}
\end{algorithm}

\begin{thm}[\lightons's improvement for logistic bandits]
  \label{thm:lightons-logb}
  For binary logistic bandits, under the conditions $\bigcup_{t=1}^T \XX_t \subseteq \BB(1)$, and $\norm{\w^\star} \leq D/2$ where $D$ is known a priori, there exists an algorithm achieving a pseudo-regret of $O(d \sqrt{T} \log T + \kappa (d \log T)^2)$, and a total runtime of $\tilde{O}(d^2 K T + d^\omega \cdot \min\{\sqrt{\kappa dT} \log\kappa, T\})$.
\end{thm}

We remark that \cref{alg:lightons-logb} does not suffer from improper learning concerns, because the algorithm's decision is the arm $\x_t$, rather than the estimated parameter $\w_t$.
Consequently, \cref{alg:lightons-logb} is free from the improper-to-proper conversion and resembles \lightonscore.

Before proving \cref{thm:lightons-logb}, we first present two critical lemmas that facilitate migrating the original ONS-based subroutine to \lightons-based \cref{alg:lightons-logb}.
The following lemma resembles our \cref{lem:regret-decompose-exact} and Lemma~1 of~\citet{ZYJ2025LogB}, showing that \cref{alg:lightons-logb} implies an OMD-like regret decomposition form similar to that of the original ONS-based subroutine.

\begin{lem}
  \label{lem:lightons-logb-omd}
  \cref{alg:lightons-logb} satisfies that, for any $\u \in \BB(D/2)$,
  \begin{equation*}
    \norm{\w_{t+1} - \u}_{H_{t-1}}^2 \leq 2 \eta \gr \tilde{\ell}_t(\w_{t+1})^\trs (\u - \w_{t+1}) + \norm{\w_t - \u}_{H_{t-1}}^2 - \norm{\w_t - \w_{t+1}}_{H_{t-1}}^2 ,
  \end{equation*}
  where $\tilde{\ell}_t$ is the second-order Taylor expansion of $\ell_t$ at $\w_t$, i.e.,
  \begin{equation*}
    \tilde{\ell}_t(\w) \triangleq \ell_t(\w_t) + \gr\ell_t(\w_t)^\trs (\w - \w_t) + \frac{1}{2} \norm{\w - \w_t}_{\gr^2\ell_t(\w_t)}^2 .
  \end{equation*}
\end{lem}

\begin{proof}\textbf{of \cref{lem:lightons-logb-omd}}
  We note that the descent-and-projection update in \cref{alg:lightons-logb} is equivalent to the OMD update, as Appendix~D of~\citet{ZYJ2025LogB} has shown.
  Specifically,
  \begin{equation}
    \label{eq:lightons-logb-omd-proj}
    \begin{aligned}
      \w_{t+1} &= \Pi^{\tilde{H}_t}_{\BB(D/2)}[\hat{\w}_{t+1}] = \Pi^{\tilde{H}_t}_{\BB(D/2)}\mbr{\w_t - \frac{1}{\eta} \tilde{H}_t^{-1} \gr\ell_t(\w_t)} \\
      \iff \w_{t+1} &= \argmin_{\w \in \BB(D/2)} \tilde{\ell}_t(\w) + \frac{1}{2\eta} \norm{\w - \w_t}_{H_{t-1}}^2 ,
    \end{aligned}
  \end{equation}
  and
  \begin{equation}
    \label{eq:lightons-logb-omd-noproj}
    \begin{aligned}
      \w_{t+1} &= \hat{\w}_{t+1} = \w_t - \frac{1}{\eta} \tilde{H}_t^{-1} \gr\ell_t(\w_t) \\
      \iff \w_{t+1} &= \argmin_{\w \in \reals^d} \tilde{\ell}_t(\w) + \frac{1}{2\eta} \norm{\w - \w_t}_{H_{t-1}}^2 .
    \end{aligned}
  \end{equation}
  To recover Lemma~1 of~\citet{ZYJ2025LogB}, we examine whether the following inequality holds
  \begin{equation}
    \label{eq:lightons-logb-omd}
    \gr\tilde{\ell}_t(\w_{t+1})^\trs \sbr{\w_{t+1}-\u} \leq \frac{1}{2\eta} \sbr{ \norm{\w_t-\u}_{H_{t-1}}^2 - \norm{\w_{t+1}-\u}_{H_{t-1}}^2 - \norm{\w_{t+1}-\w_t}_{H_{t-1}}^2 } .
  \end{equation}

  When $\norme{\hat{\w}_{t+1}} \leq k^\star D / 2$ and the Mahalanobis projection is not triggered, by \cref{eq:lightons-logb-omd-noproj} we have
  \begin{equation*}
    \gr_{\w=\w_{t+1}} \sbr{ \tilde{\ell}_t(\w) + \frac{1}{2\eta} \norm{\w - \w_t}_{H_{t-1}}^2 }^\trs (\w_{t+1} - \u) = \bm{0}^\trs (\w_{t+1} - \u) = 0 .
  \end{equation*}
  Rearranging terms, we have
  \begin{equation*}
    \begin{aligned}
      \gr\tilde{\ell}_t(\w_{t+1})^\trs \sbr{\w_{t+1}-\u} & = - \frac{1}{\eta} (\w_{t+1} - \w_t)^\trs H_{t-1} (\w_{t+1} - \u) \\
      & = \frac{1}{2\eta} \sbr{ \norm{\w_t-\u}_{H_{t-1}}^2 - \norm{\w_{t+1}-\u}_{H_{t-1}}^2 - \norm{\w_{t+1}-\w_t}_{H_{t-1}}^2 } ,
    \end{aligned}
  \end{equation*}
  which means that \cref{eq:lightons-logb-omd} holds with equality.

  When $\norme{\hat{\w}_{t+1}} > k^\star D / 2$ and the Mahalanobis projection is triggered,  by \cref{eq:lightons-logb-omd-proj}, we have
  \begin{equation*}
    \gr_{\w=\w_{t+1}} \sbr{ \tilde{\ell}_t(\w) + \frac{1}{2\eta} \norm{\w - \w_t}_{H_{t-1}}^2 }^\trs (\w_{t+1} - \u) \leq 0 .
  \end{equation*}
  Rearranging terms, we have
  \begin{equation*}
    \begin{aligned}
      \gr\tilde{\ell}_t(\w_{t+1})^\trs \sbr{\w_{t+1}-\u} & \leq - \frac{1}{\eta} (\w_{t+1} - \w_t)^\trs H_{t-1} (\w_{t+1} - \u) \\
      & = \frac{1}{2\eta} \sbr{ \norm{\w_t-\u}_{H_{t-1}}^2 - \norm{\w_{t+1}-\u}_{H_{t-1}}^2 - \norm{\w_{t+1}-\w_t}_{H_{t-1}}^2 } ,
    \end{aligned}
  \end{equation*}
  which means that \cref{eq:lightons-logb-omd} holds.

  Therefore, combining both cases, we have that \cref{eq:lightons-logb-omd} always holds.
\end{proof}

The next lemma gives the total runtime of \cref{alg:lightons-logb}.

\begin{lem}
  \label{lem:bandits-time}
  \cref{alg:lightons-logb} has a runtime of $\tilde{O}(d^2 T + d^\omega \cdot \min\{\sqrt{\kappa dT} \log\kappa, T\})$.
\end{lem}

\begin{proof}\textbf{of \cref{lem:bandits-time}}
  We first show that the expansion of the domain influences the condition number $\kappa$.
  We recall that $\ell_t(\w) = - y_t \log \sigma(\x_t^\trs \w) - (1-y_t) \log (1 - \sigma(\x_t^\trs \w))$ is the logistic loss function, with $\gr\ell_t(\w) = \sbr{ \sigma(\x_t^\trs \w) - y_t } \x_t$ and $\gr^2\ell_t(\w) = \sigma'(\x_t^\trs \w) \x_t \x_t^\trs$, where $\sigma(z) = 1 / (1 + \exp(-z))$ is the sigmoid function.
  By the definition,
  \begin{equation*}
    \kappa
    = \max_{\x \in \XX, \w \in \WW} \frac{1}{\sigma'(\x^\trs \w)}
    \leq \max_{\norme{\x} \leq 1, \norme{\w} \leq D/2} (1+\exp(\x^\trs\w))(1+\exp(-\x^\trs\w))
    \leq 2 + 2 \exp(D/2) .
  \end{equation*}
  Therefore, the condition number $\kappa$ grows exponentially with the diameter of the domain.
  For technical convenience, we write $\kappa' = \kappa^k$ when the diameter of the domain is expanded from $D$ to $D' = k D$ for some $k > 1$.

  Compared with the runtime of \lightons \cref{thm:lightons}, the only difference is that $\kappa$ appears in the number of Mahalanobis projections.
  We note that $H_t$ and $\tilde{H}_t$ also admit rank-one updates.
  It suffices to show how the logistic loss function affects the analysis of \cref{lem:proj-count}.
  Specifically, we need to bound the quantity
  \begin{equation*}
    \Phi'_T \triangleq \sum_{t=1}^{T} \norme{\frac{1}{\eta} \tilde{H}_t^{-1} \gr\ell_t(\w_t)}^2 .
  \end{equation*}
  By the update rule of \cref{alg:lightons-logb} and that $\eta > 1$ in~\citep{ZYJ2025LogB}, we have
  \begin{equation*}
    \tilde{H}_t
    = \lambda I + \sum_{i=1}^{t-1} \sigma'(\x_i^\trs \w_{i+1}) \x_i \x_i^\trs + \eta \sigma'(\x_t^\trs \w_t) \x_t \x_t^\trs
    \succ \lambda I + \sum_{i=1}^{t} \frac{1}{\kappa^k} \x_i \x_i^\trs .
  \end{equation*}
  Since $\abs{\sigma(\x_t^\trs \w_t) - y_t} \leq 1$, we have
  \begin{equation*}
    \begin{aligned}
      \norme{\tilde{H}_t^{-1} \gr\ell_t(\w_t)}^2 & = \norme{\tilde{H}_t^{-1} (\sigma(\x_t^\trs \w_t) - y_t) \x_t}^2 \leq \norme{\tilde{H}_t^{-1} \x_t}^2 \\
      & < \norme{\sbr{\lambda I + \sum_{i=1}^{t} \frac{1}{\kappa^k} \x_i \x_i^\trs}^{-1} \x_t}^2 = \kappa^{2k} \norme{\sbr{\kappa^k \lambda I + \sum_{i=1}^{t} \x_i \x_i^\trs}^{-1} \x_t}^2 .
    \end{aligned}
  \end{equation*}
  Summing up the preceding inequality with \cref{lem:neg-tr}, we have
  \begin{equation*}
    \Phi'_T < \frac{\kappa^{2k}}{\eta^2} \sum_{t=1}^{T} \norme{\sbr{\kappa^k \lambda I + \sum_{i=1}^{t} \x_i \x_i^\trs}^{-1} \x_t}^2 \leq \frac{\kappa^k d}{\eta^2 \lambda}  .
  \end{equation*}
  Therefore, reusing the proof of \cref{lem:proj-count}, we obtain that the runtime of \cref{alg:lightons-logb} with deferral coefficient $k$ is $\tilde{O}\big( d^2 T + d^\omega \cdot \min\{ \frac{\sqrt{\kappa^k dT}}{k-1}, T\} \big)$.
  Let $\Gamma(k) = \frac{\kappa^{k/2}}{k-1}$, then minimizing $\Gamma(k)$ over $k > 1$ yields $k^\star = 1 + \frac{2}{\log\kappa}$, $\kappa^k = \rme^2 \kappa$, and $\Gamma(k^\star) = \frac{\rme}{2} \sqrt{\kappa} \log \kappa$.
\end{proof}

Based on the preceding lemmas, we are ready to prove \cref{thm:lightons-logb}.

\begin{proof}\textbf{of \cref{thm:lightons-logb}}
  The upper-confidence-bound-based pseudo-regret analysis in~\citep{ZYJ2025LogB} primarily relies on their Theorem~1, which constitutes their Lemmas~4--6.
  To prove \cref{thm:lightons-logb}, we examine how replacing the ONS-based subroutine with \lightons-based \cref{alg:lightons-logb} affects these lemmas.

  \begin{itemize}
    \item \textbf{Modifications to their Lemma~4.}~
      Their Lemma~4 is supported by their Lemma~1 and local relaxation of generalized linear models.
      \cref{lem:lightons-logb-omd} shows that their Lemma~1 still holds when using \cref{alg:lightons-logb} as a replacement.
      The local relaxation is independent of the specific update rules of $\w_t$ and remains valid.
    \item \textbf{Modifications to their Lemmas~5~and~6.}~
      Their Lemmas~5~and~6 depend solely on the algebraic structure of the covariance matrix, i.e., $H_t = \lambda I + \sum_{i=1}^{t} \sigma'(\x_i^\trs \w_{i+1}) \x_i \x_i^\trs$.
      Therefore, \cref{alg:lightons-logb} does not affect their Lemmas~5~and~6.
  \end{itemize}

  Therefore, the pseudo-regret analysis of~\citet{ZYJ2025LogB} still holds when replacing their ONS-based subroutine with \cref{alg:lightons-logb}, although with an expanded domain diameter $D' = k^\star D$ and a correspondingly larger condition number $\kappa' = \kappa^{k^\star} = \rme^2 \kappa = O(\kappa)$.
  Finally, the runtime of \cref{alg:lightons-logb} follows from \cref{lem:bandits-time}.
  With the overhead $O(d^2 K T)$ for selecting arms as in~\citep{ZYJ2025LogB}, the total runtime is $\tilde{O}(d^2 K T + d^\omega \cdot \min\{\sqrt{\kappa dT} \log\kappa, T\})$.
\end{proof}

We remark that extending \cref{alg:lightons-logb} from the binary logistic bandits to the generalized linear bandits directly follows from~\citep{ZYJ2025LogB}.

\paragraph{Comparison with OQNS.}
OQNS can hardly be integrated into the method of~\citep{ZYJ2025LogB}, as it is tailored to the OXO protocol and does not admit a regret decomposition analogous to \cref{lem:lightons-logb-omd}, which is essential for eliminating $\kappa$ from leading terms in pseudo-regret.
Moreover, OQNS lacks the flexibility to accommodate customized local norms beyond those used in standard ONS.
Specifically, the local norm $\tilde{H}_t$ in~\citep{ZYJ2025LogB}, i.e., \cref{eq:omd-logb-original-mp}, is not a simple accumulation of gradient outer products, unlike OQNS's local norm in \cref{eq:oqns-update-objective}.
In contrast, \lightons retains the structural flexibility of ONS, enabling integration into GLB with minimal modifications.

\subsection{Proofs and Details for Memory-Efficient OXO}
\label{apd:app-proof-sketch}

In this subsection, we seek to reduce the working memory required to achieve the optimal $O(d \log T)$ regret in OXO, albeit conditioned on geometric properties of the problem instance.

First, we list the additional assumptions of~\citet{NIPS16:sketch-ONS} as follows:
\begin{itemize}
  \item \textbf{Additional assumption on domains.}~
    For any $t \in [T]$, the domain at the $t$-th round is the intersection of two parallel half-spaces, i.e.,
    \begin{equation}
      \label{eq:parallel-half-spaces-intersection-luo-nips16}
      \XX_t = \{ \x \mid |\w_t^\trs \x| < D/2 \} , \quad \text{where} ~ \w_t \in \reals^d ~ \text{is known and} ~ \norm{\w_t} = 1 .
    \end{equation}
  \item \textbf{Additional assumption on loss functions.}~
    For any $t \in [T]$, the curvature parameter $\gamma$, the loss function $f_t$, the trajectory $\{\x_t\}_{t=1}^T$ and the comparator $\u \in \bigcap_{t=1}^T \XX_t$ satisfy
    \begin{equation}
      \label{eq:exp-concave-like-curvature-luo-nips16}
      \ftx - \ftu \leq \gr\ftx^\trs (\x-\u) - \frac{\gamma}{2} \sbr{ \gr\ftx^\trs (\x-\u) }^2 .
    \end{equation}
\end{itemize}

Next, we formally state the theoretical guarantees of SON in the following proposition.

\begin{prp}[Theorem~3 of~\citet{NIPS16:sketch-ONS}]
  \label{prp:son}
  Under \cref{asm:bounded-gradient} and additional assumptions described in \cref{eq:parallel-half-spaces-intersection-luo-nips16,eq:exp-concave-like-curvature-luo-nips16}, SON (Algorithms 1 and 6 of~\citet{NIPS16:sketch-ONS}) satisfies that,
  \footnote{The $\log T$ term in SON's runtime arises from the eigendecomposition underlying sketching. Although Theorem~3 of~\citet{NIPS16:sketch-ONS} ignores it, we include it for comparison with \lightons and OQNS.}
  \begin{subequations}
    \begin{align}
      \label{eq:son-regret}
      \reg_T(\u) & \leq \frac{d'}{\gamma} \log \sbr{1 + \frac{G^2}{2d' \eps} T} + \frac{\gamma \eps D^2}{8} + \frac{\Delta_{1:T}}{2\gamma} , \\
      \label{eq:son-time}
      \runtime & \leq O\sbr{  d' d T \log T  } ,
    \end{align}
  \end{subequations}
  and the working memory is $O(d'd)$.
  The cumulative sketching error $\Delta_{1:T}$ can be bounded as
  \begin{equation}
    \label{eq:son-delta}
    \Delta_{1:T} \leq \min_{j \in [d']} \frac{2d'}{(d'-j+1)\eps} \sum_{i=j}^{d} \lambda_i\sbr{\sum_{t=1}^{T} \gr\ftxt \gr\ftxt^\trs} .
  \end{equation}
\end{prp}

\begin{algorithm}[t]
  \caption{\lightson}
  \label{alg:lightson}
  \begin{algorithmic}[1]
    \REQUIRE preconditioner coefficient $\eps$, dimension to reduce to $d'$.
    \STATE Initialize $\gamma_0 = \frac{1}{2} \min \{ \frac{1}{DG}, \alpha \}$, $S_0 = O_{2d' \times d}$, $R_0 = \frac{1}{\eps} I_{2d' \times 2d'}$, $\x_1 = \y_1 = \bm{0}$.
    \FOR {$t = 1, \dots, T$}
    \STATE Observe $\gr\ftxt$; and construct $\gr\gtyt$ as in \cref{lem:domain-conversion}.
    \STATE Sketch $\gr\gtyt$ into $S_t$ and $R_t$ with \cref{alg:fast-fd}.
    \algcomment{$\tilde{A}_t \preceq \tilde{A}_{t-1} + \gr\gtyt \gr\gtyt^\trs$.}
    \STATE $\hat{\y}_{t+1} = \y_t - \frac{1}{\gamma_0} \tilde{A}_t^{-1} \gr\gtyt$.
    \algcomment{$\tilde{A}_t^{-1} = \frac{1}{\eps} (I - S_t^\trs R_t S_t)$.}
    \STATE $\y_{t+1} =
    \begin{cases}
      \hat{\y}_{t+1} & \text{if} ~ \norme{\hat{\y}_{t+1}} \leq D \\
      \Pi^{\tilde{A}_t}_{\BB(D/2)}[\hat{\y}_{t+1}] & \text{otherwise}
    \end{cases}$.
    \algcomment{$\tilde{A}_t = \eps I + S_t^\trs S_t$.}
    \STATE $\x_{t+1} = \Pi_{\XX}[\y_{t+1}]$.
    \ENDFOR
  \end{algorithmic}
\end{algorithm}

\begin{algorithm}[t]
  \caption{Fast Frequent Directions of~\citet{SICOMP16:frequent-direct}}
  \label{alg:fast-fd}
  \begin{algorithmic}[1]
    \REQUIRE new gradient $\gr\gtyt$, frequent directions $S_{t-1}$, low-dimension inverse $R_{t-1}$.
    \ENSURE updated frequent directions $S_t$, updated low-dimension inverse $R_t$.
    \STATE $S_t = S_{t-1} + \e_{i_t} \gr\gtyt^\trs$, where $i_t$ is the index of the first all-zero row of $S_{t-1}$.
    \IF{$S_t$ still has all-zero rows}
    \STATE $R_t = \sbr{ \eps I + S_t S_t^\trs }^{-1}$.
    \ELSE
    \STATE $U_t \Sigma_t V_t^\trs = S_t$, truncated SVD with top $d'$ singular values.
    \STATE $S_t = \genfrac{[}{]}{0pt}{0}{(\Sigma_t^2 - (\Sigma_t)_{d',d'}^2 I)^{1/2} V_t^\trs}{O_{d' \times d}}$.
    \STATE $R_t = \diag{ \frac{1}{\eps + (\Sigma_t)_{1,1}^2 - (\Sigma_t)_{d',d'}^2}, \dots, \frac{1}{\eps + (\Sigma_t)_{d',d'}^2 - (\Sigma_t)_{d',d'}^2}, \frac{1}{\eps}, \dots, \frac{1}{\eps} }$.
    \ENDIF
  \end{algorithmic}
\end{algorithm}

Then, we present the \lightson algorithm in \cref{alg:lightson}, a memory-efficient variant of \lightons.
The key difference between \lightson and \lightons is the storage strategy.
Instead of storing full matrices $A_t \in \reals^{d \times d}$ and $V_t \in \reals^{d \times d}$, \lightson{} maintains compact sketches $S_t \in \reals^{2d' \times d}$ and $R_t \in \reals^{2d' \times 2d'}$.
Each row of $S_t$ stores a principal gradient component, while $R_t$ plays a role analogous to $V_t$ in \cref{eq:rank-one-update-inversion}.
The full matrix and its inverse are reconstructed as
\begin{equation*}
  \tilde{A}_t = \eps I + S_t^\trs S_t , \quad
  \tilde{A}_t^{-1} = \frac{1}{\eps} \sbr{ I - S_t^\trs R_t S_t } , \quad R_t = \sbr{\eps I + S_t S_t^\trs}^{-1} .
\end{equation*}
This relationship follows from the Sherman-Morrison-Woodbury formula:
\begin{equation*}
  (A + B C D)^{-1} = A^{-1} - A^{-1} B (C^{-1} + D A^{-1} B)^{-1} D A^{-1} .
\end{equation*}
Following~\citet{NIPS16:sketch-ONS}, \lightson{} uses Fast Frequent Directions of~\citet{SICOMP16:frequent-direct} in \cref{alg:fast-fd} to update $S_t$ and $R_t$.

Before proving \cref{thm:lightson}, we first introduce three lemmas to characterize the error introduced by sketching.
We introduce the notation $\Delta_t$ to denote the sketching error at the $t$-th round.
\begin{equation*}
  \Delta_t \triangleq
  \begin{cases}
    \frac{2d'}{\eps} \sigma_{d'}(S_t)^2 & \text{if SVD is triggered at $t$-th round} \\
    0 & \text{otherwise}
  \end{cases} ,
\end{equation*}
where $\sigma_i(S_t)$ is the $i$-th greatest singular value of $S_t$.
Specifically, \cref{lem:sketch-error} bounds the error for regret analysis with \cref{lem:log-det}, while \cref{lem:sketch-error-for-deferral} bounds the error for projection-count analysis with \cref{lem:neg-tr}.
\cref{lem:fd-subtracted-mass} bounds the total error with the spectrum of the Hessian-related matrix.

\begin{lem}
  \label{lem:sketch-error}
  At the $t$-th round of \cref{alg:lightson}, in \cref{alg:fast-fd}, $\Delta_t$ satisfies:
  \begin{equation*}
    \norm{\gr\gtyt}_{\tilde{A}_t^{-1}}^2 \leq \innerf{\tilde{A}_t^{-1}}{\tilde{A}_t - \tilde{A}_{t-1}} + \Delta_t .
  \end{equation*}
\end{lem}

\begin{proof}\textbf{of \cref{lem:sketch-error}}
  Let $\innerf{\cdot}{\cdot}$ denote the inner product induced by the Frobenius matrix norm, let $(\cdot)_{:,i}$ denote the $i$-th column of the matrix, and let $\bar{U}_t \bar{\Sigma}_t \bar{V}_t^\trs = S_t$ be the full SVD with all $2d'$ singular values.
  When SVD is not triggered, then $\tilde{A}_t = \tilde{A}_{t-1} + \gr\gtyt \gr\gtyt^\trs$ and $\Delta_t = 0$ trivially holds.
  When SVD is triggered,
  \begin{equation*}
    \begin{aligned}
      & \quad \innerf{ \tilde{A}_t^{-1} }{ \tilde{A}_{t-1} + \gr\gtyt \gr\gtyt^\trs - \tilde{A}_t } \\
      & = \sum_{i=1}^{2d'} \min\lbr{ (\bar{\Sigma}_t)_{i,i}^2, (\bar{\Sigma}_t)_{d',d'}^2 } \norm{(\bar{V}_t)_{:,i}}_{\tilde{A}_t^{-1}}^2 \leq \frac{2d'}{\eps} (\bar{\Sigma}_t)_{d',d'}^2 \triangleq \Delta_t .
    \end{aligned}
  \end{equation*}
  The inequality uses the fact that $\tilde{A}_{t} \succeq \eps I$ and $\norm{(\bar{V}_t)_{:,i}}_{\tilde{A}_t^{-1}}^2 \leq \norm{(\bar{V}_t)_{:,i}}_{\sbr{\eps I}^{-1}}^2 = \frac{1}{\eps}$.
  Substituting $(\bar{\Sigma}_t)_{i,i} = \sigma_{i}(S_t)$ completes the proof.
\end{proof}

\begin{lem}
  \label{lem:sketch-error-for-deferral}
  At the $t$-th round of \cref{alg:lightson}, in \cref{alg:fast-fd}, $\Delta_t$ also satisfies:
  \begin{equation*}
    \norm{\gr\gtyt}_{\tilde{A}_t^{-2}}^2 \leq \innerf{\tilde{A}_t^{-2}}{\tilde{A}_t - \tilde{A}_{t-1}} + \frac{\Delta_t}{\eps} .
  \end{equation*}
\end{lem}

We omit the proof of \cref{lem:sketch-error-for-deferral}, as it directly reuses the proof of \cref{lem:sketch-error}.

\begin{lem}[Theorem~1.1 and Section~3 of~\citet{SICOMP16:frequent-direct}]
  \label{lem:fd-subtracted-mass}
  In \cref{alg:lightson}, $\Delta_{1:T}$ satisfies that, for any $j \in [d']$,
  \begin{equation*}
    \Delta_{1:T} \triangleq \sum_{t=1}^{T} \Delta_t \leq \frac{2d'}{(d'-j+1)\eps} \sum_{i=j}^{d} \lambda_i\sbr{\sum_{t=1}^{T} \gr\gtyt \gr\gtyt^\trs} .
  \end{equation*}
\end{lem}

With the preceding lemmas characterizing the sketching error, we can prove \cref{thm:lightson} by bounding the difference between \lightons and \lightson.

\begin{proof}\textbf{of \cref{thm:lightson}}
  We note that \cref{lem:regret-decompose-exact} holds for \cref{alg:lightson}, as the proof of \cref{lem:regret-decompose-exact} only requires $A_t$ to be positive-definite.
  Thus by \cref{lem:surrogate-taylor,lem:regret-decompose-exact} we obtain
  \begin{equation*}
    \begin{aligned}
      & \ftxt - \ftu \overset{\text{\eqref{eq:surrogate-taylor}}}{\leq} \gr\gtyt^\trs (\y_t-\u) - \frac{\gamma_0}{2} \sbr{ \gr\gtyt^\trs (\y_t-\u) }^2 \\
      \overset{\text{\eqref{eq:regret-decompose-exact}}}{\leq} ~ & \frac{1}{2} \sbr{ \frac{1}{\gamma_0} \norm{\gr\gtyt}_{\tilde{A}_t^{-1}}^2 + \gamma_0 \norm{\y_t-\u}_{\tilde{A}_t}^2 -\gamma_0 \norm{\y_{t+1}-\u}_{\tilde{A}_t}^2 } - \frac{\gamma_0}{2} \sbr{ \gr\gtyt^\trs (\y_t-\u) }^2 \\
      \leq ~ & \frac{1}{2\gamma_0} \norm{\gr\gtyt}_{\tilde{A}_t^{-1}}^2 + \frac{\gamma_0}{2} \norm{\y_t-\u}_{\tilde{A}_{t-1}}^2 - \frac{\gamma_0}{2} \norm{\y_{t+1}-\u}_{\tilde{A}_t}^2 ,
    \end{aligned}
  \end{equation*}
  where the last inequality uses the fact that $\tilde{A}_t \preceq \tilde{A}_{t-1} + \gr\gtyt \gr\gtyt^\trs$, which is ensured by \cref{alg:fast-fd}.
  Then plugging \cref{lem:sketch-error} into the preceding inequality yields
  \begin{equation*}
    \begin{aligned}
      & \quad \ftxt - \ftu \\
      & \leq \frac{1}{2\gamma_0} \sbr{ \innerf{\tilde{A}_t^{-1}}{\tilde{A}_t - \tilde{A}_{t-1}} + \Delta_t } + \frac{\gamma_0}{2} \norm{\y_t-\u}_{\tilde{A}_{t-1}}^2 - \frac{\gamma_0}{2} \norm{\y_{t+1}-\u}_{\tilde{A}_t}^2 \\
      & \leq \frac{1}{2\gamma_0} \sbr{ \log \tdet{\tilde{A}_t} - \log \tdet{\tilde{A}_{t-1}} + \Delta_t } + \frac{\gamma_0}{2} \norm{\y_t-\u}_{\tilde{A}_{t-1}}^2 - \frac{\gamma_0}{2} \norm{\y_{t+1}-\u}_{\tilde{A}_t}^2 ,
    \end{aligned}
  \end{equation*}
  where the inequality uses the fact that $\innerf{X^{-1}}{X - Y} \leq \log \tdet{X} - \log \tdet{Y}$, which comes from the proof of \cref{lem:log-det,lem:neg-tr}.
  Telescoping the preceding inequality yields
  \begin{equation*}
    \begin{aligned}
      \sum_{t=1}^{T} \sbr{\ftxt-\ftu} \leq \frac{1}{2\gamma_0} \log \frac{\tdet{\tilde{A}_T}}{\tdet{\tilde{A}_0}} + \frac{1}{2\gamma_0} \Delta_{1:T} + \frac{\gamma_0}{2} \norm{\y_1-\u}_{\tilde{A}_0}^2 ,
    \end{aligned}
  \end{equation*}
  where the logarithmic term is further bounded with Jensen's inequality, which differs slightly from \cref{lem:log-det} due to the number of non-zero eigenvalues:
  \begin{equation*}
    \begin{aligned}
      \log \frac{\tdet{\tilde{A}_T}}{\tdet{\tilde{A}_0}} = \sum_{i=1}^{2d'} \log \sbr{ 1 + \frac{\sigma_i^2(S_T)}{\eps} } \leq 2d' \log \sbr{ 1 + \frac{\normf{S_T}^2}{2d' \eps} } \leq 2d' \log \sbr{ 1 + \frac{G^2}{2d' \eps} T } .
    \end{aligned}
  \end{equation*}

  Finally, the runtime follows from the following two parts:
  \begin{itemize}
    \item \textbf{Runtime aside from \cref{alg:fast-fd}.}~
      Following the same analysis of \cref{lem:proj-count,lem:bandits-time}, overall runtime aside from \cref{alg:fast-fd} is $O( (k - 1)^{-1} \sqrt{(d + \Delta_{1:T}) T/\eps} \cdot d^\omega \log T )$ due to \fastproj.
      (\cref{alg:lightson} applies $k=2$.)
      It suffices to verify the following inequality which follows from \cref{lem:sketch-error-for-deferral}:
      \begin{equation*}
        \Phi''_T \triangleq \sum_{t=1}^{T} \norme{\frac{1}{\gamma_0} \tilde{A}_t^{-1} \gr\gtyt}^2 \leq \frac{1}{\gamma_0^2} \frac{d + \Delta_{1:T}}{\eps} .
      \end{equation*}
    \item \textbf{Runtime of \cref{alg:fast-fd}.}~
      Each time SVD is triggered, the last $d'+1$ rows of $S_t$ become all-zero, thus SVD is triggered at most $\lceil T / (d'+1) \rceil$ times.
      The runtime of SVD is $O(d'^2 d \log T)$~\citep{golub-loan-mat-compute} to achieve the desired accuracy that does not affect the regret bound.
      \footnote{The runtime of SVD can be improved to $O\big( d'^2 (d + \log T) \big)$ and the factor $O(\log T)$ can be independent of the minimal singular value gap of $S_t$ and only depends on the scale of singular values of $S_t$~\citep{parlett1998symmetriceigen}.}
      The runtime of updating $R_t$ is $O(d'd)$, as this can be implemented with two rank-one updates:
      \begin{equation*}
        \begin{aligned}
          (A + \u \v^\trs)^{-1} = A^{-1} - \frac{1}{1 + \v^\trs A^{-1} \u} A^{-1} \u \v^\trs A^{-1} .
        \end{aligned}
      \end{equation*}
      Specifically,
      \begin{equation*}
        \begin{aligned}
          & \quad R_t^{-1} = \eps I + S_t S_t^\trs = \eps I + \sbr{S_{t-1}+\e_{i_t} \gr\gtyt^\trs} \sbr{S_{t-1}+\e_{i_t} \gr\gtyt^\trs}^\trs \\
          & = \underbrace{\eps I + S_{t-1} S_{t-1}^\trs}_{R_{t-1}^{-1}} + \e_{i_t} \underbrace{\sbr{S_{t-1} \gr\gtyt}^\trs}_{\a_t^\trs} + \underbrace{\sbr{S_{t-1} \gr\gtyt + \e_{i_t} \gr\gtyt \gr\gtyt^\trs}}_{\b_t} \e_{i_t}^\trs ,
        \end{aligned}
      \end{equation*}
      where $\e_{i_t}$, $\a_t$ and $\b_t$ are $2d'$-dimensional vectors that can be computed in $O(d' d)$ time.
      Therefore, the overall runtime of \cref{alg:fast-fd} is $O(d'^2 d (\log T) \lceil T / (d'+1) \rceil + d' d T) = O(d' d T \log T)$.
  \end{itemize}
\end{proof}

\paragraph{Comparison with OQNS.}
\citet{COLT23:OQNS} mention the possibility of combining OQNS with sketching but provide neither algorithms nor analysis.
Incorporating sketching further complicates OQNS's already intricate analysis, as the sketching error interacts with the log-barrier, Hessian approximation and decision updates in \cref{eq:oqns-update-objective}.
In contrast, within the OMD framework, sketching errors naturally appear as an additive term in the regret bound via OMD's stability term, illustrated in \cref{eq:omd-regret-decomposition-conceptual}, since OMD's regret decomposition is somewhat orthogonal to sketching.

\end{document}